\documentclass{article} 
\usepackage[final]{colm2026_conference}

\usepackage{microtype}
\usepackage{hyperref}
\usepackage{url}
\usepackage{booktabs}

\definecolor{darkblue}{rgb}{0, 0, 0.5}
\hypersetup{colorlinks=true, citecolor=darkblue, linkcolor=darkblue, urlcolor=darkblue}

\usepackage{amsmath,amsfonts,amssymb,amsthm}
\usepackage{graphicx}
\usepackage{algorithm}
\usepackage{algorithmic}
\usepackage{listings}
\usepackage[x11names, svgnames, dvipsnames]{xcolor}
\usepackage{natbib}
\usepackage{array}
\usepackage{mathtools}
\usepackage{tabularx}
\usepackage{enumerate}
\usepackage[inline,shortlabels]{enumitem}
\usepackage{subcaption}
\usepackage{multirow}
\usepackage{pifont}
\newcommand{\ms}[2]{{#1\tiny{$\pm$#2}}}

\usepackage[capitalize]{cleveref}
\usepackage{wrapfig}

\usepackage{arydshln}
\title{Training-Free versus Training-Based Intent Classification in LLMs: Accuracy, Robustness, and Failure Modes}

\author{Nan Chen\thanks{Equal contribution. Corresponding authors: \texttt{\{nchen38,zyang145\}@jh.edu}}
\quad
Zhouhao Yang\footnotemark[1]
\quad
Soufiane Hayou
\\
Department of Applied Mathematics and Statistics\\
Johns Hopkins University
}

\newcommand{\revise}[1]{{\color{black}#1}}

\newcommand{\reals}{\mathbb{R}}
\newcommand{\normal}{\mathcal{N}}
\newcommand{\E}{\mathbb{E}}

\newcommand{\Tr}{\mathop{\mathrm{Tr}}}
\newcommand{\M}{\mathcal{M}}
\newcommand{\vecstat}{{VecStat}}
\newcommand{\normstat}{{NormStat}}

\newcommand{\normstatKL}{NormStat:KL}
\newcommand{\vecstatKL}{VecStat:KL}
\newcommand{\vecstatCos}{VecStat:Cos}
\newcommand{\embedavg}{Avg-MLP}
\newcommand{\embedlast}{Tail-MLP}
\newcommand{\linearavg}{Avg-Linear}
\newcommand{\llmcall}[1]{LLM Call (#1-shot)}

\newcommand{\KL}{\mathrm{KL}}
\newcommand{\Normal}{\mathcal{N}}

\DeclareMathOperator{\Var}{{\rm Var}}

\newtheorem{thm}{Theorem}
\newtheorem{lemma}{Lemma}

\begin{document}

\maketitle

\begin{abstract}
Intent classification in Large Language Models (LLMs) involves categorizing user prompts into predefined classes. For instance, given a user prompt, the system must determine whether it primarily concerns mathematics, coding, or general text processing. Such classification enables routing prompts to specialized models optimized for specific domains, improving both accuracy and computational efficiency.
In this work, we conduct a systematic study comparing training-free vs training-based approaches for intent classification. For this purpose, we consider two lightweight, training-free methods based on statistics of internal representations and compare them against MLP classifiers and linear probes. Our comprehensive empirical evaluation reveals that 1) Both training-free and training-based methods saturate easy benchmarks (mathematics vs. coding vs. natural language), 2) Training-based classifiers have an advantage on harder classification tasks (e.g. Java vs Python), and 3) Training-free methods are generally more robust to mixed-intent and adversarial prompts.
\end{abstract}

\section{Introduction}
Large Language Models (LLMs) have demonstrated remarkable capabilities across diverse tasks, including mathematical reasoning \citep{wei2022chain,yao2023tree,gao2023pal} and code generation \citep{li2022competition,guo2024deepseek,zhu2024deepseek}. At the same time, modern LLM systems are increasingly deployed through routing architectures, where an intent classifier first identifies the user’s intent and then dispatches the prompt to an appropriate downstream model or tool \citep{gpt5,bocklisch2017rasa,bunk2020diet,arora2024intent}. Such routing can improve the efficiency and reliability of production-scale systems, but it also raises a basic question: \emph{what kind of intent classifier should be used?}

Existing approaches to intent classification face a difficult trade-off. Direct LLM calls are flexible, but they are hard to calibrate and less convenient when reliable uncertainty estimates are needed \citep{bang2023multitask,banerjee2025llms}. Dedicated supervised classifiers \citep{larson2019evaluation,chen2019bert} can achieve strong task-specific performance, but they might require extensive training data and computational resources, and need retraining when the intent label space evolves. In this work, we focus on two types of lightweight intent classifiers, training-free versus training-based methods\footnote{We use the wording training-free vs.\ training-based methods to refer to whether gradient updates are needed.}, and conduct a systematic empirical study to study the accuracy, robustness, and failure modes of each approach.

For training-free methods, we introduce \vecstat{} and \normstat{}, which operate entirely in the prefill phase with negligible cost. The motivation is the observation that different prompt types (mathematics, coding, general text, etc.) induce distinct activation distributions. Specifically, \vecstat{} and \normstat{} represent two levels of statistical compression: \vecstat{} preserves coordinate-wise directional information, while \normstat{} compresses features into radial summaries and uses substantially less memory.
For training-based methods, we consider an MLP classifier inspired by sentence-classification pipelines \citep{casanueva2020efficient,Jiang2024scaling}, applied to the LLM's final projection layer. The MLP head is trained on labeled data, whereas the training-free methods require no gradient updates, only a simple statistical estimation.

We conduct extensive empirical analysis to compare the two approaches. We apply these methods to LLMs ranging from 1B to 32B parameters and evaluate intent classification at both coarse-grained and fine-grained levels across seven benchmark datasets. The empirical results reveal that \emph{there is no one-fits-all method for intent classification}. On easy coarse-grained tasks, both training-free and training-based methods often saturate the benchmark, while on harder fine-grained tasks, training-based classifiers typically achieve higher accuracy. At the same time, this advantage does not necessarily persist under our stress tests: when prompts contain mixed intent, training-free methods provide better uncertainty estimates, and when prompts are adversarially rephrased to inject misleading content, they tend to remain more stable than the trained-based methods. Our contributions are:

\begin{itemize}[leftmargin=1.3em]
    \item We introduce two lightweight statistical methods, \normstat{} and \vecstat{}, that operate on prefill-phase, and develop theory explaining their respective strengths. We show that VecStat is advantageous when class separation is primarily directional, while NormStat is sufficient in coarser regimes where radial information is rich enough to separate data, with the additional benefit of cheaper calibration.
     \item We conduct a systematic empirical study comparing training-free and training-based intent classification across seven LLMs (1B--32B parameters) at both coarse and fine levels of granularity. Results show both paradigms saturate on easy tasks, while training-based methods typically outperform training-free ones on harder, fine-grained distinctions.

    \item \revise{We create a mixed-intent dataset and an adversarial prompt dataset to stress-test intent classifiers under mixed-intent ambiguity and misleading surface cues.} We find that training-free methods are typically more robust to mixed-intent prompts and adversarial rephrasings, making them potentially attractive for practical LLM routing settings where inputs are often noisy or ambiguous.

\end{itemize}

\subsection{Related Work}
\paragraph{Task Classification}
Intent classification maps a user prompt to a predefined label. Classical approaches either (i) train supervised classifiers over tokenized utterances (e.g., CNNs) to produce a distribution over intents \citep{hashemi2016query,goo2018slot,he2019using}, or (ii) fine-tune contextual encoders, particularly BERT-based models, where hidden states feed specialized intent classification heads, often jointly trained with slot filling tasks \citep{chen2019bert,bocklisch2017rasa,bunk2020diet}. In modern LLM-based systems, intent classification serves as a critical routing mechanism that allows the selection of appropriate downstream tools and models, enforces guardrails and fallback policies, and optimizes inference cost and latency \citep{souha2023pre,arora2024intent}. The predominant approach involves direct LLM inference through several key techniques \citep{liu2023pre,rodriguez2024intentgpt,wang2023large,arora2024intent,hong2024exploring}. However, the computational expense of LLM inference at scale has motivated hybrid architectures that combine fast, lightweight classifiers (including PEFT-tuned encoders) with LLMs through uncertainty-aware routing mechanisms. These systems employ confidence thresholding, entropy-based measures, or learned routing policies to reserve expensive LLM calls for ambiguous cases where simpler models exhibit high uncertainty \citep{arora2024intent}.

\paragraph{LLMs as text encoder}
Recent advances in LLMs have prompted researchers to explore their use as text encoders. An interesting approach is embedding extraction where existing methods typically operate on the last layer outputs through three strategies: using the last token embedding \citep{ma2024fine,neelakantan2022text,wang2024improving,meng2024sfrembedding,Jiang2024scaling}, averaging across all token embeddings \citep{muennighoff2022sgpt,muennighoff2024generative,behnamghader2024llm2vec}, or employing trainable modules \citep{lee2024nv,tang2024pooling}. 
Interested readers can refer to ~\citep{tao2024llms,nie2024text} for a more detailed review on this topic.
In contrast to these approaches, this work addresses user-intent classification for routing where both accuracy and computational efficiency are primary considerations. Our method utilizes prefill-time outputs from general-purpose LLMs without modification or additional training. By leveraging computational intermediates already produced during LLM prefill phase, this approach avoids the storage overhead of maintaining a dedicated billion-parameter model for intent classification. 

\paragraph{Neural Feature Analysis}
Our approach extracts representations $Wz$, where $W$ is a pretrained weight matrix and $z$ is model's hidden state. This design is motivated by two lines of research. First, linear probes effectively extract semantic information from transformer representations \citep{alain2016understanding,hewitt2019structural}, with sparse autoencoder studies suggesting that many concepts are captured by a small number of sparse features in the activation space \citep{cunningham2024sae,gao2024scalingSAE}. Superposition theory provides theoretical grounding, explaining how features remain recoverable through linear projections \citep{elhage2022toy}. Second, activation steering research demonstrates that intent-related behaviors can be manipulated through linear interventions in the representation space \citep{turner2023activation,panickssery2023steering}. Finally, activations $Wz$ were successfully used in \cite{hayou2025ploppreciseloraplacement} to determine target module for LoRA finetuning, showing that activation capture  data signal.

Further discussion of related works in LLM routing is deferred to Appendix~\ref{app:more_related_works}.

\section{Methodology}

We study training-free versus training-based intent classification in large language models. Given a prompt, the goal is to assign it to one of several intent classes, such as mathematics, code, or general text. All methods considered in this paper operate on neural features computed during the model’s forward pass before autoregressive decoding \citep{shazeer2019fast,ainslie2023gqa,chang2024palu,aguirre2025fine,jie2025specache}. The difference lies in how these representations are used. Training-based methods fit a classifier head from labeled data, whereas training-free methods compare the prompt’s neural feature statistics to class-specific reference statistics computed from calibration data.
This distinction leads to different strengths and weaknesses. Training-based classifiers can achieve higher accuracy on harder fine-grained tasks, while training-free methods often behave more robustly when prompts are ambiguous or adversarially rephrased.


\subsection{Training-free Intent Classification: A Statistical Approach}\label{sec:training_free}
Consider an LLM with weight modules $\M = \{W_1, W_2, \dots, W_p\}$, for some $p \geq 1$. We abuse the notation and use $W_\ell$ to refer to both the module and its weight matrix. Let $\mathbf x = (x_t)_{1 \leq t \leq T}$ be a prompt of $T$ tokens. For each weight module $W_k$, let $\left(y_{\ell,t}\right)_{1\leq t\leq T}$ denote the output features in module $W_\ell$. For instance, $\left(y_{\ell,t}\right)_{1\leq t\leq T}$ could be the output of a Query head, or the projection layer in an MLP block. Each $y_{\ell,t}$ is a $d$-dimensional vector given by $y_{\ell,t} = W_\ell z_{\ell, t}$, where $d$ is the output dimension in module $W_\ell$, and $z_{\ell, t}$ is the module input for token $t$. 

\begin{figure}[t]
    \centering
    \vspace{-4em}
    \includegraphics[width=0.8\linewidth]{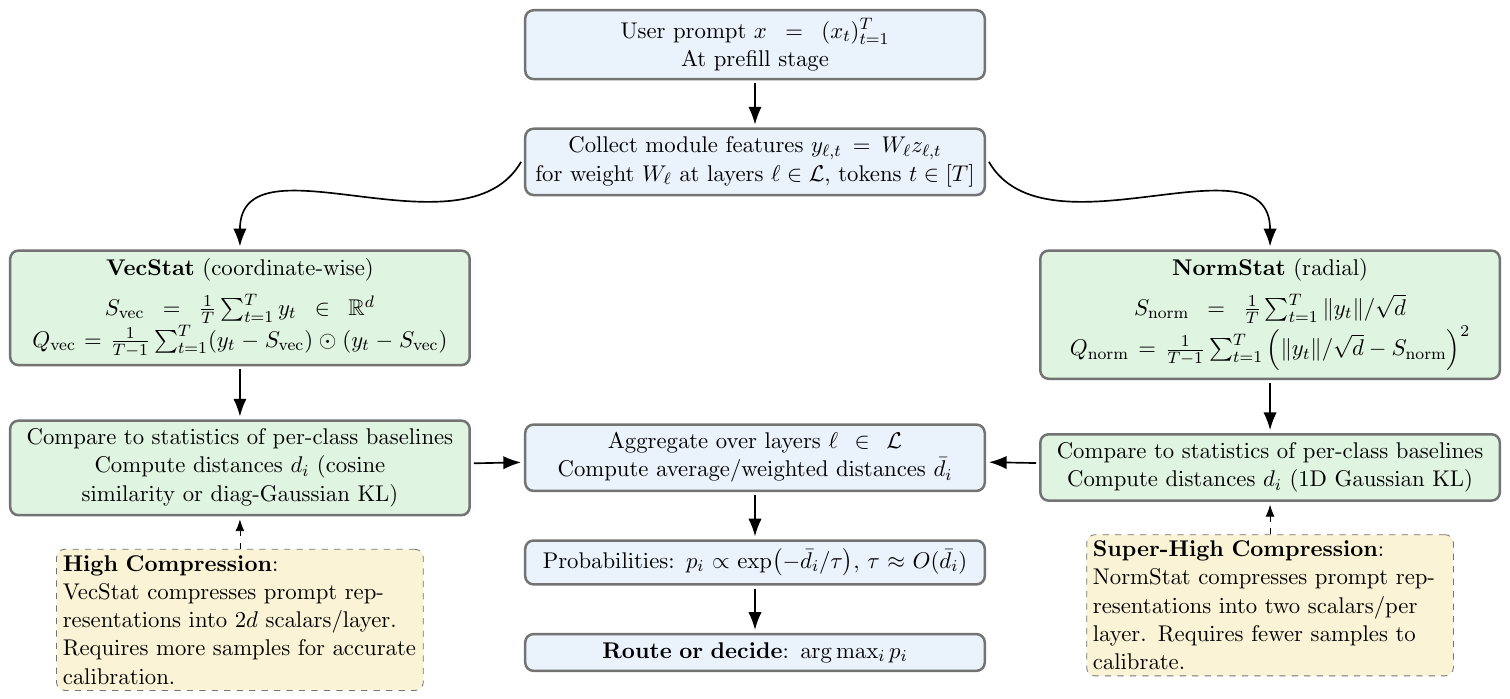}
    \vspace{-7.5em}
   \caption{\small Flowchart of \vecstat\ vs \normstat\ for prefill-time intent classification and routing.}
   \vspace{-1em}
    \label{fig:flowchart}
\end{figure} 

\paragraph{Intent classification.} We aim to classify the prompt $\mathbf x$ into one of the classes $C_1, C_2, \dots, C_m$, where $m \geq 2$. For instance, a binary classification where $C_1$ is mathematics and $C_2$ is coding. For each module $W_\ell$, we compute summary statistics from the
features $\{y_{\ell,t}\}_{t=1}^T$ and compare them to per-class baselines:
(i) for each class $C_i$, precompute the same statistics on \emph{calibration data} at the same module $W_\ell$;
(ii) for the incoming prompt, compute the statistics at $W_\ell$ and measure similarity to each baseline;
(iii) assign the class  $C_i$ with highest similarity. 

We consider two methods: 1) \emph{\vecstat}, which relies on coordinate-wise mean and variance for classification, and 2) \emph{\normstat}, a lighter variant that relies solely on the norm across all tokens and coordinates. \cref{fig:flowchart} summarizes both methods. In the following, we present the two methods in the single-layer case. When multiple layers are used, we aggregate similarity scores across $\ell$ by averaging.

\noindent\textbf{Vector Statistic }(\vecstat): calculate  \emph{coordinate-wise} token means and second moments tokens
\begin{equation}\label{eq:vecstat}
S_{\mathrm{vec}}=\frac{1}{T} \sum_{t=1}^T y_t\in\reals^d,\qquad Q_{\mathrm{vec}}=\frac{1}{T-1}\sum_{t=1}^T (y_t - S_{\mathrm{vec}})\odot (y_t - S_{\mathrm{vec}})\in\reals^d.
\end{equation}

\noindent\textbf{Norm Statistic }(\normstat): summarize each $y_t$ through a the norm $\|y_t\|$ and aggregate across tokens to obtain the statistics
\begin{equation}\label{eq:normstat}
S_{\text{norm}} = \frac{1}{T} \sum_{t=1}^T \frac{\|y_t\|}{\sqrt{d}}\in\reals, \quad Q_{\text{norm}} = \frac{1}{T-1} \sum_{t=1}^T \left(\frac{\|y_t\|}{\sqrt{d}} - S_{\text{norm}}\right)^2\in\reals.
\end{equation}

\begin{wrapfigure}{r}{0.3\textwidth} 
\vspace{-3em}
\centering
\includegraphics[width=\linewidth]{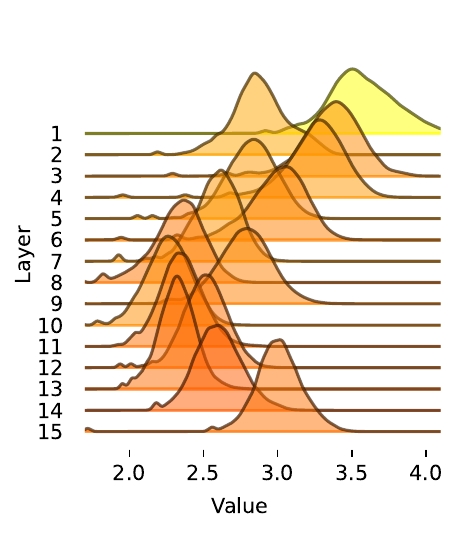}
\caption{\small{\textbf{Per-token query-norm distributions across layers.}
  Histograms of $\{\|y_{\ell,t}\|/\sqrt{d}\}_{t=1}^T$ for representative layers $\ell$.}}
\label{fig:norm_dist}
\vspace{-1em}
\end{wrapfigure}

For both methods, we score a prompt against each class using the closed-form KL divergence between \revise{Gaussian surrogate distributions parameterized by the corresponding summary statistics} (see Eqs.~\eqref{eq:vecstat-kl-sum} and \eqref{eq:normstat-kl-1d} in the Appendix). \vecstat~uses a diagonal Gaussian over the activation coordinates, whereas \normstat~uses a univariate Gaussian over the normalized token norms $\|y_{\ell,t}\|/\sqrt{d}$. \revise{The resulting score is the exact KL divergence between the Gaussian surrogates, but it is not intended to approximate the KL divergence between the unknown underlying activation distributions. Rather, Gaussian moment matching provides a simple and computationally efficient scoring rule based on first- and second-order statistics. Figure~\ref{fig:norm_dist} provides qualitative motivation for this choice by showing approximately bell-shaped radial statistics at representative layers, while we do not claim that LLM activations are literally Gaussian or independent.}\footnote{One could estimate an empirical KL without summaries, but doing so robustly at inference time is prohibitively expensive in both compute and memory.} We also evaluate cosine similarity for \vecstat~as an alternative that does not rely on Gaussian KL scoring.



These two choices form a \emph{statistical compression ladder}: \vecstat~keeps per-coordinate first and second moments, while \normstat~compresses all coordinates to a single \emph{radial} information per token and then to its mean/variance across tokens. The rest of this section develops this story rigorously: (i) we discuss when each is statistically preferable using a simplified Gaussian setting, and (ii) we connect those guarantees to \emph{compute}, \emph{memory}, and \emph{calibration} costs. Proofs are in \cref{app:proofs}.

\paragraph{Intuitive Analysis in a Gaussian Setting}

To compare NormStat and VecStat, we study an analytically tractable Gaussian proxy model. 
For a fixed class $k \in \{1,\dots,m\}$, suppose that the token features (not necessarily representations in an LLM) satisfy
\begin{equation}\label{eq:model}
y_t \mid k \overset{\mathrm{i.i.d.}}{\sim} \mathcal N(\mu_k,\Sigma_k),
\qquad t=1,\dots,T,
\end{equation}
where $\Sigma_k = \mathrm{Diag}(\sigma_{k,1}^2,\dots,\sigma_{k,d}^2)$ is diagonal. 

\paragraph{Expressivity.}
\revise{Since NormStat depends only on 
$\{\|y_t\|_2\}_{t=1}^T$, it is invariant under orthogonal transformations and cannot distinguish classes whose separation lies purely in direction. In contrast, VecStat is sensitive to coordinate-wise changes. Specifically, when classes share the same mean and differ only through spherical covariance scales, radial information is sufficient and NormStat is already \emph{Bayes-optimal}. In the directional regime, where classes have equal covariance and equal mean norm but different mean directions, every radius-only rule is blind, whereas a coordinate-aware rule can achieve exponentially small error in $T$. See the following theorem for a formal statement. The proof is provided in \cref{app:proofs}.}
\begin{thm}[\normstat~vs.\ \vecstat]
\label{thm:radial-vs-coord}
Consider binary classification with uniform class priors.
\begin{enumerate}[ref=\ref{thm:radial-vs-coord}-(\arabic*)]
    \item \label{thm:directional}
    \textbf{Directional regime.}
    If $\Sigma_1=\Sigma_2=\sigma^2 I_d$,
    $\|\mu_1\|=\|\mu_2\|$, and $\mu_1\neq\mu_2$,
    then every classifier based only on
    $\{\|y_t\|\}_{t=1}^T$ has error $1/2$, whereas the
    likelihood-ratio test satisfies
    \[
    \Pr(\hat{k}\neq k)
    \leq
    \exp\left(
    -\frac{T\|\mu_1-\mu_2\|^2}{8\sigma^2}
    \right).
    \]

    \item \label{thm:isotropic}
    \textbf{Isotropic-scale regime.}
    If $\mu_1=\mu_2=0$ and
    $\Sigma_k=\sigma_k^2 I_d$ with $\sigma_1\neq\sigma_2$,
    then the likelihood-ratio test depends only on
    \(
    R_T=\sum_{t=1}^T\|y_t\|^2.
    \)
    Hence radial information is sufficient, and coordinate-wise
    information cannot improve the Bayes risk.
\end{enumerate}
\end{thm}
\revise{\Cref{thm:radial-vs-coord} describes two limiting regimes, whereas
real activation distributions may contain both radial and directional
differences. To illustrate an intermediate setting, write
\[
y_t=s_t\mu+\varepsilon_t,
\qquad
\Pr(s_t=1)=\pi,
\qquad
\varepsilon_t\sim\mathcal N(0,\sigma^2I_d),
\]
where $s_t\in\{\pm1\}$, $\mu\in\mathbb R^d$, and
$\gamma\coloneqq\mathbb E[s_t]=2\pi-1$. Equivalently,
\[
y_t\sim
\pi\mathcal N(\mu,\sigma^2I_d)
+(1-\pi)\mathcal N(-\mu,\sigma^2I_d).
\]
The corresponding statistics satisfy
\[
\mathbb E[S_{\mathrm{vec}}]=\gamma\mu,
\qquad
\mathbb E[Q_{\mathrm{vec}}]
=
\sigma^2\mathbf 1_d
+
(1-\gamma^2)\mu\odot\mu,
\qquad
\mathbb E[S_{\mathrm{norm}}]
=
F(\|\mu\|,\sigma^2),
\]
for a function $F$ depending only on the radial parameters. Because the norm
is invariant under $\mu\mapsto-\mu$, the distributions of
$S_{\mathrm{norm}}$ and $Q_{\mathrm{norm}}$ do not depend on $\pi$.
By contrast, \vecstat~retains information about $\pi$ through both its mean
and coordinate-wise variance. As $\pi\to\tfrac12$, the first-moment signal
$\mathbb E[S_{\mathrm{vec}}]$ vanishes, but the second-moment term
$\mu\odot\mu$ generally remains. Thus, \vecstat~may still retain
coordinate-wise information that is invisible to \normstat, even in the
symmetric-mixture case.}

\paragraph{Calibration Cost.} When using Normstat and Vecstat, an important aspect is \emph{calibration} which we use to refer to estimating key statistics used for classification.\footnote{Calibration for Vecstat and Normstat is the equivalent of training for MLP-based methods.} An important aspect of statistical methods is sample complexity, or more specifically, the convergence rate of key statistics in sample size. This provides an estimate of the total number of calibration samples needed to create the target statistics for classes $k \in \{1,2, \dots, m\}$. The next theorem show the calibration advantage of \normstat~over \vecstat.

\begin{thm}[Calibration cost]
\label{thm:calibration}
Fix a class $k$. Let $y_1,\dots,y_N \stackrel{\text{i.i.d.}}{\sim} \mathcal N(\mu_k,\Sigma_k)$ in $\mathbb R^d$,
where $N$ is the number of calibration samples drawn for this class.
Let $q = \E[\|y_1\|]$ and define
$\hat\mu_k = N^{-1}\sum_{i=1}^N y_i,$ and $\quad
\hat q = N^{-1}\sum_{i=1}^N d^{-1/2}\|y_i\|$.
Then, for any $\delta \in (0,1)$, with probability at least $1 - \delta$, we have:
\begin{enumerate}
\item \textbf{\normstat\ (dimension-free):}  $\left|\hat q- q \right|\lesssim \sqrt{\frac{\log(1/\delta)}{N}}.$
\item \textbf{\vecstat\ (dimension-dependent):} $\|\hat\mu_k-\mu_k\|_2 \lesssim  \sqrt{\frac{d+\log(1/\delta)}{N}}$.
\end{enumerate}
\end{thm}

Considering just the statistics $\hat \mu$ and $\hat q$, to obtain an estimation error of order $\epsilon$, one needs $N=\Omega(\epsilon^{-2})$ for $\hat q$ and $N=\Omega(d \epsilon^{-2})$ for $\hat \mu$, showing the computational advantage of \normstat~over \vecstat. This is particularly important in data scarce regimes with few samples for each class. We discuss this in more details in the next section.

\subsection{Training-based Intent Classification: MLP Classifier and Linear Probes}
For a comprehensive and fair comparison with the training-free methods above, we consider a trained classifier head on top of frozen LLM features from the \emph{prefill} phase. We use the last
Transformer block and write $y_{\ell^\star,t}\in\reals^d$ for its token features
($t=1,\dots,T$). Inspired by prompt/sentence
classification pipelines (e.g., \citep{ma2024fine,wang2024improving,meng2024sfrembedding}), we build a
single prompt-level vector in two ways:
\[
\textbf{Avg-MLP:}\quad z_{\mathrm{avg}} := \frac{1}{T}\sum_{t=1}^T y_{\ell^\star,t}
=S_{\mathrm{vec}}^{(\ell^\star)}\in\reals^d,
\qquad
\textbf{Tail-MLP:}\quad z_{\mathrm{tail}} := y_{\ell^\star,T}\in\reals^d.
\]
Given $z\in\{z_{\mathrm{avg}}, z_{\mathrm{tail}}\}$, we train a two-layer MLP with hidden width $h$ with cross entropy loss. Since $z_{\mathrm{avg}}$ or $z_{\mathrm{tail}}$
is produced during prefill, the incremental latency is a single MLP forward pass.
We also consider a simple linear probe variant that we call \textbf{Avg-Linear} where we use a simple projection instead of MLP. 
We defer discussion of training/calibration costs and practical trade-offs between training-free and training-based methods to~\cref{app:add_discussion}.


\section{Experiments}
In this section, we evaluate the effectiveness of \normstat, \vecstat, Avg-MLP, Tail-MLP, and Avg-Linear across multiple LLMs and classification datasets. All experimental details can be found in~\cref{appendix:exp-details}, and additional experimental results are presented in~\cref{appendix:exp}.
Our results provide a systematic comparison between training-free and training-based intent classification methods, showing the advantages and disadvantages of each approach. For completeness, we also compare with zero-shot direct LLM call for intent classification.
The code and datasets associated with this work are publicly available at \url{https://github.com/Zhouhao-Yang/Training-Free-versus-Training-Based-Intent-Classification-in-LLMs}.

\subsection{Experimental Setup}
\paragraph{Classification tasks and granularities.}
We consider four classification tasks with different levels of granularity. Task~1 addresses coarse-grained intent classification among three broad categories: general text, mathematics, and code; Task~2 considers programming language identification; Task~3 considers natural language identification;  Task~4 considers mathematical subfield classification.  Task 1-3 are coarse-grained (level-1 granularity), while Task~4 is a fine-grained task (level-2 granularity).

For calibration in Task~1, we use representative datasets for each intent class: MMLU European History~\citep{hendrycks2021ethics,hendryckstest2021} for general text, GSM8K~\citep{cobbe2021gsm8k} for mathematics, and Magicoder~\citep{wei2023magicoder} for code. We test on MMLU US History for general text, GSM8K and MATH500~\citep{lightman2023lets} for mathematics (in-distribution and out-of-distribution, respectively), and Magicoder and HumanEval~\citep{chen2021evaluating} for code (in-distribution and out-of-distribution, respectively). For Task~2, we use Magicoder for programming language identification. For Task~3, we use the Aya dataset~\citep{singh2024aya} for natural language identification, splitting each language subset into calibration and test sets. For Task~4, we use domain-specific subsets of Competition Math~\citep{hendrycksmath2021} for mathematical subfield classification. More details are provided in~\cref{appendix:datatset_composition}.


\paragraph{Method and LLM selection} We compare five classification methods: \normstat{}, \vecstat{} (with two variants: cosine similarity, \vecstatCos{}, and KL divergence, \vecstatKL{}), and training-based baselines \embedavg{}, \embedlast{}, and \linearavg{}. Training-based methods use the same calibration data for training to ensure fair comparison.
We also benchmark 0-shot and 3-shot LLM calls, where model predicts intent directly from the prompt. 
We find that providing a high-level overview of the intent classes in the prompt is necessary for achieving a reasonable performance.
All calibration prompts are truncated to 512 tokens, with training-free methods probing all linear modules. 
We evaluate Qwen3 and Llama models across 1B–32B scales, including base and instruction-tuned variants; see \cref{appendix:benchmark_LLM} for details. \revise{Additionally, we fine-tune RoBERTa-Base end-to-end with a task-specific classification head, using the same labeled training/calibration and test splits as the other methods.}

\paragraph{Evaluation Metrics} For task 1, we compute accuracy on each test dataset independently, where each dataset contains samples from a single ground-truth class. This approach ensures our evaluation is not biased by varying dataset sizes across classes. For task 2 and 3, we report mean accuracy across all classes within each task due to space constraints. This mean accuracy corresponds to the balanced accuracy metric, providing equal weight to each class regardless of test set size and effectively handling the natural class imbalance among test datasets. All experiments are run with three seeds, with mean and std reported. 


\subsection{Empirical Results}
\Cref{tab:method_comparison} summarizes the computational and practical trade-offs among intent classification methods. \Cref{tab:main-table} reports classification results for Tasks~1--3 (all level-1 granularity) across five methods and four representative LLMs. For Task~1, we report per-dataset accuracy; for Tasks~2 and~3, we report mean balanced accuracy across classes. \Cref{tab:level-2-math-results-short} provides per-subfield accuracy for Task~4 (level-2 mathematical subfield classification). Full results for all seven LLMs are reported in \cref{appendix:l1-exp} (level-1) and \cref{appendix:l2-exp} (level-2).

\begin{table}[t]
\centering
\small
\resizebox{0.995\linewidth}{!}{
\begin{tabular}{lcccc}
\hline
\textbf{Method} & \textbf{FLOPs Overhead} & \textbf{Memory Overhead} & \textbf{Extendability of New Classes} & \textbf{Adversarial Robustness}\\
\hline
\textbf{NormStat} 
& $O(Td)$
& $O(m)$ 
& Compute new baselines
& Better\\
\textbf{VecStat} 
& $O(Td)$
& $O(md)$ 
& Compute new baselines
& Better\\
\textbf{MLP} 
& $O(hd)$
& $O(hd)$
& Retrain a new MLP head
& Worse\\
\textbf{LLM Call} 
& $\Omega(Td^2)$ 
& $\sim$
& Extend via prompt engineering
& $\times$ \\
\hline
\end{tabular}
}
\caption{\small Classifier Comparison. $T$: prompt length; $d$: hidden width; $m$: \# classes; $h$: MLP hidden size. }

\label{tab:method_comparison}
\end{table}

\paragraph{Computational Overhead.}
NormStat stores only $O(m)$ scalars and performs $O(m)$ scoring FLOPs at 
inference, making it the most lightweight option. VecStat stores $O(md)$ 
numbers and requires $O(md)$ scoring FLOPs. Meanwhile, training-based 
methods rely on an MLP head, requiring $O(hd)$ FLOPs and $O(hd)$ parameter storage, where $h$ is the hidden layer dimension. Direct LLM calls incur 
the highest computational overhead at $\Omega(Td^2)$ FLOPs.

\begin{table}[t]
\centering

\small
\caption{\small Classification accuracy across methods on Tasks~1--3 (all level-1 granularity). Task~1 reports per-dataset accuracy for domain classification (general text, math, code). Tasks~2 and~3 report mean balanced accuracy across classes: programming language identification and natural language identification, respectively. LLM Call results omitted for Llama-3.2-1B models due to invalid outputs.}
\label{tab:main-table}
\resizebox{\linewidth}{!}{%
\begin{tabular}{llcccccccc}
\toprule
\multirow{2}{*}{Model} & \multirow{2}{*}{Method} & \multicolumn{5}{c}{Task~1} & \multicolumn{1}{c}{Task~2} & \multicolumn{1}{c}{Task~3} \\
\cmidrule(lr){3-7} \cmidrule(lr){8-8} \cmidrule(lr){9-9}
 & & gsm8k & humaneval & magicoder & math500 & mmlu\_history & programming & natural language \\



 



\midrule

\multirow{6}{*}{Llama-3.2-1B}
 & \embedavg    & \ms{100.00}{0.00} & \ms{100.00}{0.00} & \ms{99.99}{0.01} & \ms{64.33}{8.33} & \ms{99.84}{0.28} & \ms{99.97}{0.03} & \ms{99.91}{0.04} \\
 & \embedlast   & \ms{99.95}{0.04} & \ms{100.00}{0.00} & \ms{99.97}{0.03} & \ms{98.93}{0.42} & \ms{99.35}{1.13} & \ms{99.46}{0.19} & \ms{99.83}{0.07} \\
 & \linearavg   & \ms{100.00}{0.00} & \ms{99.80}{0.35} & \ms{99.99}{0.01} & \ms{71.47}{7.78} & \ms{98.69}{1.86} & \ms{99.96}{0.03} & \ms{99.90}{0.02} \\

\cdashline{2-9}
 
 & \normstatKL  & \ms{99.49}{0.09} & \ms{90.85}{0.00} & \ms{96.39}{0.20} & \ms{83.40}{0.00} & \ms{92.48}{0.57} & \ms{49.21}{1.38} & \ms{86.53}{1.33} \\
 & \vecstatKL   & \ms{100.00}{0.00} & \ms{99.39}{0.00} & \ms{99.98}{0.02} & \ms{78.80}{0.12} & \ms{100.00}{0.00} & \ms{98.99}{0.20} & \ms{99.19}{0.08} \\
 & \vecstatCos  & \ms{100.00}{0.00} & \ms{99.39}{0.00} & \ms{99.97}{0.02} & \ms{77.60}{0.20} & \ms{100.00}{0.00} & \ms{98.71}{0.20} & \ms{99.72}{0.01} \\

\midrule

\multirow{7}{*}{Qwen3-8B}
 & \embedavg    & \ms{99.42}{0.74} & \ms{99.59}{0.35} & \ms{99.99}{0.02} & \ms{77.27}{11.02} & \ms{100.00}{0.00} & \ms{99.97}{0.02} & \ms{99.92}{0.01} \\
 & \embedlast   & \ms{99.82}{0.12} & \ms{98.58}{2.46} & \ms{99.98}{0.02} & \ms{81.20}{9.72} & \ms{100.00}{0.00} & \ms{99.66}{0.12} & \ms{99.88}{0.04} \\
 & \linearavg   & \ms{99.57}{0.24} & \ms{100.00}{0.00} & \ms{100.00}{0.00} & \ms{83.00}{2.03} & \ms{100.00}{0.00} & \ms{99.97}{0.02} & \ms{99.94}{0.01} \\

\cdashline{2-9}

 & \normstatKL  & \ms{85.14}{0.35} & \ms{10.37}{1.06} & \ms{99.85}{0.06} & \ms{92.93}{0.12} & \ms{99.51}{0.49} & \ms{56.39}{0.87} & \ms{90.09}{0.40} \\
 & \vecstatKL   & \ms{99.95}{0.04} & \ms{99.59}{0.35} & \ms{99.99}{0.02} & \ms{92.20}{0.00} & \ms{100.00}{0.00} & \ms{99.23}{0.13} & \ms{99.20}{0.08} \\
 & \vecstatCos  & \ms{100.00}{0.00} & \ms{95.73}{0.61} & \ms{99.98}{0.03} & \ms{94.20}{0.00} & \ms{100.00}{0.00} & \ms{99.34}{0.14} & \ms{99.68}{0.03} \\

\cdashline{2-9}

& \llmcall{0}  & \ms{99.67}{0.04} & \ms{100.00}{0.00} & \ms{99.42}{0.07} & \ms{99.60}{0.00} & \ms{99.67}{0.28} & \ms{98.89}{0.05} & \ms{81.72}{0.09} \\

& \revise{\llmcall{3}} & \revise{\ms{99.14}{0.18}} & \revise{\ms{97.15}{0.35}} & \revise{\ms{99.43}{0.07}} & \revise{\ms{99.87}{0.12}} & \revise{\ms{99.51}{0.49}} & - & -\\

\midrule

\multirow{7}{*}{Qwen3-32B}
 & \embedavg    & \ms{95.88}{3.31} & \ms{100.00}{0.00} & \ms{99.99}{0.01} & \ms{86.87}{5.22} & \ms{100.00}{0.00} & \ms{99.97}{0.03} & \ms{99.93}{0.02} \\
 & \embedlast   & \ms{99.39}{0.00} & \ms{100.00}{0.00} & \ms{99.98}{0.00} & \ms{96.80}{0.40} & \ms{99.02}{0.49} & \ms{98.41}{0.39} & \ms{99.84}{0.01} \\
 & \linearavg   & \ms{98.61}{1.59} & \ms{100.00}{0.00} & \ms{99.99}{0.02} & \ms{87.53}{6.94} & \ms{100.00}{0.00} & \ms{99.97}{0.03} & \ms{99.94}{0.01} \\

\cdashline{2-9}
 
 & \normstatKL  & \ms{97.93}{0.04} & \ms{24.59}{0.35} & \ms{99.78}{0.06} & \ms{97.93}{0.12} & \ms{100.00}{0.00} & \ms{57.02}{1.03} & \ms{89.62}{0.08} \\
 & \vecstatKL   & \ms{100.00}{0.00} & \ms{99.39}{0.00} & \ms{99.98}{0.03} & \ms{96.60}{0.00} & \ms{100.00}{0.00} & \ms{99.53}{0.11} & \ms{98.74}{0.05} \\
 & \vecstatCos  & \ms{100.00}{0.00} & \ms{98.17}{0.00} & \ms{99.98}{0.03} & \ms{96.80}{0.35} & \ms{100.00}{0.00} & \ms{99.61}{0.10} & \ms{99.58}{0.02} \\

\cdashline{2-9}

  & \llmcall{0}  & \ms{86.91}{0.29} & \ms{100.00}{0.00} & \ms{98.69}{0.06} & \ms{96.27}{0.42} & \ms{100.00}{0.00} & \ms{99.82}{0.03} & \ms{99.16}{0.01}\\

   & \revise{\llmcall{3}} & \revise{\ms{97.80}{0.20}} & \revise{\ms{100.00}{0.00}} & \revise{\ms{98.87}{0.05}} & \revise{\ms{99.80}{0.20}} & \revise{\ms{100.00}{0.00}} & - & - \\

\midrule

\revise{RoBERTa} & - & \revise{\ms{100.00}{0.00}} & \revise{\ms{100.00}{0.00}} & \revise{\ms{99.98}{0.02}} & \revise{\ms{42.33}{13.05}} & \revise{\ms{100.00}{0.00}} & \revise{\ms{99.97}{0.05}} & \revise{\ms{99.92}{0.16}} \\

\bottomrule
\end{tabular}%
}
\vspace{-0.5em}
\end{table}

\paragraph{Level-1 classification: All methods perform well in Task 1 and 3; Task~2 exposes NormStat's limitation.} For task 1 and 3, all methods achieve strong in-distribution performance. The effectiveness of \normstat{}, despite using only radial statistics, demonstrates itself as a computational- and memory-efficient approach when the classes are different enough. 
\linearavg{} performs comparably to \embedavg{} but with slightly lower accuracy, consistent with the reduced expressiveness of a linear layer versus two-layer architecture. Direct LLM inference shows improved classification accuracy with larger models, though it requires careful prompt design to achieve a reasonable performance.
\revise{Adding few-shot examples can improve LLM Call performance through in-context learning, but at the cost of a longer prompt and consequently higher computational cost.}
Notably, Llama-3.2-1B models fail to produce valid responses, hence their results are omitted.
\revise{The fine-tuned RoBERTa encoder is competitive on most level-1 tasks but falls to \(42.33\%\) on the out-of-distribution MATH500 set, where several LLM-feature methods remain above \(90\%\).}
Furthermore, out-of-distribution generalization varies substantially, as evidenced by the performance difference between GSM8K and MATH500 for mathematical tasks, suggesting that different methods might capture distinct aspects of domain characteristics.

Meanwhile, task 2 reveals a clear limitation of \normstat: despite training-based methods and \vecstat~ saturate on the classification, \normstat~ degrades substantially.
This result is consistent with \cref{thm:radial-vs-coord} that directional information in the neural feature space is critical for within-domain discrimination: since programming languages share the same broad ``code'' domain, their feature distributions differ primarily in direction rather than radial scale, making norm-only statistics insufficient.

\paragraph{Level-2 classification: Training-based methods are more accurate.}
Fine-grained mathematical subfield classification reveals a clear gap between training-based and training-free methods (Table~\ref{tab:level-2-math-results-short}). 
\revise{When averaged across subfields, both the trained LLM-feature classifiers and RoBERTa consistently outperform the training-free methods and direct LLM inference.}
However, no single method dominates uniformly: the best-performing method varies by subfield. 
These findings suggest that effective discrimination among closely related mathematical topics benefits from non-linear transformations learned through supervised training, rather than simple statistical summaries of activation distributions.

\begin{table}[htb]
\centering
\small
\caption{\small Level-2 mathematical subfield classification results. Values represent per-subfield accuracy across seven mathematical subfields. Best accuracy in each subfield is highlighted in bold.}
\label{tab:level-2-math-results-short}
\resizebox{\linewidth}{!}{%
\begin{tabular}{llccccccc}
\toprule
Model & Method & Algebra & Counting \& Probability & Geometry & Intermediate Algebra & Number Theory & Prealgebra & Precalculus \\
\midrule
 \multirow{7}{*}{Qwen3-8B} & \embedavg & \textbf{\ms{73.94}{1.63}} & \ms{82.70}{0.98} & \ms{87.13}{0.56} & \textbf{\ms{81.00}{3.61}} & \ms{86.59}{4.35} & \ms{52.95}{5.24} & \textbf{\ms{89.36}{0.42}} \\
 
 & \embedlast & \ms{68.43}{1.81} & \ms{80.52}{2.03} & \ms{89.92}{1.91} & \ms{73.87}{6.12} & \ms{82.92}{5.55} & \textbf{\ms{53.58}{0.64}} & \ms{81.10}{5.72} \\
 & \linearavg & \ms{73.28}{4.60} & \ms{80.82}{3.83} & \textbf{\ms{90.29}{2.47}} & \ms{79.92}{3.87} & \ms{82.65}{0.90} & \ms{52.19}{3.70} & \ms{84.76}{0.54} \\
 \cdashline{2-9}
 & \normstatKL & \ms{21.01}{6.35} & \ms{19.03}{1.50} & \ms{40.19}{3.56} & \ms{46.38}{0.91} & \ms{73.45}{4.22} & \ms{0.73}{0.59} & \ms{31.98}{1.31} \\
 
 & \vecstatKL & \ms{36.23}{3.18} & \ms{51.24}{1.17} & \ms{35.88}{3.28} & \textbf{\ms{81.00}{1.60}} & \ms{88.89}{0.47} & \ms{1.52}{0.12} & \ms{55.15}{3.27} \\
 
 & \vecstatCos & \ms{64.46}{1.45} & \ms{61.95}{1.87} & \ms{36.25}{3.19} & \ms{74.13}{0.95} & \textbf{\ms{89.33}{0.33}} & \ms{2.06}{0.40} & \ms{55.08}{1.08} \\
 \cdashline{2-9}
 & \llmcall{0} & \ms{44.49}{0.22} & \textbf{\ms{83.22}{1.58}} & \ms{78.32}{2.46} & \ms{67.76}{0.44} & \ms{52.38}{2.21} & \ms{18.70}{1.02} & \ms{64.63}{1.63} \\
 \midrule
 \multirow{7}{*}{Qwen3-32B} & \embedavg & \textbf{\ms{73.25}{4.75}} & \ms{77.83}{2.89} & \ms{89.01}{3.74} & \textbf{\ms{84.14}{3.27}} & \ms{81.88}{2.64} & \textbf{\ms{57.92}{4.22}} & \ms{86.65}{3.02} \\
 
 & \embedlast & \ms{66.70}{8.95} & \ms{78.43}{2.02} & \ms{88.52}{4.65} & \ms{79.61}{3.33} & \ms{74.06}{11.10} & \ms{50.57}{10.68} & \ms{80.62}{3.30} \\
& \linearavg & \ms{71.97}{5.76} & \ms{79.93}{3.82} & \ms{87.86}{0.92} & \ms{82.33}{3.88} & \ms{85.06}{0.72} & \ms{52.82}{4.70} & \textbf{\ms{87.53}{1.89}} \\
\cdashline{2-9}
 
 & \normstatKL & \ms{31.99}{2.02} & \ms{21.72}{0.85} & \ms{40.56}{3.48} & \ms{44.61}{1.14} & \ms{75.53}{2.56} & \ms{0.13}{0.05} & \ms{31.44}{1.47} \\
 
 & \vecstatKL & \ms{55.76}{1.11} & \ms{58.50}{2.04} & \ms{35.94}{3.28} & \ms{77.80}{1.04} & \textbf{\ms{89.87}{0.34}} & \ms{1.46}{0.12} & \ms{52.78}{0.59} \\
 
 & \vecstatCos & \ms{71.16}{0.46} & \ms{64.42}{2.16} & \ms{37.04}{3.20} & \ms{73.32}{0.82} & \ms{89.49}{0.87} & \ms{3.00}{0.72} & \ms{58.47}{1.31} \\
\cdashline{2-9}
& \llmcall{0} & \ms{69.28}{0.36} & \textbf{\ms{86.59}{1.44}} & \textbf{\ms{91.92}{0.42}} & \ms{35.81}{1.07} & \ms{59.61}{1.31} & \ms{20.04}{0.51} & \ms{63.28}{0.12} \\

\midrule

 \revise{RoBERTa} & - & \revise{\ms{67.03}{3.37}} & \revise{\ms{81.95}{1.58}} & \revise{\ms{90.28}{1.39}} & \revise{\ms{78.59}{1.72}} & \revise{\ms{82.76}{2.61}} & \revise{\ms{51.91}{1.58}} & \revise{\ms{87.53}{1.89}}\\

\bottomrule
\end{tabular}%
}
\vspace{-1em} 
\end{table}

\paragraph{Distance metric comparison} The cosine distance variant of \vecstat{} consistently outperforms its KL divergence counterpart, particularly in task 4, suggesting that angular separation between prompt and baseline statistics better captures directional differences.

\subsection{Uncertainty Quantification for Mixed-Intent Prompts}\label{subsec:mixed-intent}

To evaluate whether the proposed methods can handle ambiguous prompts, we construct a mixed-intent dataset by interleaving samples from both math and code datasets at five known mix ratios, assessed across two prompt orderings (code-first and math-first). The probability outputs of each method are temperature-calibrated, and performance is measured by the RMSE between the predicted math probability and the true math fraction. Full experimental details are provided in \cref{appendix:mix-intent}.

\Cref{fig:mix-intent-qwen3-1.7B,tab:mix-intent-qwen3-1.7B} report results on Qwen3-1.7B; results for additional models are in \cref{fig:mix-intent-all,tab:mix-intent-calibration}. \vecstat{} provides the most accurate uncertainty estimates across both prompt orderings, achieving the lowest calibration RMSE. \embedlast{} is the weakest method overall and the most sensitive to prompt order, consistent with its reliance on the last-token embedding. In contrast, \normstat{} and \embedavg{} are more stable across orderings but their prediction curves remain flatter than the ideal diagonal, suggesting they under-react to changes in mixture ratio rather than tracking them continuously.

\begin{figure}[htb]
  \begin{minipage}[c]{0.5\linewidth}
    \centering
    \includegraphics[width=\linewidth]{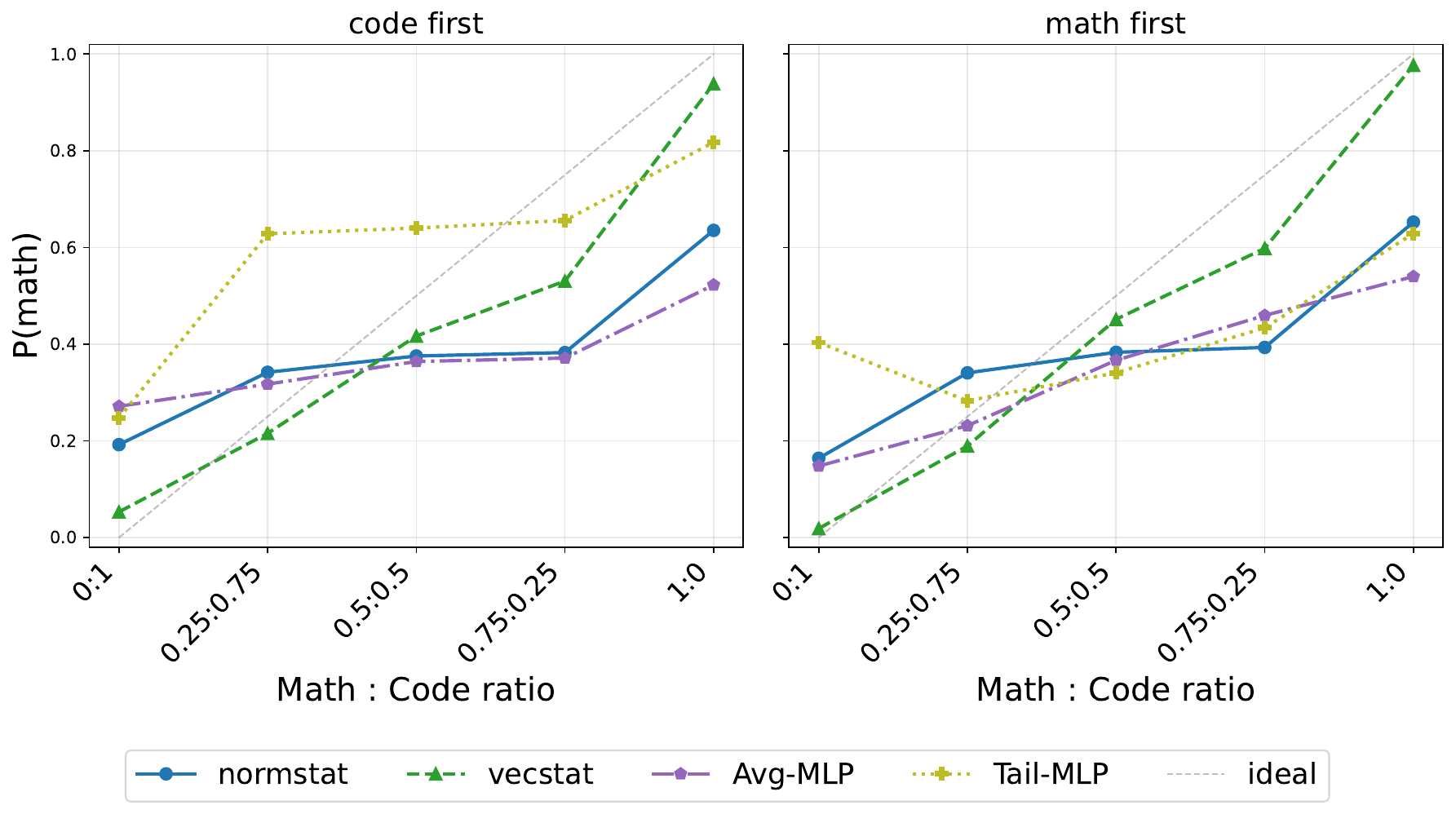}
    \captionof{figure}{\small Predicted $P(\text{math})$ vs.\ mix ratio for \textsc{Qwen3-1.7B}.
                       Curves are temperature-calibrated.}
    \label{fig:mix-intent-qwen3-1.7B}
  \end{minipage}
  \hfill
  \begin{minipage}[c]{0.45\linewidth}
    \centering
    \captionof{table}{\small Calibration RMSE for \textsc{Qwen3-1.7B}.
    $\mathrm{RMSE} = \sqrt{\frac{1}{N}\sum_{i=1}^{N}(\hat{p}_i - p_i^*)^2}$,
    where $\hat{p}_i$ is the predicted math probability and
    $p_i^*$ is the target math fraction, evaluated at mix-ratio points.            \textbf{Bold} marks the best method per column.}
    \label{tab:mix-intent-qwen3-1.7B}
    \scriptsize
    \begin{tabular}{lcc}
      \toprule
      \textbf{Method} & code-first & math-first \\
      \midrule
      \normstat  & 0.2213          & 0.2098          \\
      \vecstat   & \textbf{0.1216} & \textbf{0.0858} \\
      Avg-MLP   & 0.2451          & 0.1764          \\
      Tail-MLP  & 0.2412          & 0.2688          \\
      \bottomrule
    \end{tabular}
  \end{minipage}
  \vspace{-0.7em}
\end{figure}

\subsection{Robustness to Adversarial Attack}\label{sec:adv}

\begin{wraptable}{r}{0.559\linewidth}
\centering
\vspace{-2.4em}
\caption{\small Accuracy on adversarial datasets.}
\label{tab:adv_dataset-short}
\scriptsize
\begin{tabular}{llccc}
\toprule
Model & Method & Easy & Medium & Hard \\
\midrule
{GPT-5-Nano} & LLM Call & 93 & 64 & 34 \\
\revise{GPT-5} & \revise{LLM Call} & \revise{98} & \revise{98} & \revise{64} \\
\midrule 
\multirow{5}{*}{Qwen3-32B}
 & \embedavg      & \ms{11.73}{8.92}  & \ms{0.40}{0.69}  & {0.00} \\
 & \embedlast     & \ms{61.60}{14.67} & \ms{0.93}{0.81}  & {0.20} \\
 \cdashline{2-5}
 & \normstatKL & \textbf{\ms{92.33}{0.12}}  & \textbf{\ms{80.67}{0.12}} & {0.00} \\
 & \vecstatKL  & \ms{70.93}{0.31}  & \ms{32.33}{0.31} & {0.00} \\
 & \vecstatCos & \ms{64.87}{0.58}  & \ms{30.33}{0.64} & {0.00} \\
\midrule
\multirow{5}{*}{Qwen3-8B}
 & \embedavg       & \ms{12.67}{7.82}  & \ms{4.20}{3.30}  & {0.00} \\
 & \embedlast      & \ms{36.67}{31.67} & \ms{1.47}{2.20}  & {0.00} \\
  \cdashline{2-5}
 & \normstatKL & \ms{36.13}{0.50}  & \ms{27.00}{0.40} & {0.00} \\
 & \vecstatKL  & \ms{41.00}{0.69}  & \ms{25.00}{0.35} & {0.00} \\
 & \vecstatCos & \textbf{\ms{78.60}{1.39}}  & \textbf{\ms{54.80}{0.69}} & {0.20} \\
\bottomrule
\end{tabular}
\vspace{-1.4em}
\end{wraptable}

\paragraph{Adversarial Dataset.} 
We create three adversarial variants of MATH500 at increasing levels of camouflage—Easy (lexical), Medium (structural), and Hard (genre-level)~\footnote{\revise{The dataset is available at~\url{https://huggingface.co/datasets/nanchennn/Adv_MATH500}}}. 
Each variant disguises math problems as code-related tasks to misguide the classifier, while remaining recognizable as math problems to human readers. 
For each level, we craft a dedicated prompt (see \cref{tab:adv_prompt_templates}) and call GPT-4o to rewrite the problems accordingly. 
\revise{Furthermore, we use API-based LLM calls to GPT-5-Nano and GPT-5 to sanity-check the intended difficulty stratification.}


\paragraph{Observations.}

\cref{tab:adv_dataset-short} reports accuracy on three adversarial variants of MATH500 for two large-scale models (see \cref{tab:adv_dataset} for full results). 
Performance degrades monotonically from Easy to Medium to Hard across all methods, as expected. 
On the Hard tier, all methods collapse to near-zero accuracy, indicating a shared failure mode.
\revise{In contrast, GPT-5 achieves \(64\%\), suggesting that identifying the underlying mathematical intent despite the bug-report framing requires substantially stronger semantic understanding.}
\revise{Developing lightweight classifiers with comparable robustness represents a promising avenue for future work.}


Training-free methods are more robust than training-based ones. 
On Qwen3-8B and Qwen3-32B, the MLP-based classifiers degrade sharply, while training-free methods retain non-trivial accuracy on the Easy and Medium tiers.
We attribute this to two reasons.
First, for large models, training-free classifiers might require far fewer calibration samples than training-based ones. 
Second, training-based classifiers (MLP heads) learn a discriminative boundary directly from the token-level embedding distribution of calibration data. 
When adversarial rephrasing injects coding vocabulary into a math problem, it shifts the token distribution toward the coding class, causing the learned boundary to shift. 

\subsection{Effect of number of probed layers and prompt length}
In \cref{fig:main_nb_layers}, \vecstat{} demonstrates robust performance 
regardless of layer count, while \normstat{} exhibits dataset-dependent 
behavior, though both achieve competitive accuracy using only the first 12 
layers out of 28. These findings imply that intent classification can be 
performed without completing a full forward pass, substantially saving 
computation costs. As shown in \cref{fig:main_seqlen}, \vecstat{} maintains near-optimal accuracy across prompt lengths from 32 to 512 tokens, whereas \normstat{} is more sensitive to prompt length, plateauing at approximately 128 tokens. 
Additional results are presented in \cref{appendix-nb_layers,appendix:seqlen}.

\begin{figure}[t]
  \centering
  \begin{minipage}[t]{0.46\textwidth}
    \centering
    \includegraphics[width=\linewidth]{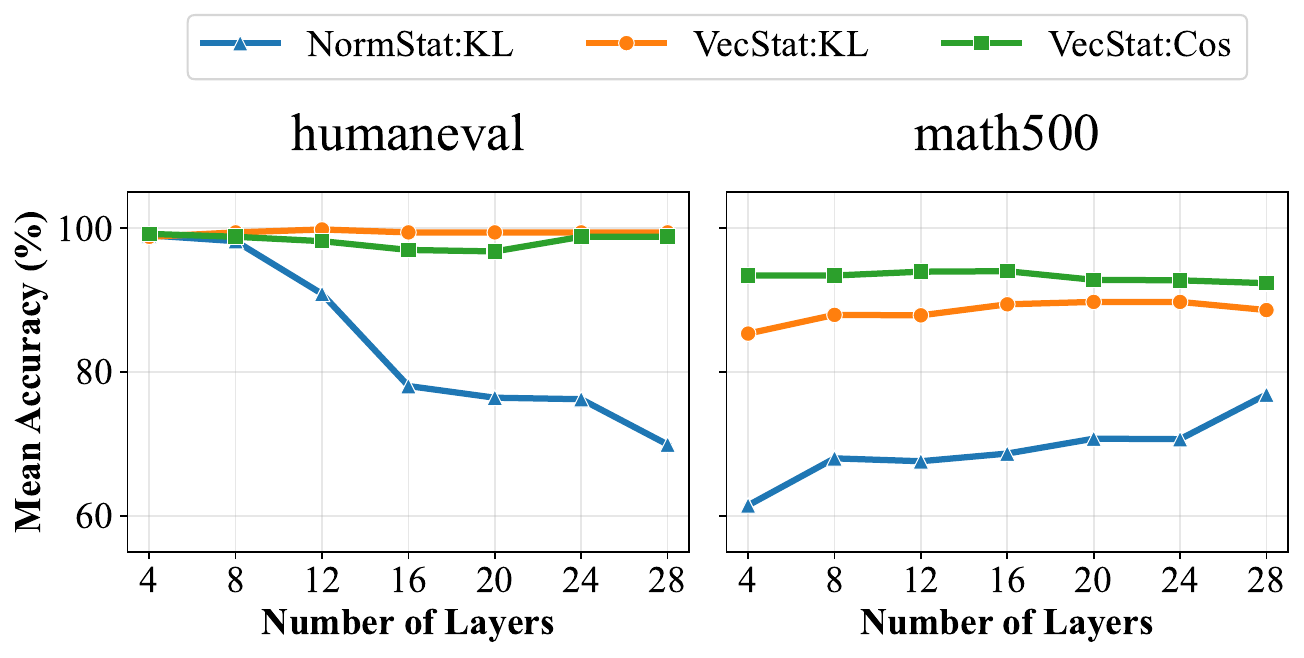}
    \caption{\small Effect of the number of layers on level-1 classification accuracy for Qwen3-1.7B.}
    \label{fig:main_nb_layers}
  \end{minipage}\hfill
  \begin{minipage}[t]{0.46\textwidth}
    \centering
    \includegraphics[width=\linewidth]{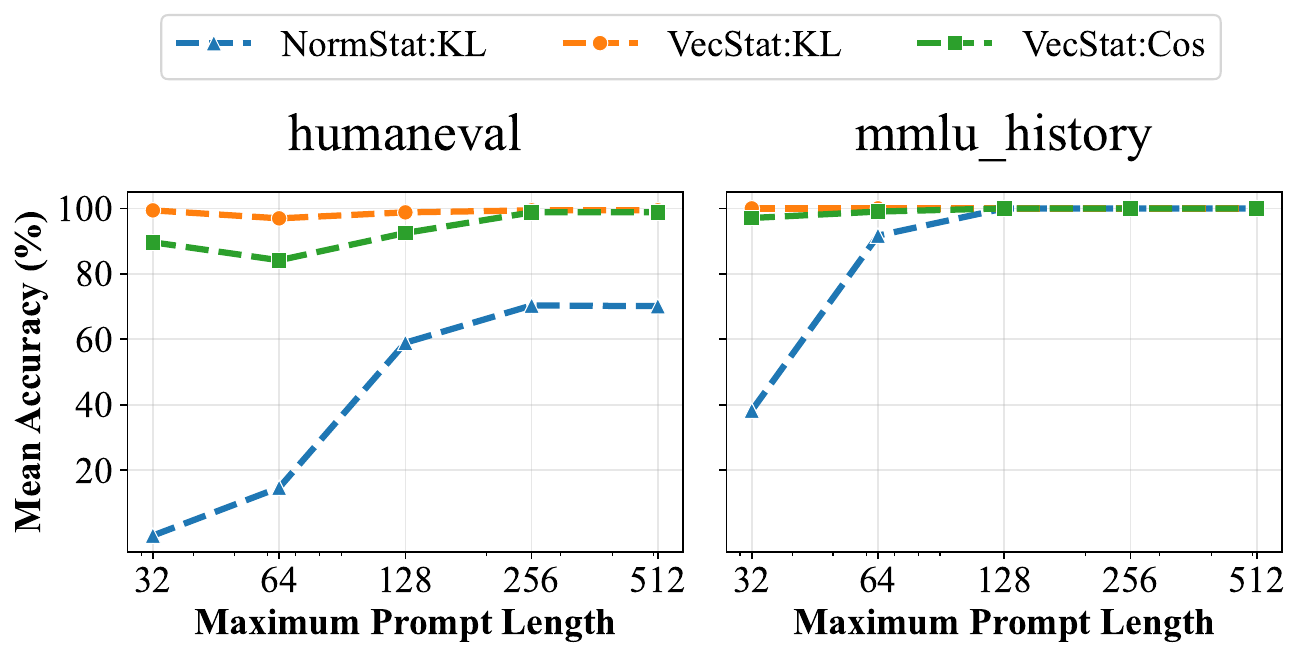}
    \caption{\small Effect of the maximum prompt length on level-1 classification for Qwen3-1.7B.}
    \label{fig:main_seqlen}
  \end{minipage}

  \vspace{-1.0em}
\end{figure}

\subsection{Calibration Analysis}

To validate~\cref{thm:calibration}, we compare the empirical and theoretical convergence rates of \normstat\ and \vecstat\ on the MagiCoder and present the result in~\cref{fig:calibration-convergence}. We vary the calibration sample size from 512 to 32768 and run with different seeds. Our results demonstrate strong match with the theoretical bounds. Both methods show $\tilde O(N^{-0.5})$ rate for the mean error, following the predicted theoretical curves.
Notably, \normstat\ attains much lower absolute errors, which is consistent with its dimension-free bound, whereas \vecstat\ sits higher due to its dimension-dependent bounds.
Calibration results for other LLMs are in~\cref{appendix:calibration}.

\begin{figure}[h]
\centering    
\begin{subfigure}[b]{0.45\textwidth}
\centering
\includegraphics[width=\textwidth]{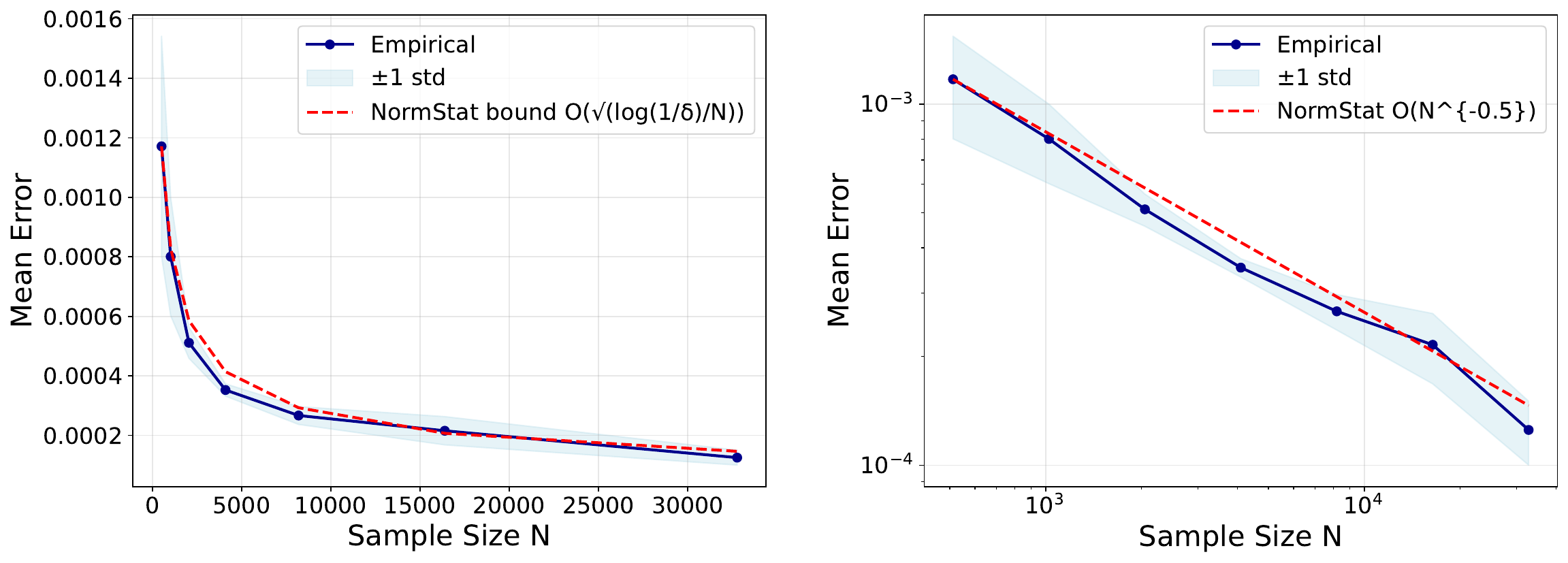}
\caption{\normstat}
\label{fig:qwen-magicoder-norm}
\end{subfigure}
\hfill
\begin{subfigure}[b]{0.45\textwidth}
\centering
\includegraphics[width=\textwidth]{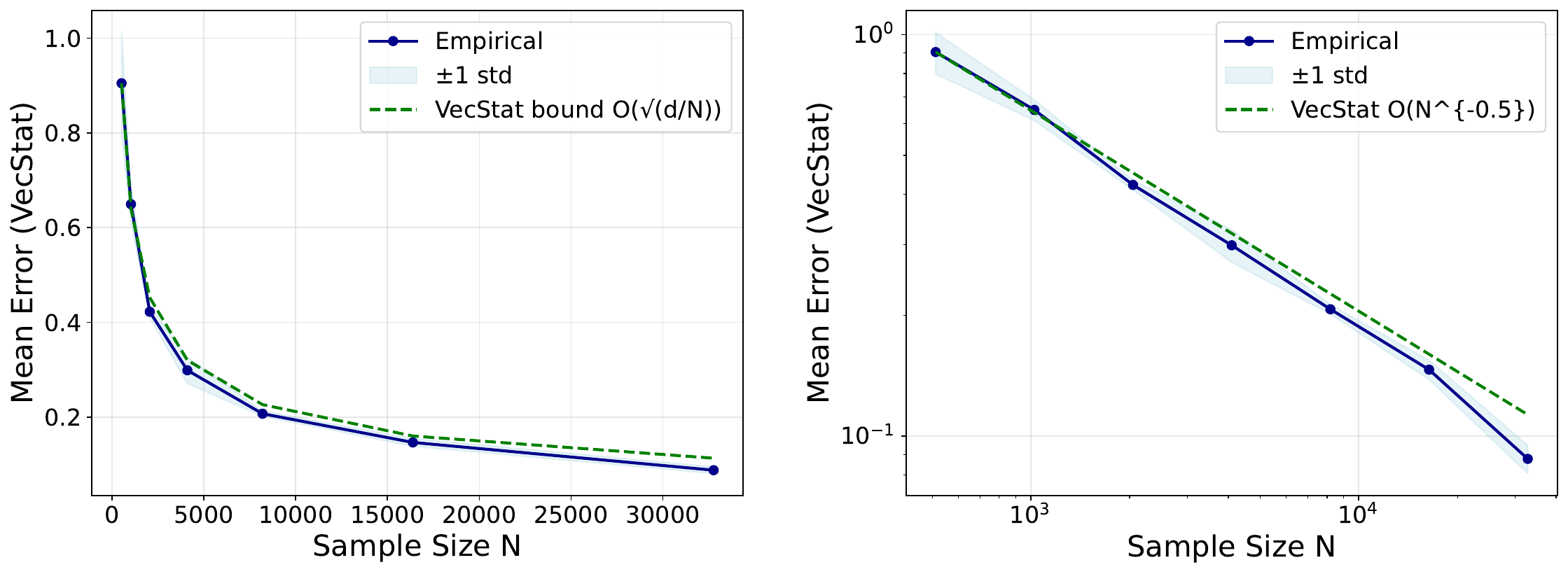}
\caption{\vecstat}
\label{fig:qwen-magicoder-proj}
\end{subfigure}
\caption{\small Calibration convergence analysis for Qwen3-8B on the MagiCoder dataset. Each subplot shows both linear and log-log scales
  comparing empirical results with theoretical bounds. }
\label{fig:calibration-convergence}
\vspace{-0.80em}
\end{figure}

\section{Conclusion}
We propose two training-free methods, \normstat{} and \vecstat{} and show that there is no one-size-fits-all method for intent classification. 
Training-based methods are stronger on fine-grained tasks, while training-free methods remain competitive on coarse-grained tasks and more robust under noisy or ambiguous prompts. 

\newpage

\section*{Acknowledgement}
\revise{
This work used GPU servers at DeltaAI at the National Center for Supercomputing Applications (NCSA) through allocation \#CIS250883 from the Advanced Cyberinfrastructure Coordination Ecosystem: Services \& Support (ACCESS) program, which is supported by U.S. National Science Foundation grants \#2138259, \#2138286, \#2138307, \#2137603, and \#2138296.
NC was funded by the NSF–Simons Research Collaboration on the Mathematical and Scientific Foundations of Deep Learning (MoDL) (NSF
DMS 2031985).}

\bibliographystyle{colm2026_conference}
\bibliography{refs}

@misc{hayou2025ploppreciseloraplacement,
      title={PLoP: Precise LoRA Placement for Efficient Finetuning of Large Models}, 
      author={Soufiane Hayou and Nikhil Ghosh and Bin Yu},
      year={2025},
      eprint={2506.20629},
      archivePrefix={arXiv},
      primaryClass={cs.LG},
      url={https://arxiv.org/abs/2506.20629}, 
}

@inproceedings{he2019using,
  title={Using convolutional neural network with BERT for intent determination},
  author={He, Changai and Chen, Sibao and Huang, Shilei and Zhang, Jian and Song, Xiao},
  booktitle={2019 International Conference on Asian Language Processing (IALP)},
  pages={65--70},
  year={2019},
  organization={IEEE}
}

@inproceedings{hashemi2016query,
  title={Query intent detection using convolutional neural networks},
  author={Hashemi, Homa B and Asiaee, Amir and Kraft, Reiner},
  booktitle={International conference on web search and data mining, workshop on query understanding},
  volume={23},
  year={2016}
}

@article{liu2023pre,
  title={Pre-train, prompt, and predict: A systematic survey of prompting methods in natural language processing},
  author={Liu, Pengfei and Yuan, Weizhe and Fu, Jinlan and Jiang, Zhengbao and Hayashi, Hiroaki and Neubig, Graham},
  journal={ACM computing surveys},
  volume={55},
  number={9},
  pages={1--35},
  year={2023},
  publisher={ACM New York, NY}
}

@inproceedings{hong2024exploring,
  title={Exploring the Use of Natural Language Descriptions of Intents for Large Language Models in Zero-shot Intent Classification},
  author={Hong, Taesuk and Ahn, Youbin and Lee, Dongkyu and Shin, Joongbo and Won, Seungpil and Han, Janghoon and Choi, Stanley Jungkyu and Seo, Jungyun},
  booktitle={Proceedings of the 25th Annual Meeting of the Special Interest Group on Discourse and Dialogue},
  pages={458--465},
  year={2024}
}

@inproceedings{souha2023pre,
  title={Pre-trained models for intent classification in chatbot: Comparative study and critical analysis},
  author={Souha, Adnane and Ouaddi, Charaf and Benaddi, Lamya and Jakimi, Abdeslam},
  booktitle={2023 6th international conference on advanced communication technologies and networking (CommNet)},
  pages={1--6},
  year={2023},
  organization={IEEE}
}

@article{arora2024intent,
  title={Intent detection in the age of llms},
  author={Arora, Gaurav and Jain, Shreya and Merugu, Srujana},
  journal={arXiv preprint arXiv:2410.01627},
  year={2024}
}

@article{turner2023activation,
  title={Activation addition: Steering language models without optimization},
  author={Turner, Alexander Matt and Thiergart, Lisa and Leech, Gavin and Udell, David and Vazquez, Juan J and Mini, Ulisse and MacDiarmid, Monte},
  journal={arXiv e-prints},
  pages={arXiv--2308},
  year={2023}
}

@article{panickssery2023steering,
  title={Steering llama 2 via contrastive activation addition},
  author={Panickssery, Nina and Gabrieli, Nick and Schulz, Julian and Tong, Meg and Hubinger, Evan and Turner, Alexander Matt},
  journal={arXiv preprint arXiv:2312.06681},
  year={2023}
}

@inproceedings{goo2018slot,
  title={Slot-gated modeling for joint slot filling and intent prediction},
  author={Goo, Chih-Wen and Gao, Guang and Hsu, Yun-Kai and Huo, Chih-Li and Chen, Tsung-Chieh and Hsu, Keng-Wei and Chen, Yun-Nung},
  booktitle={Proceedings of the 2018 Conference of the North American Chapter of the Association for Computational Linguistics: Human Language Technologies, Volume 2 (Short Papers)},
  pages={753--757},
  year={2018}
}

@article{chen2019bert,
  title={Bert for joint intent classification and slot filling},
  author={Chen, Qian and Zhuo, Zhu and Wang, Wen},
  journal={arXiv preprint arXiv:1902.10909},
  year={2019}
}

@inproceedings{larson2019evaluation,
  title={An evaluation dataset for intent classification and out-of-scope prediction},
  author={Larson, Stefan and Mahendran, Anish and Peper, Joseph J and Clarke, Christopher and Lee, Andrew and Hill, Parker and Kummerfeld, Jonathan K and Leach, Kevin and Laurenzano, Michael A and Tang, Lingjia and others},
  booktitle={Proceedings of the 2019 Conference on Empirical Methods in Natural Language Processing and the 9th International Joint Conference on Natural Language Processing (EMNLP-IJCNLP)},
  pages={1311--1316},
  year={2019}
}

@article{casanueva2020efficient,
  title={Efficient intent detection with dual sentence encoders},
  author={Casanueva, I{\~n}igo and Tem{\v{c}}inas, Tadas and Gerz, Daniela and Henderson, Matthew and Vuli{\'c}, Ivan},
  journal={arXiv preprint arXiv:2003.04807},
  year={2020}
}

@article{gpt5,
  title={GPT-5 System Card},
  author={OpenAI},
  url={https://cdn.openai.com/gpt-5-system-card.pdf},
  year={2025}
}

@inproceedings{cunningham2024sae,
  title   = {Sparse Autoencoders Find Highly Interpretable Features in Language Models},
  author  = {Cunningham, Henry and Huben, Ryan and others},
  booktitle = {ICLR},
  year    = {2024},
  url     = {https://proceedings.iclr.cc/paper_files/paper/2024/file/1fa1ab11f4bd5f94b2ec20e794dbfa3b-Paper-Conference.pdf}
}

@misc{gao2024scalingSAE,
  title  = {Scaling and Evaluating Sparse Autoencoders},
  author = {Leo Gao and Tom Dupr{\'e} la Tour and Henk Tillman and Gabriel Goh and Rajan Troll and Alec Radford and Ilya Sutskever and Jan Leike and Jeffrey Wu},
  year   = {2024},
  url    = {https://cdn.openai.com/papers/sparse-autoencoders.pdf}
}

@article{bunk2020diet,
  title={Diet: Lightweight language understanding for dialogue systems},
  author={Bunk, Tanja and Varshneya, Daksh and Vlasov, Vladimir and Nichol, Alan},
  journal={arXiv preprint arXiv:2004.09936},
  year={2020}
}

@article{elhage2022toy,
  title={Toy models of superposition},
  author={Elhage, Nelson and Hume, Tristan and Olsson, Catherine and Schiefer, Nicholas and Henighan, Tom and Kravec, Shauna and Hatfield-Dodds, Zac and Lasenby, Robert and Drain, Dawn and Chen, Carol and others},
  journal={arXiv preprint arXiv:2209.10652},
  year={2022}
}

@article{wang2023large,
  title={Large language models are zero-shot text classifiers},
  author={Wang, Zhiqiang and Pang, Yiran and Lin, Yanbin},
  journal={arXiv preprint arXiv:2312.01044},
  year={2023}
}

@inproceedings{aguirre2025fine,
  title={Fine-Tuning Medium-Scale LLMs for Joint Intent Classification and Slot Filling: A Data-Efficient and Cost-Effective Solution for SMEs},
  author={Aguirre, Maia and M{\'e}ndez, Ariane and Del Pozo, Arantza and Torres, Mar{\'\i}a In{\'e}s and Torralbo, Manuel},
  booktitle={Proceedings of the 31st International Conference on Computational Linguistics: Industry Track},
  pages={251--262},
  year={2025}
}

@article{rodriguez2024intentgpt,
  title={Intentgpt: Few-shot intent discovery with large language models},
  author={Rodriguez, Juan A and Botzer, Nicholas and Vazquez, David and Pal, Christopher and Pedersoli, Marco and Laradji, Issam},
  journal={arXiv preprint arXiv:2411.10670},
  year={2024}
}

@article{shazeer2019fast,
  title={Fast transformer decoding: One write-head is all you need},
  author={Shazeer, Noam},
  journal={arXiv preprint arXiv:1911.02150},
  year={2019}
}

@article{ainslie2023gqa,
  title={Gqa: Training generalized multi-query transformer models from multi-head checkpoints},
  author={Ainslie, Joshua and Lee-Thorp, James and De Jong, Michiel and Zemlyanskiy, Yury and Lebr{\'o}n, Federico and Sanghai, Sumit},
  journal={arXiv preprint arXiv:2305.13245},
  year={2023}
}

@article{chang2024palu,
  title={Palu: Compressing kv-cache with low-rank projection},
  author={Chang, Chi-Chih and Lin, Wei-Cheng and Lin, Chien-Yu and Chen, Chong-Yan and Hu, Yu-Fang and Wang, Pei-Shuo and Huang, Ning-Chi and Ceze, Luis and Abdelfattah, Mohamed S and Wu, Kai-Chiang},
  journal={arXiv preprint arXiv:2407.21118},
  year={2024}
}

@article{jie2025specache,
  title={SpeCache: Speculative Key-Value Caching for Efficient Generation of LLMs},
  author={Jie, Shibo and Tang, Yehui and Han, Kai and Deng, Zhi-Hong and Han, Jing},
  journal={arXiv preprint arXiv:2503.16163},
  year={2025}
}

@incollection{lehmann2011completeness,
  title={Completeness, similar regions, and unbiased estimation-Part I},
  author={Lehmann, Erich Leo and Scheff{\'e}, Henry},
  booktitle={Selected works of EL Lehmann},
  pages={233--268},
  year={2011},
  publisher={Springer}
}

@book{lehmann1998theory,
  title={Theory of point estimation},
  author={Lehmann, Erich Leo and Casella, George},
  year={1998},
  publisher={Springer}
}

@article{tang2024pooling,
  title={Pooling and attention: What are effective designs for llm-based embedding models?},
  author={Tang, Yixuan and Yang, Yi},
  journal={arXiv preprint arXiv:2409.02727},
  year={2024}
}

@inproceedings{wang2024improving,
  author={Liang Wang and Nan Yang and Xiaolong Huang and Linjun Yang and Rangan Majumder and Furu Wei},
  title={Improving Text Embeddings with Large Language Models},
  year={2024},
  cdate={1704067200000},
  pages={11897-11916},
  url={https://doi.org/10.18653/v1/2024.acl-long.642},
  booktitle={ACL (1)}
}

@article{meng2024sfrembedding,
  title={Sfrembedding-mistral: enhance text retrieval with transfer learning},
  author={Meng, Rui and Liu, Ye and Joty, Shafiq Rayhan and Xiong, Caiming and Zhou, Yingbo and Yavuz, Semih},
  journal={Salesforce AI Research Blog},
  volume={3},
  pages={6},
  year={2024}
}

@article{neelakantan2022text,
  title={Text and code embeddings by contrastive pre-training},
  author={Neelakantan, Arvind and Xu, Tao and Puri, Raul and Radford, Alec and Han, Jesse Michael and Tworek, Jerry and Yuan, Qiming and Tezak, Nikolas and Kim, Jong Wook and Hallacy, Chris and others},
  journal={arXiv preprint arXiv:2201.10005},
  year={2022}
}

@inproceedings{muennighoff2024generative,
  title={Generative representational instruction tuning},
  author={Muennighoff, Niklas and Su, Hongjin and Wang, Liang and Yang, Nan and Wei, Furu and Yu, Tao and Singh, Amanpreet and Kiela, Douwe},
  booktitle={International Conference on Learning Representations},
  volume={2025},
  pages={45544--45613},
  year={2025}
}

@article{behnamghader2024llm2vec,
  title={Llm2vec: Large language models are secretly powerful text encoders},
  author={BehnamGhader, Parishad and Adlakha, Vaibhav and Mosbach, Marius and Bahdanau, Dzmitry and Chapados, Nicolas and Reddy, Siva},
  journal={arXiv preprint arXiv:2404.05961},
  year={2024}
}

@article{muennighoff2022sgpt,
  title={Sgpt: Gpt sentence embeddings for semantic search},
  author={Muennighoff, Niklas},
  journal={arXiv preprint arXiv:2202.08904},
  year={2022}
}

@inproceedings{lee2024nv,
  title={Nv-embed: Improved techniques for training llms as generalist embedding models},
  author={Lee, Chankyu and Roy, Rajarshi and Xu, Mengyao and Raiman, Jonathan and Shoeybi, Mohammad and Catanzaro, Bryan and Ping, Wei},
  booktitle={International Conference on Learning Representations},
  volume={2025},
  pages={79310--79333},
  year={2025}
}

@inproceedings{ma2024fine,
  title={Fine-tuning llama for multi-stage text retrieval},
  author={Ma, Xueguang and Wang, Liang and Yang, Nan and Wei, Furu and Lin, Jimmy},
  booktitle={Proceedings of the 47th International ACM SIGIR Conference on Research and Development in Information Retrieval},
  pages={2421--2425},
  year={2024}
}

@article{wei2022chain,
  title={Chain-of-thought prompting elicits reasoning in large language models},
  author={Wei, Jason and Wang, Xuezhi and Schuurmans, Dale and Bosma, Maarten and Xia, Fei and Chi, Ed and Le, Quoc V and Zhou, Denny and others},
  journal={Advances in neural information processing systems},
  volume={35},
  pages={24824--24837},
  year={2022}
}

@article{li2022competition,
  title={Competition-Level Code Generation with {AlphaCode}},
  author       = {Yujia Li and
                  David H. Choi and
                  Junyoung Chung and
                  Nate Kushman and
                  Julian Schrittwieser and
                  R{\'{e}}mi Leblond and
                  Tom Eccles and
                  James Keeling and
                  Felix Gimeno and
                  Agustin Dal Lago and
                  Thomas Hubert and
                  Peter Choy and
                  Cyprien de Masson d'Autume and
                  Igor Babuschkin and
                  Xinyun Chen and
                  Po{-}Sen Huang and
                  Johannes Welbl and
                  Sven Gowal and
                  Alexey Cherepanov and
                  James Molloy and
                  Daniel J. Mankowitz and
                  Esme Sutherland Robson and
                  Pushmeet Kohli and
                  Nando de Freitas and
                  Koray Kavukcuoglu and
                  Oriol Vinyals},
  journal={Science},
  volume={378},
  number={6624},
  pages={1092--1097},
  year={2022},
  publisher={American Association for the Advancement of Science}
}

@article{guo2024deepseek,
  title={DeepSeek-Coder: When the Large Language Model Meets Programming--The Rise of Code Intelligence},
  author={Guo, Daya and Zhu, Qihao and Yang, Dejian and Xie, Zhenda and Dong, Kai and Zhang, Wentao and Chen, Guanting and Bi, Xiao and Wu, Yu and Li, YK and others},
  journal={arXiv preprint arXiv:2401.14196},
  year={2024}
}

@article{zhu2024deepseek,
  title={Deepseek-coder-v2: Breaking the barrier of closed-source models in code intelligence},
  author={Zhu, Qihao and Guo, Daya and Shao, Zhihong and Yang, Dejian and Wang, Peiyi and Xu, Runxin and Wu, Y and Li, Yukun and Gao, Huazuo and Ma, Shirong and others},
  journal={arXiv preprint arXiv:2406.11931},
  year={2024}
}

@article{yao2023tree,
  title={Tree of thoughts: Deliberate problem solving with large language models},
  author={Yao, Shunyu and Yu, Dian and Zhao, Jeffrey and Shafran, Izhak and Griffiths, Tom and Cao, Yuan and Narasimhan, Karthik},
  journal={Advances in neural information processing systems},
  volume={36},
  pages={11809--11822},
  year={2023}
}

@inproceedings{gao2023pal,
  title={Pal: Program-aided language models},
  author={Gao, Luyu and Madaan, Aman and Zhou, Shuyan and Alon, Uri and Liu, Pengfei and Yang, Yiming and Callan, Jamie and Neubig, Graham},
  booktitle={International Conference on Machine Learning},
  pages={10764--10799},
  year={2023},
  organization={PMLR}
}

@article{bocklisch2017rasa,
  title={Rasa: Open source language understanding and dialogue management},
  author={Bocklisch, Tom and Faulkner, Joey and Pawlowski, Nick and Nichol, Alan},
  journal={arXiv preprint arXiv:1712.05181},
  year={2017}
}

@inproceedings{banerjee2025llms,
  title={Llms will always hallucinate, and we need to live with this},
  author={Banerjee, Sourav and Agarwal, Ayushi and Singla, Saloni},
  booktitle={Intelligent Systems Conference},
  pages={624--648},
  year={2025},
  organization={Springer}
}

@article{bang2023multitask,
  title={A multitask, multilingual, multimodal evaluation of chatgpt on reasoning, hallucination, and interactivity},
  author={Bang, Yejin and Cahyawijaya, Samuel and Lee, Nayeon and Dai, Wenliang and Su, Dan and Wilie, Bryan and Lovenia, Holy and Ji, Ziwei and Yu, Tiezheng and Chung, Willy and others},
  journal={arXiv preprint arXiv:2302.04023},
  year={2023}
}

@article{yang2025qwen3,
  title={Qwen3 technical report},
  author={Yang, An and Li, Anfeng and Yang, Baosong and Zhang, Beichen and Hui, Binyuan and Zheng, Bo and Yu, Bowen and Gao, Chang and Huang, Chengen and Lv, Chenxu and others},
  journal={arXiv preprint arXiv:2505.09388},
  year={2025}
}

@article{dubey2024llama,
  title={The llama 3 herd of models},
  author={Dubey, Abhimanyu and Jauhri, Abhinav and Pandey, Abhinav and Kadian, Abhishek and Al-Dahle, Ahmad and Letman, Aiesha and Mathur, Akhil and Schelten, Alan and Yang, Amy and Fan, Angela and others},
  journal={arXiv e-prints},
  pages={arXiv--2407},
  year={2024}
}

@article{hendryckstest2021,
      title={Measuring Massive Multitask Language Understanding},
      author={Dan Hendrycks and Collin Burns and Steven Basart and Andy Zou and Mantas Mazeika and Dawn Song and Jacob Steinhardt},
      journal={Proceedings of the International Conference on Learning Representations (ICLR)},
      year={2021}
    }

@article{cobbe2021gsm8k,
  title={Training Verifiers to Solve Math Word Problems},
  author={Cobbe, Karl and Kosaraju, Vineet and Bavarian, Mohammad and Chen, Mark and Jun, Heewoo and Kaiser, Lukasz and Plappert, Matthias and Tworek, Jerry and Hilton, Jacob and Nakano, Reiichiro and Hesse, Christopher and Schulman, John},
  journal={arXiv preprint arXiv:2110.14168},
  year={2021}
}

@article{lightman2023lets,
      title={Let's Verify Step by Step}, 
      author={Lightman, Hunter and Kosaraju, Vineet and Burda, Yura and Edwards, Harrison and Baker, Bowen and Lee, Teddy and Leike, Jan and Schulman, John and Sutskever, Ilya and Cobbe, Karl},
      journal={arXiv preprint arXiv:2305.20050},
      year={2023}
}

@article{hendrycksmath2021,
    title={Measuring Mathematical Problem Solving With the MATH Dataset},
    author={Dan Hendrycks
    and Collin Burns
    and Saurav Kadavath
    and Akul Arora
    and Steven Basart
    and Eric Tang
    and Dawn Song
    and Jacob Steinhardt},
    journal={arXiv preprint arXiv:2103.03874},
    year={2021}
}

@article{wei2023magicoder,
  title={Magicoder: Empowering code generation with oss-instruct},
  author={Wei, Yuxiang and Wang, Zhe and Liu, Jiawei and Ding, Yifeng and Zhang, Lingming},
  journal={arXiv preprint arXiv:2312.02120},
  year={2023}
}

@misc{chen2021evaluating,
      title={Evaluating Large Language Models Trained on Code},
      author={Mark Chen and Jerry Tworek and Heewoo Jun and Qiming Yuan and Henrique Ponde de Oliveira Pinto and Jared Kaplan and Harri Edwards and Yuri Burda and Nicholas Joseph and Greg Brockman and Alex Ray and Raul Puri and Gretchen Krueger and Michael Petrov and Heidy Khlaaf and Girish Sastry and Pamela Mishkin and Brooke Chan and Scott Gray and Nick Ryder and Mikhail Pavlov and Alethea Power and Lukasz Kaiser and Mohammad Bavarian and Clemens Winter and Philippe Tillet and Felipe Petroski Such and Dave Cummings and Matthias Plappert and Fotios Chantzis and Elizabeth Barnes and Ariel Herbert-Voss and William Hebgen Guss and Alex Nichol and Alex Paino and Nikolas Tezak and Jie Tang and Igor Babuschkin and Suchir Balaji and Shantanu Jain and William Saunders and Christopher Hesse and Andrew N. Carr and Jan Leike and Josh Achiam and Vedant Misra and Evan Morikawa and Alec Radford and Matthew Knight and Miles Brundage and Mira Murati and Katie Mayer and Peter Welinder and Bob McGrew and Dario Amodei and Sam McCandlish and Ilya Sutskever and Wojciech Zaremba},
      year={2021},
      eprint={2107.03374},
      archivePrefix={arXiv},
      primaryClass={cs.LG}
}

@misc{singh2024aya,
      title={Aya Dataset: An Open-Access Collection for Multilingual Instruction Tuning}, 
      author={Shivalika Singh and Freddie Vargus and Daniel Dsouza and Börje F. Karlsson and Abinaya Mahendiran and Wei-Yin Ko and Herumb Shandilya and Jay Patel and Deividas Mataciunas and Laura OMahony and Mike Zhang and Ramith Hettiarachchi and Joseph Wilson and Marina Machado and Luisa Souza Moura and Dominik Krzemiński and Hakimeh Fadaei and Irem Ergün and Ifeoma Okoh and Aisha Alaagib and Oshan Mudannayake and Zaid Alyafeai and Vu Minh Chien and Sebastian Ruder and Surya Guthikonda and Emad A. Alghamdi and Sebastian Gehrmann and Niklas Muennighoff and Max Bartolo and Julia Kreutzer and Ahmet Üstün and Marzieh Fadaee and Sara Hooker},
      year={2024},
      eprint={2402.06619},
      archivePrefix={arXiv},
      primaryClass={cs.CL}
}

@inproceedings{lhoest-etal-2021-datasets,
    title = "Datasets: A Community Library for Natural Language Processing",
    author = "Lhoest, Quentin  and
      Villanova del Moral, Albert  and
      Jernite, Yacine  and
      Thakur, Abhishek  and
      von Platen, Patrick  and
      Patil, Suraj  and
      Chaumond, Julien  and
      Drame, Mariama  and
      Plu, Julien  and
      Tunstall, Lewis  and
      Davison, Joe  and
      {\v{S}}a{\v{s}}ko, Mario  and
      Chhablani, Gunjan  and
      Malik, Bhavitvya  and
      Brandeis, Simon  and
      Le Scao, Teven  and
      Sanh, Victor  and
      Xu, Canwen  and
      Patry, Nicolas  and
      McMillan-Major, Angelina  and
      Schmid, Philipp  and
      Gugger, Sylvain  and
      Delangue, Cl{\'e}ment  and
      Matussi{\`e}re, Th{\'e}o  and
      Debut, Lysandre  and
      Bekman, Stas  and
      Cistac, Pierric  and
      Goehringer, Thibault  and
      Mustar, Victor  and
      Lagunas, Fran{\c{c}}ois  and
      Rush, Alexander  and
      Wolf, Thomas",
    booktitle = "Proceedings of the 2021 Conference on Empirical Methods in Natural Language Processing: System Demonstrations",
    month = nov,
    year = "2021",
    address = "Online and Punta Cana, Dominican Republic",
    publisher = "Association for Computational Linguistics",
    url = "https://aclanthology.org/2021.emnlp-demo.21",
    pages = "175--184",
    eprint={2109.02846},
    archivePrefix={arXiv},
    primaryClass={cs.CL},
}

@article{hendrycks2021ethics,
      title={Aligning AI With Shared Human Values},
      author={Dan Hendrycks and Collin Burns and Steven Basart and Andrew Critch and Jerry Li and Dawn Song and Jacob Steinhardt},
      journal={Proceedings of the International Conference on Learning Representations (ICLR)},
      year={2021}
    }

@inproceedings{Jiang2024scaling,
  author={Ting Jiang and Shaohan Huang and Zhongzhi Luan and Deqing Wang and Fuzhen Zhuang},
  title={Scaling Sentence Embeddings with Large Language Models},
  year={2024},
  cdate={1704067200000},
  pages={3182-3196},
  url={https://aclanthology.org/2024.findings-emnlp.181},
  booktitle={EMNLP (Findings)}
}

@article{tao2024llms,
  title={Llms are also effective embedding models: An in-depth overview},
  author={Tao, Chongyang and Shen, Tao and Gao, Shen and Zhang, Junshuo and Li, Zhen and Hua, Kai and Hu, Wenpeng and Tao, Zhengwei and Ma, Shuai},
  journal={arXiv preprint arXiv:2412.12591},
  year={2024}
}

@article{nie2024text,
  title={When text embedding meets large language model: a comprehensive survey},
  author={Nie, Zhijie and Feng, Zhangchi and Li, Mingxin and Zhang, Cunwang and Zhang, Yanzhao and Long, Dingkun and Zhang, Richong},
  journal={arXiv preprint arXiv:2412.09165},
  year={2024}
}

@article{alain2016understanding,
  title={Understanding intermediate layers using linear classifier probes},
  author={Alain, Guillaume and Bengio, Yoshua},
  journal={arXiv preprint arXiv:1610.01644},
  year={2016}
}

@inproceedings{hewitt2019structural,
  title={A structural probe for finding syntax in word representations},
  author={Hewitt, John and Manning, Christopher D},
  booktitle={Proceedings of the 2019 Conference of the North American Chapter of the Association for Computational Linguistics: Human Language Technologies, Volume 1 (Long and Short Papers)},
  pages={4129--4138},
  year={2019}
}

@article{jitkrittum2025universal,
  title={Universal model routing for efficient llm inference},
  author={Jitkrittum, Wittawat and Narasimhan, Harikrishna and Rawat, Ankit Singh and Juneja, Jeevesh and Wang, Congchao and Wang, Zifeng and Go, Alec and Lee, Chen-Yu and Shenoy, Pradeep and Panigrahy, Rina and others},
  journal={arXiv preprint arXiv:2502.08773},
  year={2025}
}

@article{dekoninck2024unified,
  title={A unified approach to routing and cascading for llms},
  author={Dekoninck, Jasper and Baader, Maximilian and Vechev, Martin},
  journal={arXiv preprint arXiv:2410.10347},
  year={2024}
}

@article{somerstep2025carrot,
  title={Carrot: A cost aware rate optimal router},
  author={Somerstep, Seamus and Polo, Felipe Maia and de Oliveira, Allysson Flavio Melo and Mangal, Prattyush and Silva, M{\'\i}rian and Bhardwaj, Onkar and Yurochkin, Mikhail and Maity, Subha},
  journal={arXiv preprint arXiv:2502.03261},
  year={2025}
}

@article{stripelis2024tensoropera,
  title={Tensoropera router: A multi-model router for efficient llm inference},
  author={Stripelis, Dimitris and Hu, Zijian and Zhang, Jipeng and Xu, Zhaozhuo and Shah, Alay Dilipbhai and Jin, Han and Yao, Yuhang and Avestimehr, Salman and He, Chaoyang},
  journal={arXiv preprint arXiv:2408.12320},
  year={2024}
}

@article{feng2024graphrouter,
  title={Graphrouter: A graph-based router for llm selections},
  author={Feng, Tao and Shen, Yanzhen and You, Jiaxuan},
  journal={arXiv preprint arXiv:2410.03834},
  year={2024}
}

@article{hari2023tryage,
  title={Tryage: Real-time, intelligent routing of user prompts to large language models},
  author={Hari, Surya Narayanan and Thomson, Matt},
  journal={arXiv preprint arXiv:2308.11601},
  year={2023}
}

@article{jiang2023llm,
  title={Llm-blender: Ensembling large language models with pairwise ranking and generative fusion},
  author={Jiang, Dongfu and Ren, Xiang and Lin, Bill Yuchen},
  journal={arXiv preprint arXiv:2306.02561},
  year={2023}
}

@article{wang2023fusing,
  title={Fusing models with complementary expertise},
  author={Wang, Hongyi and Polo, Felipe Maia and Sun, Yuekai and Kundu, Souvik and Xing, Eric and Yurochkin, Mikhail},
  journal={arXiv preprint arXiv:2310.01542},
  year={2023}
}

@article{aggarwal2023automix,
  title={Automix: Automatically mixing language models},
  author={Aggarwal, Pranjal and Madaan, Aman and Anand, Ankit and Potharaju, Srividya Pranavi and Mishra, Swaroop and Zhou, Pei and Gupta, Aditya and Rajagopal, Dheeraj and Kappaganthu, Karthik and Yang, Yiming and others},
  journal={arXiv preprint arXiv:2310.12963},
  year={2023}
}

@article{chen2023frugalgpt,
  title={Frugalgpt: How to use large language models while reducing cost and improving performance},
  author={Chen, Lingjiao and Zaharia, Matei and Zou, James},
  journal={arXiv preprint arXiv:2305.05176},
  year={2023}
}

@article{yue2023large,
  title={Large language model cascades with mixture of thoughts representations for cost-efficient reasoning},
  author={Yue, Murong and Zhao, Jie and Zhang, Min and Du, Liang and Yao, Ziyu},
  journal={arXiv preprint arXiv:2310.03094},
  year={2023}
}

@article{zhao2024eagle,
  title={Eagle: Efficient training-free router for multi-llm inference},
  author={Zhao, Zesen and Jin, Shuowei and Mao, Z Morley},
  journal={arXiv preprint arXiv:2409.15518},
  year={2024}
}

@article{wu2025efficient,
  title={Efficient Training-Free Online Routing for High-Volume Multi-LLM Serving},
  author={Wu, Fangzhou and Silwal, Sandeep},
  journal={arXiv preprint arXiv:2509.02718},
  year={2025}
}

\newpage
\appendix
\section{Additional Discussion}
\label{app:add_discussion}
\subsection{Training/Calibration Cost}
\label{app:training_cost}

\begin{wrapfigure}[14]{r}{0.4\textwidth} 
\vspace{-1em}
\centering
\includegraphics[width=\linewidth]{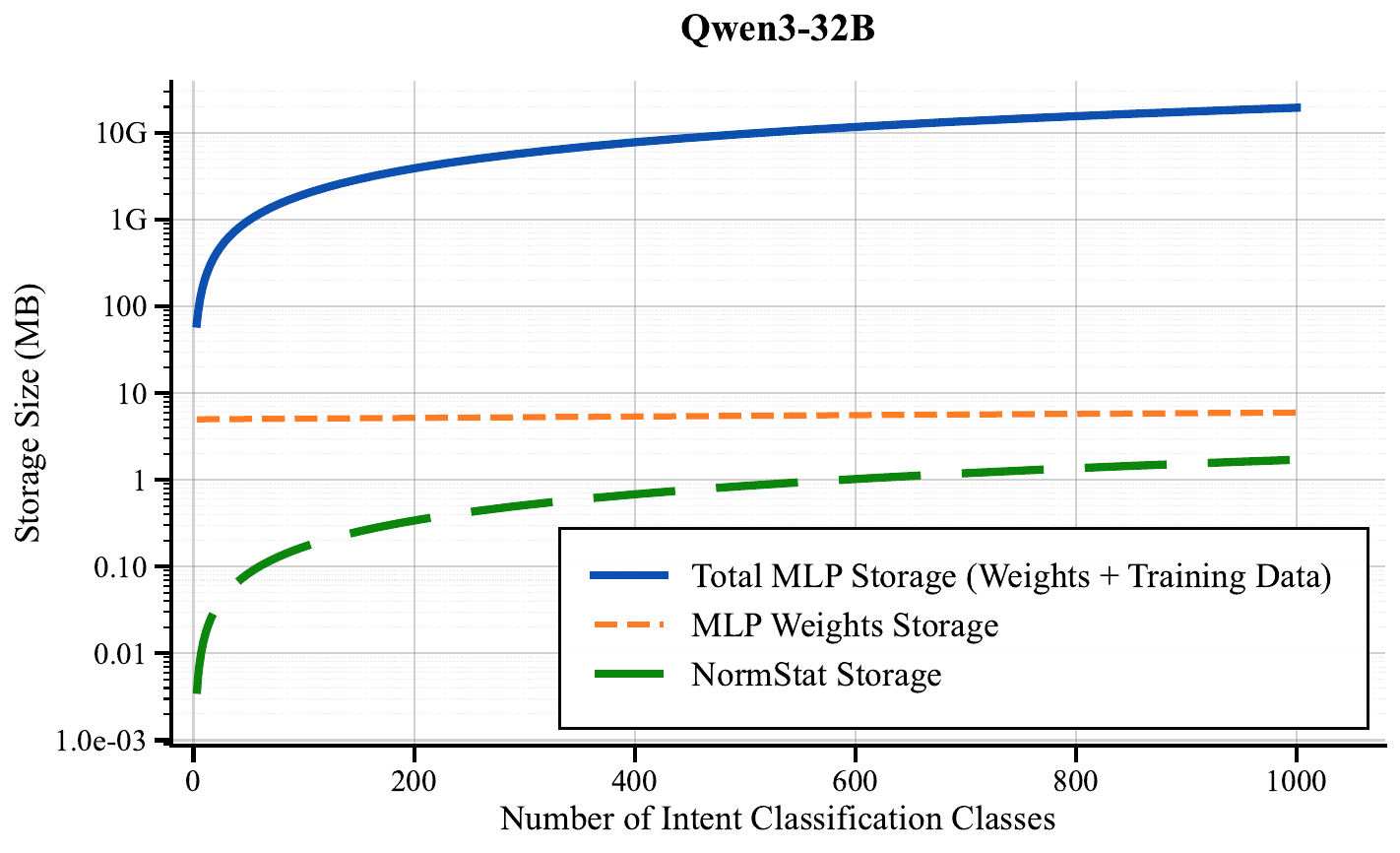}
    \caption{\small Storage simulation for intent classification on Qwen3-32B as the number of classes increases to 1,000.}
    \label{fig:storage_simulation}
      \vspace{-2em}
\end{wrapfigure}

Both training-based and training-free methods add negligible latency to intent classification and their computational overhead is minimal compared to the LLM prefill phase. The fundamental difference lies in how adaptable these methods are with evolving classification targets. When adding a new class, training-free methods maintain constant cost: \emph{computing statistics for the new class data}. On the contrary, training-based methods face a trade-off: either (i) regenerate training embeddings for all classes (both existing and new) from scratch to enable retraining—a process whose cost grows linearly with the total number of classes, or (ii) cache all training embeddings to enable rapid retraining at the expense of storage. \cref{fig:storage_simulation} quantifies the storage implications of option (ii) for Qwen3-32B: as the number of classes gradually increases to $1,000$, the total storage requirement reaches approximately $20$GB (primarily for cached embeddings), although the MLP weights alone require only $6$MB. By comparison, \normstat{} requires only 1.7MB for the same $1,000$ classes—a difference of over an order of magnitude. As a result, \emph{training-free methods are particularly well-suited for dynamic intent classification} systems where new classes are frequently added or removed.

\subsection{Inference-Time Runtime Cost}
\revise{
Let \(T\), \(d\), and \(m\) denote the prompt length, LLM hidden width, and number of intent classes, respectively. 
Given the prefill features, the inference-time costs of the two methods are
\[
\normstat = O(Td + m),
\qquad
\vecstat = O(Td + md).
\]
For both methods, the \(O(Td)\) term computes the prompt statistics. 
\normstat{} then performs \(m\) scalar class comparisons, whereas \vecstat{} compares \(d\)-dimensional coordinate-wise statistics against each class.
When \(m \ll T\), the shared \(O(Td)\) term dominates both methods, so \normstat{} and \vecstat{} have similar inference costs. 
In the opposite regime \(m \gg T\),
\vecstat{}'s \(O(md)\) scoring term dominates its statistic-computation cost, and its class-scoring cost is roughly a factor of \(d\) larger than \normstat's \(O(m)\) scoring cost.
Nevertheless, in typical routing settings where \(m \ll T(d+T)\), \vecstat{}'s scoring overhead remains small relative to the LLM prefill cost, which is \(\Theta(Td^2 + T^2d)\). 
When $m$ becomes very large, \normstat{} is preferable because its scoring cost scales only as \(O(m)\).
}

\subsection{When are training-free methods preferable?}
\label{app:training_free_discussion}
\begin{itemize}[leftmargin=1.6em]
    \item High-throughput, multi-tenant systems. A provider may host a large number of routers (per product, per customer, or per domain), where the intent label space evolves over time as new tools, experts, or domains are introduced. In such settings, every change in the class set would require retraining an MLP head, whereas VecStat/NormStat only require adding or removing per-class statistics—keeping the adaptation cost essentially constant as the system scales.
    \item Untrusted or safety-critical deployments. In systems that must handle untrusted inputs (e.g. public-facing assistants), reliable uncertainty quantification is crucial for detecting malicious or out-of-distribution prompts, and human interpretability is needed for post-hoc audits (e.g., to assess potential fairness issues). Our training-free methods directly expose calibrated per-class statistics, which can be inspected and monitored without the additional modeling and engineering complexity required to obtain well-calibrated uncertainty estimates from MLP heads. 
\end{itemize}

\subsection{Limitations in Experiment Design}
\revise{Our evaluation studies routing an input sequence at a fixed decision point and relies primarily on public benchmark datasets, since realistic production routing logs are confidential and generally unavailable. 
Our mixed-intent dataset introduces controlled ambiguity by varying the ratio and ordering of math and code content, but it does not capture the full range of underspecified, conversational, domain-overlapping, or history-dependent requests encountered in deployment. 
Although the routed sequence could in principle be a single prompt, a concatenated conversation history, or a summarized
dialogue state, our experiments instantiate it primarily with single-turn prompts and do not evaluate dynamic re-routing as a conversation evolves. 
Such re-routing may require a newly selected downstream model to process the full conversation history, adding prefill cost and latency. 
Nonetheless, evaluating these methods on realistic routing traffic and multi-turn conversations, together with the resulting trade-offs among accuracy, cost, and latency,
remains important future work.}

\section{Additional Theoretical Analysis}\label{app:theory}

\subsection{Two endpoints on the compression ladder}
For each class $k\in[m]$, per-class baselines are computed at the same module $W_\ell$ on calibration data.
We compare the two methods in terms of FLOPs and memory.

\paragraph{\vecstat.}
\textbf{Method 1:}
Compute $(S_{\mathrm{vec}},Q_{\mathrm{vec}})$ as in \eqref{eq:vecstat}.
With per-class parameters $(\mu_k,\Sigma_k=\mathrm{Diag}(\sigma_{k,1}^2,\ldots,\sigma_{k,d}^2))$.
The log-likelihood ratio (LLR) between classes $i$ and $j$ is
\begin{align}
\label{eq:LLR}
\log\frac{p_i(Y)}{p_j(Y)}
=-\frac{T}{2}\sum_{m=1}^d
\Bigg[
\log\frac{\sigma_{i,m}^2}{\sigma_{j,m}^2}
+\frac{Q_{\mathrm{vec},m}-2\mu_{i,m}S_{\mathrm{vec},m}+\mu_{i,m}^2}{\sigma_{i,m}^2}
-\frac{Q_{\mathrm{vec},m}-2\mu_{j,m}S_{\mathrm{vec},m}+\mu_{j,m}^2}{\sigma_{j,m}^2}
\Bigg],
\end{align}
which is equivalently the average of coordinate-wise Gaussian KLs (since $\Sigma_k$ is diagonal): 
\begin{align}
 \sum_{m=1}^d \KL\big(\Normal(\mu_{i,m},\sigma_{i,m}^2)\,\|\Normal(\mu_{j,m},\sigma_{j,m}^2)\big) = \frac{1}{2}\sum_{m=1}^d \left[
\log\frac{\sigma_{j,m}^2}{\sigma_{i,m}^2}
+ \frac{\sigma_{i,m}^2 + (\mu_{i,m}-\mu_{j,m})^2}{\sigma_{j,m}^2}
- 1
\right]. \label{eq:vecstat-kl-sum}
\end{align}

\noindent\textbf{Method 2:}
Using $S_{\mathrm{vec}}$, classify via
\[
\mathrm{cos}_k(S_{\mathrm{vec}},\mu_k)=\frac{\langle S_{\mathrm{vec}},\mu_k\rangle}{\|S_{\mathrm{vec}}\|\|\mu_k\|},
\qquad
\hat b=\arg\max_{k\in[m]} \ \mathrm{cos}_k(S_{\mathrm{vec}},\mu_k).
\]

\noindent\textbf{Costs:}
\emph{Per-token compute:} $\Theta(d)$.
\emph{Prompt-state:} $O(d)$.
\emph{Baseline storage:} $O(md)$ numbers.
(If only cosine scoring is used, $Q_{\mathrm{vec}}$ need not be stored.)

\medskip
\paragraph{\normstat.}
\textbf{Method:}
Compute $(S_{\text{norm}},Q_{\text{norm}})$ as in \eqref{eq:normstat}, then compare
$(S_{\text{norm}},Q_{\text{norm}})$ to per-class baselines $(\mu_{x,k},\sigma_{x,k}^2)$ via a 1D Gaussian KL:
\begin{align}
\KL\big(\Normal(\mu_{x,i},\sigma_{x,i}^2)\|\Normal(\mu_{x,j},\sigma_{x,j}^2)\big)
= \frac{1}{2}\left[
\log\frac{\sigma_{x,j}^2}{\sigma_{x,i}^2}
+ \frac{\sigma_{x,i}^2 + (\mu_{x,i}-\mu_{x,j})^2}{\sigma_{x,j}^2}
- 1
\right]. \label{eq:normstat-kl-1d}
\end{align}

\noindent\textbf{Cost.}
\emph{Per-token compute:} $\Theta(d)$.
\emph{Prompt-state:} $O(1)$.
\emph{Baseline storage:} $O(m)$ scalars.



\subsection{Sufficiency}
We establish the minimal sufficiency of $(S_{\mathrm{vec}},Q_{\mathrm{vec}})$ for the diagonal–Gaussian model, which is a classical result, see \citep{lehmann1998theory,lehmann2011completeness}. Intuitively, a sufficient statistic is a lossless compression for inference about the unknown class/parameters:
once $(S_{\mathrm{vec}},Q_{\mathrm{vec}})$ is known, the raw sample $Y$ contains no further information.
Minimal sufficiency means no additional compression is possible without losing information—every other
sufficient statistic is a measurable function of $(S_{\mathrm{vec}},Q_{\mathrm{vec}})$.

\begin{lemma}
\label{lemma:sufficiency}
Under \eqref{eq:model} with diagonal $\Sigma_k$, the pair $(S_{\mathrm{vec}},Q_{\mathrm{vec}})$ is a \emph{minimal sufficient statistic}
for the family $\{\normal(\mu_k,\Sigma_k)^{\otimes T}\}$, and the class LLR
\eqref{eq:LLR}
depends on the data only through $(S_{\mathrm{vec}},Q_{\mathrm{vec}})$.
\end{lemma}

The proof is provided in \cref{app:proofs}.




\section{Proofs}\label{app:proofs}
\subsection{Proof of \cref{lemma:sufficiency}}
\begin{proof}[Proof of Lemma~\ref{lemma:sufficiency}]

Let $Y\coloneqq(y_1,\ldots,y_T)\in\mathbb{R}^{T\times d}$. For a class $k$, let
\[
\theta_k=(\mu_k,\sigma_{k,1}^2,\ldots,\sigma_{k,d}^2),\qquad
\Sigma_k=\operatorname{Diag}(\sigma_{k,1}^2,\ldots,\sigma_{k,d}^2).
\] Define the token-wise \emph{sums}
\[
S\coloneqq\sum_{t=1}^T y_t\in\mathbb{R}^d,\qquad
Q\coloneqq\sum_{t=1}^T \big(y_t\odot y_t\big)\in\mathbb{R}^d,
\]
and write $S_m\coloneqq\sum_{t=1}^T y_{t,m}$, $Q_m\coloneqq\sum_{t=1}^T y_{t,m}^2$ for coordinates $m=1,\dots,d$.
These relate to the averaged statistics in \eqref{eq:vecstat} as follows:
\[
S = TS_{\mathrm{vec}},
\qquad
Q =(T-1)Q_{\mathrm{vec}} + T\big(S_{\mathrm{vec}}\odot S_{\mathrm{vec}}\big)
\]
Since (minimal) sufficiency is invariant under invertible reparameterizations of the statistic,
we may work with $(S,Q)$ and translate back to $(S_{\mathrm{vec}},Q_{\mathrm{vec}})$ via the identities above.

\noindent\textbf{Sufficiency.}
Note that $y_t$ are i.i.d.\ with density
\[
p_{\theta_k}(y)=\frac{1}{(2\pi)^{d/2}\prod_{m=1}^d \sigma_{k,m}}
\exp\left(-\frac12\sum_{m=1}^d \frac{(y_m-\mu_{k,m})^2}{\sigma_{k,m}^2}\right).
\]
Hence the joint density of $Y$ under class $k$ is
\begin{align*}
p_{\theta_k}(Y)
&=\frac{1}{(2\pi)^{Td/2}\prod_{m=1}^d \sigma_{k,m}^T}
\exp\left(-\frac12\sum_{t=1}^T\sum_{m=1}^d \frac{(y_{t,m}-\mu_{k,m})^2}{\sigma_{k,m}^2}\right)\\
&=\underbrace{\frac{1}{(2\pi)^{Td/2}}}_{\coloneqq h(Y)}\cdot\underbrace{\frac{1}{\prod_{m=1}^d \sigma_{k,m}^T}
\exp\left(-\frac12\sum_{m=1}^d \frac{Q_m-2\mu_{k,m}S_m+T\mu_{k,m}^2}{\sigma_{k,m}^2}\right)}_{\coloneqq g_{\theta_k}(S,Q)}.
\end{align*}
Thus $p_{\theta_k}(Y)=h(Y)\,g_{\theta_k}(S,Q)$. By the Neyman--Fisher factorization theorem, $(S,Q)$ is sufficient for $\theta_k$, and hence
$(S_{\mathrm{vec}},Q_{\mathrm{vec}})$ is sufficient by the invertible mapping above.

\noindent\textbf{Minimality.}
Using the Lehmann--Scheffé characterization: a statistic $T(Y)$ is minimal sufficient iff for any $Y,Y'$ the likelihood ratio
$p_{\theta}(Y)/p_{\theta}(Y')$ is free of $\theta$ if and only if $T(Y)=T(Y')$.
For our family,
\[
\frac{p_{\theta}(Y)}{p_{\theta}(Y')}
=\exp\left(-\frac12\sum_{m=1}^d \frac{(Q_m-Q'_m)-2\mu_m(S_m-S'_m)}{\sigma_m^2}\right).
\]
If $(S,Q)=(S',Q')$ then this ratio equals $1$, hence is parameter--free. Conversely, if for some $m$ either
$S_m\neq S'_m$ or $Q_m\neq Q'_m$, the exponent depends on $\mu_m$ (when $S_m\neq S'_m$) or on $\sigma_m^2$
(when $Q_m\neq Q'_m$); thus the ratio cannot be constant in $\theta$.

Moreover, for classes $i$ and $j$, subtracting the two log-likelihoods above yields
\[
\log\frac{p_i(Y)}{p_j(Y)}
=-\frac12\sum_{m=1}^d\left(
T\log\frac{\sigma_{i,m}^2}{\sigma_{j,m}^2}
+\frac{Q_m-2\mu_{i,m}S_m+T\mu_{i,m}^2}{\sigma_{i,m}^2}
-\frac{Q_m-2\mu_{j,m}S_m+T\mu_{j,m}^2}{\sigma_{j,m}^2}
\right),
\]
which is exactly \eqref{eq:LLR} and depends on $Y$ only through $(S,Q)$, and equivalently only through $(S_{\mathrm{vec}},Q_{\mathrm{vec}})$
via the identities at the start of the proof.

This establishes that $(S_{\mathrm{vec}},Q_{\mathrm{vec}})$ is minimal sufficient and that the LLR depends on
the sample only through this pair.

\end{proof}

\subsection{Proof of \cref{thm:radial-vs-coord}}
\begin{proof}[Proof of Theorem~\ref{thm:radial-vs-coord}]
\textbf{Directional regime:}
Since $\|\mu_1\|=\|\mu_2\|$, there exists an orthogonal matrix $U$ with $U\mu_1=\mu_2$.
If $Y\sim\mathcal{N}(\mu_1,\sigma^2 I_d)$ then $UY\sim\mathcal{N}(\mu_2,\sigma^2 I_d)$ and $\|UY\|=\|Y\|$.
Thus for each $t$, $\|y_t\|\mid k=1$ and $\|y_t\|\mid k=2$ have the same distribution, and by independence the vectors
$(\|y_t\|)_{t=1}^T\mid k=1$ and $(\|y_t\|)_{t=1}^T\mid k=2$ are identically distributed.
With a uniform prior, any decision rule that depends only on $\{\|y_t\|\}$ has the same acceptance probability under both classes,
so its Bayes error is $1/2$.

For the log-likelihood ratio test (LRT), the log-likelihood ratio for two Gaussians with common covariance $\sigma^2 I_d$ is
\[
\Lambda(y_{1:T})
= \frac{T}{\sigma^2}\,\big\langle S_{\mathrm{vec}},\,\mu_1-\mu_2\big\rangle
  - \frac{T}{2\sigma^2}\big(\|\mu_1\|^2-\|\mu_2\|^2\big).
\]
With equal priors the LRT accepts $k=1$ iff $\Lambda\ge 0$.
Under $\|\mu_1\|=\|\mu_2\|$, the constant term vanishes and the decision reduces to the sign of
$\langle S_{\mathrm{vec}},\,\mu_1-\mu_2\rangle$, i.e., to $\hat{k}$ above.

Let $u\coloneqq(\mu_1-\mu_2)/\|\mu_1-\mu_2\|$ and $Z\coloneqq\langle S_{\mathrm{vec}},u\rangle$.
Since $S_{\mathrm{vec}}\mid k \sim \mathcal N(\mu_k,\tfrac{\sigma^2}{T}I_d)$ and $\|u\|=1$,
\[
Z\mid k \sim \mathcal N\!\left(\langle \mu_k,u\rangle,\,\frac{\sigma^2}{T}\right),
\quad
\langle \mu_1,u\rangle=\tfrac12\|\mu_1-\mu_2\|,\ \
\langle \mu_2,u\rangle=-\tfrac12\|\mu_1-\mu_2\|.
\]
Hence, by symmetry,
\[
\Pr\big(\hat{k}(y_{1:T})\neq k\big)
= \Pr_{k=1}(Z<0)
= \Phi\!\left(-\frac{\|\mu_1-\mu_2\|}{2\sigma}\sqrt{T}\right)
\le \exp\!\left(-\frac{T}{8\sigma^2}\,\|\mu_1-\mu_2\|^2\right),
\]
where $\Phi$ is the standard normal CDF and the last step uses $\Phi(-x)\le e^{-x^2/2}$ for $x\ge 0$.

\noindent\textbf{Isotropic-scale regime.}
Let $\phi_d(\,\cdot\,; m, \Sigma)$ denote the $d$-variate Gaussian density.
For $k\in\{1,2\}$, the \emph{joint} density of $y_{1:T}$ under class $k$ is
\(p_k(y_{1:T}) \coloneqq \prod_{t=1}^T \phi_d\left(y_t;  0, \sigma_k^2 I_d\right).\)

With $\mu_1=\mu_2=0$, $\Sigma_k=\sigma_k^2 I_d$, and $R_T= \sum_{t=1}^T \|y_t\|^2$,  one can calculate the log-likelihood ratio 
\begin{align*}
\log\frac{p_1(y_{1:T})}{p_2(y_{1:T})}
= \sum_{t=1}^T \log\frac{\phi_d(y_t;0,\sigma_1^2 I_d)}{\phi_d(y_t;0,\sigma_2^2 I_d)}
= \frac{dT}{2}\log\frac{\sigma_2^2}{\sigma_1^2}
   + \frac{1}{2}\left(\frac{1}{\sigma_2^2}-\frac{1}{\sigma_1^2}\right) R_T.
\end{align*}
The right-hand side is an affine (hence strictly monotone when $\sigma_1\neq\sigma_2$)
function of $R_T$. By the Neyman--Pearson lemma, any Bayes--optimal test is a threshold on $R_T$,
so purely radial statistics are sufficient for optimality and coordinate-wise additions cannot lower the Bayes risk.
\end{proof}

\subsection{Proof of \cref{thm:calibration}}

\begin{thm}
Fix a class $k$. Let $y_1,\dots,y_N \stackrel{\text{i.i.d.}}{\sim} \mathcal N(\mu_k,\Sigma_k)$ in $\mathbb R^d$,
where $N$ is the number of calibration samples drawn for this class.
Define
\[
\hat\mu_k \coloneqq \frac{1}{N}\sum_{i=1}^N y_i,\quad
\hat q \coloneqq \frac{1}{N}\sum_{i=1}^N \|y_i\|^2,
\quad
\sigma_{\max}^2 \coloneqq \|\Sigma_k\|_{\mathrm{op}}.
\]
For coordinate variances, write $\sigma_{k,j}^2 \coloneqq (\Sigma_k)_{jj}$ and
\(
\widehat{\sigma}_{k,j}^2 \coloneqq \frac{1}{N}\sum_{i=1}^N (y_{i,j}-\hat\mu_{k,j})^2,\ j=1,\dots,d.
\)
Then:
\begin{enumerate}
\item \textbf{\normstat\ (dimension-free).} For $q=\|y\|^2$, one has
\[
\E[q]=\|\mu_k\|^2+\Tr(\Sigma_k),
\quad
\Var(q)=2\Tr(\Sigma_k^2)+4\mu_k^\top\Sigma_k\mu_k.
\]
By Bernstein’s inequality for sub-exponential variables, for all $\delta\in(0,1)$,
\[
\left|\hat q-q\right|\lesssim \sqrt{\frac{\Var(q)\log(1/\delta)}{N}}
\quad \text{with probability at least }1-\delta.
\]
Normalizing by $d$ makes the bound $O(\sqrt{\log(1/\delta)/N})$, i.e. dimension-free.
\item \textbf{\vecstat\ (dimension-dependent).} With probability at least  $1-\delta$,
\[
\|\hat\mu_k-\mu_k\|_2 \le C_1 \sigma_{\max}\sqrt{\frac{d+\log(1/\delta)}{N}}, \quad
\max_j \left|\widehat{\sigma}_{k,j}^2-\sigma_{k,j}^2\right|\leq C_2 \sigma_{\max}^2\sqrt{\frac{\log(d/\delta)}{N}},
\]
for absolute constants $C_1,C_2$.
To keep LLR plug-in error small of order $\epsilon$, one needs
$N=\Omega(d/\epsilon_\mu^2)$ for mean accuracy in $\ell_2$ and $N=\Omega(\log d/\epsilon_\sigma^2)$ for variances in $\ell_\infty$.
\end{enumerate}
\end{thm}
\begin{proof}[Proof of Theorem~\ref{thm:calibration}]
\textbf{\normstat:}
Denote $q_i\coloneqq \|y_i\|^2$ and $Z_i\coloneqq q_i-\mathbb E[q]$, so that
$\hat q-\mathbb E[q]=\frac{1}{N}\sum_{i=1}^N Z_i$.
For $y_i\sim\mathcal N(\mu_k,\Sigma_k)$, the centered quadratic form $Z_i$ obeys the Hanson--Wright
tail bound:
there exist absolute constants $c_1,c_2>0$ such that for all $t>0$,
\begin{equation}\label{eq:HW}
\Pr\left(|Z_i|\ge t\right)\ \le\ 2\exp  \left[-\,c_1\,
\min  \left(\frac{t^2}{\Var(q)},\ \frac{t}{B}\right)\right],
\end{equation}
where \(\Var(q)=2\Tr(\Sigma_k^2)+4\,\mu_k^\top\Sigma_k\mu_k,\ B=\|\Sigma_k\|_{\mathrm{op}}+\|\mu_k\|^2 .\)
From \eqref{eq:HW}, the $Z_i$ are i.i.d.\ mean-zero sub-exponential.
A standard Bernstein inequality for sums of independent sub-exponential variables then yields, for some absolute $c>0$ and all $t>0$,
\[
\Pr \left(\left|\frac{1}{N}\sum_{i=1}^N Z_i\right|\ge t\right)
\le 2\exp\left[-\,cN\,
\min\left(\frac{t^2}{\Var(q)},\ \frac{t}{B}\right)\right].
\]
Choosing $t\lesssim \Var(q)/B$ and inverting the tail gives, for any $\delta\in(0,1)$,
\[
\left|\hat q-\mathbb E[q]\right|
\ \lesssim\ \sqrt{\frac{\Var(q)\log(2/\delta)}{N}}
\quad\text{with probability at least }1-\delta.
\]
Since under bounded eigenvalues $\Var(q)=\Theta(d)$, dividing by $d$ yields
\(
\left|\frac{1}{d}\hat q-\frac{1}{d}\mathbb E[q]\right|
 \lesssim \sqrt{\tfrac{\log(2/\delta)}{N}},
\)
which is dimension-free.

\noindent\textbf{\vecstat:}
Let $z_i\coloneqq\Sigma_k^{-1/2}(y_i-\mu_k)\sim \mathcal N(0,I_d)$.
Then
\[
\hat\mu_k-\mu_k
= \Sigma_k^{1/2}\left(\frac{1}{N}\sum_{i=1}^N z_i\right)
\sim \mathcal N \left(0,\ \frac{1}{N}\Sigma_k\right).
\]
Hence, \(\|\Sigma_k^{-1/2}(\hat\mu_k-\mu_k)\|_2^2
\sim \frac{1}{N} \chi^2_d\).
Recall the standard Laurent-Massart inequalities: for any $x>0$,
\begin{equation}\label{chi2-tail}
\Pr\left(\chi^2_d - d \ge 2\sqrt{dx}+2x  \right)\le e^{-x},
\quad
\Pr\left(d-\chi^2_d \ge 2\sqrt{dx}  \right)\le e^{-x}.
\end{equation}
Applying \eqref{chi2-tail} with $x=\log(1/\delta)$ and scaling by $1/N$ yields, with probability $\ge 1-\delta$,
\[
\|\Sigma_k^{-1/2}(\hat\mu_k-\mu_k)\|_2
\le \sqrt{\frac{1}{N}\left(d+2\sqrt{d\log(1/\delta)}+2\log(1/\delta)  \right)}.
\]
Multiplying by $\|\Sigma_k^{1/2}\|_{\mathrm{op}}=\sigma_{\max}$ and using
$\sqrt{a+b}\le \sqrt a + \sqrt b$ gives
\[
\|\hat\mu_k-\mu_k\|_2
\leq C_1 \sigma_{\max} \sqrt{\frac{d+\log(1/\delta)}{N}},
\]
for an absolute constant $C_1>0$.

\end{proof}

\section{Experiment Details}\label{appendix:exp-details}

\subsection{Dataset Composition and Processing Strategies}
\label{appendix:datatset_composition}

\begin{table}[ht]
\centering
\caption{\small Dataset composition for Task~1: general intent classification between general text, mathematics, and code (level-1 granularity).}
\label{tab:level_1_data_stat}
\resizebox{0.98\linewidth}{!}{%
\begin{tabular}{lllll}
\toprule
\textbf{Category} & \textbf{Calibration Data} & \textbf{\# Calibration Samples} & \textbf{Classification Data} & \textbf{\# Classification Samples} \\
\midrule
General Text & MMLU (European History) & 165 & MMLU (US History) & 204 \\
\midrule
\multirow{2}{*}{Math} & \multirow{2}{*}{GSM8K} & \multirow{2}{*}{2,000} & GSM8K & 1,319 \\
 &  &  & MATH500 & 500 \\
\midrule
\multirow{2}{*}{Code} & \multirow{2}{*}{Magicoder} & \multirow{2}{*}{2,000} & Magicoder & 5,000 \\
 &  &  & HumanEval & 164 \\
\bottomrule
\end{tabular}
}
\end{table}

\begin{table}[ht]
\centering
\caption{\small Dataset composition for Tasks~2 and~3: programming language identification and natural language identification (both are level-1 granularity).}
\label{tab:level_1_programming}
\resizebox{0.98\linewidth}{!}{
\begin{tabular}{ll>{\raggedright\arraybackslash}p{5cm}ll}
\toprule
\textbf{Task} & \textbf{Data Source} & \textbf{Classes} & \textbf{\# Calibration Samples} & \textbf{\# Classification Samples} \\
 &  & & \textbf{per Class} & \textbf{per Class} \\
\midrule
Task~2 & Magicoder & C++, C\#, Java, PHP, Python, Rust, Shell, Swift, TypeScript & 2,000 & 5,000 \\
\midrule
Task~3 & Aya & Sinhala, Tamil, English, Moroccan Arabic, Japanese & 512 & 3,000 \\
\bottomrule
\end{tabular}
}
\end{table}

\begin{table}[ht]
\centering
\caption{\small Dataset composition for Task~4: mathematical subfield classification (level-2 granularity).}
\label{tab:level_2_data_stat}
\resizebox{0.98\linewidth}{!}{
\begin{tabular}{ll>{\raggedright\arraybackslash}p{5cm}ll}
\toprule
\textbf{Task} & \textbf{Data Source} & \textbf{Classes} & \textbf{\# Calibration Samples} & \textbf{\# Classification Samples} \\
 &  & & \textbf{per Class} & \textbf{per Class} \\
\midrule
Task~4 & Competition Math & Algebra, Counting \& Probability, Geometry, Intermediate Algebra, Number Theory, Prealgebra, Precalculus & 800 & 3,000 \\
\bottomrule
\end{tabular}
}
\end{table}

We evaluate our intent classification methods across two levels of granularity. The experimental protocol consists of two stages: calibration and classification. During calibration, we compute per-class baseline statistics (\normstat{} or \vecstat{}) from calibration data passed through pretrained LLMs. During classification, we compute the same statistics for test prompts and assign labels based on the minimum KL divergence (or cosine distance) between the prompt's statistics and the calibrated per-class baselines.

\paragraph{Classification tasks and granularities.}
We consider four classification tasks spanning two levels of granularity.
\begin{itemize}[leftmargin=1.6em]
    \item \textbf{Task~1} (level-1): general intent classification among three broad domains—general text, mathematics, and code. This task represents the canonical routing scenario where prompts are directed to specialized models based on broad category.
    \item \textbf{Task~2} (level-1): programming language identification across nine languages within code prompts.
    \item \textbf{Task~3} (level-1): natural language identification across five linguistically diverse languages.
    \item \textbf{Task~4} (level-2): fine-grained mathematical subfield classification across seven topics within mathematics.
\end{itemize}
Tasks~1--3 are coarse-grained (level-1), while Task~4 is fine-grained (level-2).

\paragraph{Datasets and evaluation.}
\Cref{tab:level_1_data_stat,tab:level_1_programming,tab:level_2_data_stat} present the complete dataset composition for Tasks~1--4.

For Task~1, we calibrate using domain-representative datasets: MMLU European History for general text (165 samples), GSM8K for mathematics (2,000 samples), and Magicoder for code (2,000 samples). Classification employs both in-distribution and out-of-distribution test sets to assess generalization: MMLU US History for general text, GSM8K (in-distribution) and MATH500 (out-of-distribution) for mathematics, and Magicoder (in-distribution) and HumanEval (out-of-distribution) for code.

For Tasks~2--4, we maintain consistent calibration sizes where feasible: 2,000 samples per programming language (Task~2), 512 samples per natural language (Task~3), and 800 samples per mathematical subfield (Task~4). Classification sets contain up to 5,000 samples per class for Task~2, and up to 3,000 samples per class for Tasks~3 and~4, subject to dataset availability. Further details on each dataset are provided in~\cref{appendix:dataset_details}.

\subsection{Datasets}
\label{appendix:dataset_details}

We employ seven benchmark datasets spanning general text, mathematics, and code domains.

\begin{itemize}[leftmargin=1.6em]
    \item \textbf{General Text.} For Task~1, we use MMLU~\citep{hendrycks2021ethics,hendryckstest2021}, a comprehensive benchmark of multiple-choice questions across 57 subjects. Calibration data is drawn from the High School European History subset, and evaluation is performed on the High School US History subset, using question-and-choice pairs as input.

    \item \textbf{Natural Language.} For Task~3, we use the Aya dataset~\citep{singh2024aya}, which contains human-annotated prompts across 65 languages. We select the five linguistically diverse languages specified in~\cref{tab:level_1_programming}, splitting each language subset into calibration and test sets, and use the input field for classification.

    \item \textbf{Mathematics.} We employ three benchmarks. GSM8K~\citep{cobbe2021gsm8k} provides grade-school word problems requiring multi-step reasoning; we sample 2,000 instances for calibration and use the full test set for Task~1 evaluation. MATH500~\citep{lightman2023lets} serves as an out-of-distribution test set for Task~1, containing 500 problems from the MATH benchmark. Competition Math~\citep{hendrycksmath2021} provides competition problems spanning seven mathematical subfields and is used for Task~4; we use the problem field as model input.

    \item \textbf{Code.} Magicoder~\citep{wei2023magicoder} is our primary code resource, containing solutions in multiple programming languages. We use the solution field as model input for both Task~1 (general code classification) and Task~2 (programming language identification), focusing on the nine languages listed in~\cref{tab:level_1_programming}. HumanEval~\citep{chen2021evaluating} provides 164 function-level programming problems; we use the prompt field (containing function signatures and docstrings) as model input, serving as an out-of-distribution test set for Task~1.
\end{itemize}

All datasets are publicly available through HuggingFace Datasets~\citep{lhoest-etal-2021-datasets}. Sampling strategies and train-test splits follow the specifications in \cref{tab:level_1_data_stat,tab:level_1_programming,tab:level_2_data_stat}, with test samples capped at the minimum of the specified count and available data.

\subsection{Selected LLMs} \label{appendix:benchmark_LLM}
We evaluate our approach on 7 pretrained large language models spanning 1B to 32B parameters, encompassing both base and instruction-tuned variants. This selection provides comprehensive coverage across model scales and training stages. We consider the following two LLM families:
\begin{itemize}[leftmargin=1.6em]
    \item \textbf{Qwen family}~\citep{yang2025qwen3}: We evaluate four models from the Qwen3 series. The instruction-tuned variants include Qwen3-1.7B (28 layers), Qwen3-4B (36 layers), Qwen3-8B (36 layers), and Qwen3-32B (64 layers), each post-trained with supervised fine-tuning and reinforcement learning from human feedback (RLHF). Additionally, we include Qwen3-1.7B-Base to assess performance on pretrained models without alignment. For all Qwen3 evaluations, we switch on non-thinking mode to ensure consistent comparison across models.
    \item \textbf{Llama family}~\citep{dubey2024llama}: We evaluate Llama-3.2-1B (16 layers) and its instruction-tuned counterpart Llama-3.2-1B-Instruct. The instruction-tuned variant underwent supervised fine-tuning and RLHF to better align with human preferences.
\end{itemize}

This benchmark model selection enables systematic evaluation across three critical dimensions: model scale (from $1$B to $32$B parameters), training paradigm (pretrained-only versus post-trained), and architectural diversity (Qwen and Llama families). The substantial range in model sizes—spanning over an order of magnitude in parameters—allows us to rigorously test whether our method can effectively operate across vastly different computational scales and model capacities. The comparison between base and aligned models reveals how post-training procedures affect our method's performance, demonstrating whether it remains equally effective for both pretrained and instruction-tuned models. The prompt used for direct LLM inference under level-1 setting is in~\cref{tab:llm_call_prompt_template}.

\begin{table}[t]
\centering
\caption{\small Prompt template for direct LLM inference under level-1 intent classification setting. The template includes placeholders \texttt{\{examples\_block\}} for optional few-shot examples and \texttt{\{user\_text\}} for the input to be classified. The few-shot examples are taken from the calibration data.}
\label{tab:llm_call_prompt_template}
\begin{tabular}{p{1.0\textwidth}}
\toprule
\textbf{LLM Prompt Template for Intent Classification} \\
\midrule
You are a classifier.\\
Your task is to look at the user's input text and decide which of these three categories it belongs to:\\
\\
1. General text – natural language content like sentences, questions, explanations, stories, or instructions that are not primarily mathematics or code.\\
2. Math – text that is primarily mathematical expressions, equations, formulas, or word problems where the main focus is on mathematics.\\
3. Code – text that is primarily programming code or pseudocode (any programming language, including configuration snippets or shell commands).\\
\\
Output rules:\\
\\
- If the input is general text, output: A\\
- If the input is math, output: B\\
- If the input is code, output: C\\
\\
Important Notes:\\
\\
- Output only a single letter: A, B, or C.\\
- Do not output anything else (no explanation, no punctuation, no spaces).\\
\\
Below are some classification examples for this task: \\
\textbf{\{examples\_block\}}\\
\\
Now classify the following input accordingly and output just one letter.\\
\\
Input:\\
\textbf{\{user\_text\}}\\
Output: \\
\bottomrule
\end{tabular}
\end{table}

\subsection{More Implementation Details}



To stabilize training, we normalize input features and weight matrices through the following process:

\begin{enumerate}
\item Input Normalization: The input tensor $z$ is normalized to unit norm for stability:
\begin{equation*}
z_{\text{normalized}} = \frac{z}{\|z\|_2 + \epsilon}.
\end{equation*}

\item Weight Normalization: The weight matrix $W$ is normalized using its Frobenius norm:
\begin{equation*}
W_{\text{normalized}} = \frac{W}{\sqrt{\text{mean}(W^2)} + \epsilon}.
\end{equation*}

\item Activation Computation: The normalized input is multiplied by the normalized weight matrix:
\begin{equation*}
Wz = z_{\text{normalized}} \cdot W_{\text{normalized}}^T.
\end{equation*}
\end{enumerate}
The statistics for \normstat{} and \vecstat{} are computed from $Wz$. Note that this is likely suboptimal; in a production-scale implementation we should read $Wz$ directly from the module’s output rather than recomputing it.

\subsection{Pseudo-Algorithm}

Our task inference approach follows the following steps:

\begin{algorithm}
\caption{Intent Classification with \normstat~or \vecstat}
\begin{algorithmic}[1]
\REQUIRE Input prompt $x$, baseline scores $\{S_1, S_2, \ldots, S_m\}$ for tasks $C_1, C_2, \ldots, C_m$, distance function ``dist" (\normstat: KL; \vecstat: KL or cosine similarity)
\ENSURE Predicted task $T_{\text{pred}}$ and confidence scores
\STATE Compute statistic scores $S_x$ from input $x$
\STATE \textbf{for} each task $C_i$ \textbf{do}
\STATE \quad Compute distance $d_i = \text{dist}(S_x, S_i)$
\STATE \textbf{end for}
\STATE $C_{\text{pred}} = \arg\min_{C_i} d_i$
\STATE Compute probabilities via softmax: $p_i = \frac{\exp(-(d_i - \bar{d})/\tau)}{\sum_j \exp(-(d_j-\bar{d})/\tau)}$ (used for uncertainty quantification, $\bar{d}$ is the average, $\tau$ is the temperature)
\RETURN $C_{\text{pred}}$, $\{p_1, p_2, \ldots, p_m\}$
\end{algorithmic}
\end{algorithm}

\subsection{Hardware and software environment}
We conducted experiments on two computational platforms based on model scales. For models up to 4B parameters, we utilized an NVIDIA L40S GPU with 48GB of memory. For larger models (Qwen3-8B and Qwen3-32B), experiments were performed on an NVIDIA Grace Hopper GH200 superchip, featuring a Grace ARM 72-core CPU with 120GB RAM and a NVIDIA H100 GPU with 96GB of memory. All experiments are implemented using Python 3.12.0 and PyTorch 2.7.0 with CUDA 12.6.

\section{Additional Experimental Results}\label{appendix:exp}

\subsection{Additional Results for Tasks~1--3 (Level-1)}\label{appendix:l1-exp}

\Cref{tab:level-1-result} reports per-dataset accuracy for Task~1 (general intent classification) across all seven LLMs.
\Cref{tab:level-2-plang-results} reports per-language accuracy for Task~2 (programming language identification).
\Cref{tab:level-nat-lang-results} reports per-language accuracy for Task~3 (natural language identification).

For Task~2, \vecstat{} maintains near-perfect accuracy while \normstat{} shows substantial degradation, aligning with our theoretical prediction that directional information becomes critical for within-domain discrimination. This performance gap is consistent across all nine programming languages. In contrast, for Task~3, all methods perform strongly, with \normstat{} remaining competitive, reflecting stronger radial separation among natural language classes.

\begin{table}[h]
\centering
\small
\caption{\small Task~1 classification results for all seven LLMs (level-1 granularity). Values represent per-dataset accuracy for coarse-grained domain classification (general text, mathematics, and code).}
\label{tab:level-1-result}
\resizebox{\linewidth}{!}{%
\begin{tabular}{llccccc}
\toprule
Model & Method & gsm8k & humaneval & magicoder & math500 & mmlu\_history \\
\midrule
\multirow{7}{*}{Qwen3-1.7B} & \embedavg & \ms{99.97}{0.04} & \ms{99.80}{0.35} & \ms{99.99}{0.02} & \ms{78.33}{1.67} & \ms{100.00}{0.00} \\
 & \embedlast & \ms{100.00}{0.00} & \ms{100.00}{0.00} & \ms{99.97}{0.02} & \ms{99.00}{0.00} & \ms{100.00}{0.00} \\

& \linearavg &\ms{100.00}{0.00} & \ms{99.80}{0.35} & \ms{100.00}{0.00} & \ms{69.20}{9.18} & \ms{100.00}{0.00} \\

\cdashline{2-7}
  
 & \normstatKL & \ms{97.35}{0.00} & \ms{70.12}{0.61} & \ms{99.27}{0.11} & \ms{76.60}{0.20} & \ms{100.00}{0.00} \\
& \vecstatKL & \ms{100.00}{0.00} & \ms{99.39}{0.00} & \ms{99.97}{0.03} & \ms{88.33}{0.30} & \ms{100.00}{0.00} \\
& \vecstatCos & \ms{100.00}{0.00} & \ms{98.78}{0.00} & \ms{99.97}{0.03} & \ms{92.26}{0.11} & \ms{100.00}{0.00} \\

\cdashline{2-7}

 & \llmcall{0} & \ms{51.18}{0.30} & \ms{99.39}{0.61} & \ms{97.38}{0.11} & \ms{80.87}{0.61} & \ms{82.19}{1.02} \\


 
 

\midrule

\multirow{7}{*}{Qwen3-1.7B-Base} & \embedavg & \ms{78.09}{31.54} & \ms{99.80}{0.35} & \ms{100.00}{0.00} & \ms{40.60}{19.91} & \ms{100.00}{0.00} \\

 & \embedlast & \ms{100.00}{0.00} & \ms{84.55}{26.76} & \ms{99.87}{0.18} & \ms{99.00}{0.00} & \ms{100.00}{0.00} \\

 & \linearavg & \ms{79.76}{0.53} & \ms{99.59}{0.70} & \ms{100.00}{0.00} & \ms{30.73}{4.28} & \ms{100.00}{0.00} \\

\cdashline{2-7}

 
 & \normstatKL & \ms{79.71}{0.29} & \ms{82.93}{0.00} & \ms{99.71}{0.03} & \ms{88.47}{0.12} & \ms{100.00}{0.00} \\
 
 & \vecstatKL & \ms{40.46}{0.10} & \ms{100.00}{0.00} & \ms{100.00}{0.0} & \ms{49.06}{0.95} & \ms{100.00}{0.00} \\
 
 & \vecstatCos & \ms{99.84}{0.00} & \ms{99.39}{0.00} & \ms{99.98}{0.02} & \ms{92.30}{0.39} & \ms{100.00}{0.00} \\

 \cdashline{2-7}

& \llmcall{0} & \ms{90.60}{0.00} & \ms{0.00}{0.00} & \ms{81.03}{0.39} & \ms{59.40}{0.00} & \ms{4.41}{0.00} \\


 
 
 
\midrule

\multirow{6}{*}{Llama-3.2-1B} & \embedavg & \ms{100.00}{0.00} & \ms{100.00}{0.00} & \ms{99.99}{0.01} & \ms{64.33}{8.33} & \ms{99.84}{0.28} \\

 & \embedlast & \ms{99.95}{0.04} & \ms{100.00}{0.00} & \ms{99.97}{0.03} & \ms{98.93}{0.42} & \ms{99.35}{1.13} \\

 & \linearavg & \ms{100.00}{0.00} & \ms{99.80}{0.35} & \ms{99.99}{0.01} & \ms{71.47}{7.78} & \ms{98.69}{1.86} \\

 \cdashline{2-7}

 
 & \normstatKL & \ms{99.49}{0.09} & \ms{90.85}{0.00} & \ms{96.39}{0.20} & \ms{83.40}{0.00} & \ms{92.48}{0.57} \\
 
 & \vecstatKL & \ms{100.00}{0.00} & \ms{99.39}{0.00} & \ms{99.98}{0.02} & \ms{78.80}{0.12} & \ms{100.00}{0.00} \\
 
 & \vecstatCos & \ms{100.00}{0.00} & \ms{99.39}{0.00} & \ms{99.97}{0.02} & \ms{77.60}{0.20} & \ms{100.00}{0.00} \\



\midrule

\multirow{6}{*}{Llama-3.2-1B-Instruct} & \embedavg & \ms{100.00}{0.00} & \ms{99.59}{0.35} & \ms{99.99}{0.02} & \ms{72.40}{2.91} & \ms{100.00}{0.00} \\

 & \embedlast & \ms{100.00}{0.00} & \ms{100.00}{0.00} & \ms{99.89}{0.05} & \ms{94.53}{0.42} & \ms{100.00}{0.00} \\

& \linearavg & \ms{100.00}{0.00} & \ms{99.80}{0.35} & \ms{100.00}{0.00} & \ms{66.80}{7.75} & \ms{100.00}{0.00} \\


 \cdashline{2-7}
 
 & \normstatKL & \ms{100.00}{0.00} & \ms{65.24}{0.00} & \ms{97.21}{0.21} & \ms{96.20}{0.00} & \ms{98.53}{0.00} \\
 
 & \vecstatKL & \ms{100.00}{0.00} & \ms{99.39}{0.00} & \ms{99.96}{0.01} & \ms{87.87}{0.11} & \ms{100.00}{0.00}\\
 
  & \vecstatCos & \ms{100.00}{0.00} & \ms{99.39}{0.00} & \ms{99.97}{0.03} & \ms{85.60}{0.00} & \ms{100.00}{0.00}\\


\midrule

\multirow{7}{*}{Qwen3-4B} & \embedavg & \ms{99.95}{0.09} & \ms{100.00}{0.00} & \ms{99.99}{0.01} & \ms{79.27}{4.39} & \ms{100.00}{0.00} \\

 & \embedlast & \ms{100.00}{0.00} & \ms{100.00}{0.00} & \ms{99.96}{0.02} & \ms{92.27}{1.17} & \ms{100.00}{0.00} \\

 & \linearavg & \ms{99.90}{0.04} & \ms{100.00}{0.00} & \ms{99.99}{0.01} & \ms{68.87}{18.85} & \ms{100.00}{0.00} \\


 \cdashline{2-7}
 
 & \normstatKL & \ms{96.36}{0.00} & \ms{35.57}{0.35} & \ms{99.84}{0.03} & \ms{89.40}{0.00} & \ms{99.84}{0.28} \\
 
 & \vecstatKL & \ms{99.97}{0.04} & \ms{100.00}{0.00} & \ms{99.99}{0.02} & \ms{91.73}{0.23} & \ms{100.00}{0.00} \\
 
 & \vecstatCos & \ms{100.00}{0.00} & \ms{96.34}{0.00} & \ms{99.97}{0.03} & \ms{93.40}{0.20} & \ms{100.00}{0.00} \\

 \cdashline{2-7}

 & \llmcall{0} & \ms{91.84}{0.16} & \ms{100.00}{0.00} & \ms{99.47}{0.06} & \ms{98.33}{0.12} & \ms{95.42}{0.28} \\
 

 \midrule

\multirow{7}{*}{Qwen3-8B} & \embedavg & \ms{99.42}{0.74} & \ms{99.59}{0.35} & \ms{99.99}{0.02} & \ms{77.27}{11.02} & \ms{100.00}{0.00} \\

 & \embedlast & \ms{99.82}{0.12} & \ms{98.58}{2.46} & \ms{99.98}{0.02} & \ms{81.20}{9.72} & \ms{100.00}{0.00} \\

  & \linearavg & \ms{99.57}{0.24} & \ms{100.00}{0.00} & \ms{100.00}{0.00} & \ms{83.00}{2.03} & \ms{100.00}{0.00} \\


 \cdashline{2-7}
 
 & \normstatKL & \ms{85.14}{0.35} & \ms{10.37}{1.06} & \ms{99.85}{0.06} & \ms{92.93}{0.12} & \ms{99.51}{0.49} \\
 
 & \vecstatKL & \ms{99.95}{0.04} & \ms{99.59}{0.35} & \ms{99.99}{0.02} & \ms{92.20}{0.00} & \ms{100.00}{0.00} \\
 
 & \vecstatCos & \ms{100.00}{0.00} & \ms{95.73}{0.61} & \ms{99.98}{0.03} & \ms{94.20}{0.00} & \ms{100.00}{0.00} \\

\cdashline{2-7}

& \llmcall{0} & \ms{99.67}{0.04} & \ms{100.00}{0.00} & \ms{99.42}{0.07} & \ms{99.60}{0.00} & \ms{99.67}{0.28} \\


 \midrule
\multirow{7}{*}{Qwen3-32B} & \embedavg & \ms{95.88}{3.31} & \ms{100.00}{0.00} & \ms{99.99}{0.01} & \ms{86.87}{5.22} & \ms{100.00}{0.00} \\

 & \embedlast & \ms{99.39}{0.00} & \ms{100.00}{0.00} & \ms{99.98}{0.00} & \ms{96.80}{0.40} & \ms{99.02}{0.49} \\

 & \linearavg & \ms{98.61}{1.59} & \ms{100.00}{0.00} & \ms{99.99}{0.02} & \ms{87.53}{6.94} & \ms{100.00}{0.00} \\


\cdashline{2-7}
 
 & \normstatKL & \ms{97.93}{0.04} & \ms{24.59}{0.35} & \ms{99.78}{0.06} & \ms{97.93}{0.12} & \ms{100.00}{0.00} \\
 
 & \vecstatKL & \ms{100.00}{0.00} & \ms{99.39}{0.00} & \ms{99.98}{0.03} & \ms{96.60}{0.00} & \ms{100.00}{0.00} \\
 
 & \vecstatCos & \ms{100.00}{0.00} & \ms{98.17}{0.00} & \ms{99.98}{0.03} & \ms{96.80}{0.35} & \ms{100.00}{0.00} \\

 \cdashline{2-7}

& \llmcall{0} & \ms{86.91}{0.29} & \ms{100.00}{0.00} & \ms{98.69}{0.06} & \ms{96.27}{0.42} & \ms{100.00}{0.00} \\
  
  
\bottomrule
\end{tabular}%
}
\end{table}

\begin{table}[t]
\centering
\small
\caption{\small Task~2 classification results for all seven LLMs (level-1 granularity). Values represent per-language accuracy across nine programming languages from the Magicoder dataset.}
\label{tab:level-2-plang-results}
\resizebox{\linewidth}{!}{%
\begin{tabular}{llccccccccc}
\toprule
Model & Method & cpp & csharp & java & php & python & rust & shell & swift & typescript \\
\midrule
\multirow{7}{*}{Qwen3-1.7B} & \embedavg & \ms{99.97}{0.03} & \ms{100.00}{0.00} & \ms{100.00}{0.00} & \ms{99.96}{0.06} & \ms{99.89}{0.01} & \ms{100.00}{0.00} & \ms{99.92}{0.14} & \ms{100.00}{0.00} & \ms{99.94}{0.07} \\

 & \embedlast & \ms{99.54}{0.09} & \ms{99.57}{0.06} & \ms{99.41}{0.26} & \ms{99.89}{0.11} & \ms{99.35}{0.27} & \ms{99.32}{0.17} & \ms{99.92}{0.14} & \ms{99.81}{0.07} & \ms{99.52}{0.03} \\

& \linearavg & \ms{99.97}{0.03} & \ms{99.98}{0.04} & \ms{100.00}{0.00} & \ms{99.96}{0.06} & \ms{99.87}{0.01}& \ms{100.00}{0.00} & \ms{99.92}{0.14} & \ms{100.00}{0.00} & \ms{99.93}{0.05} \\

\cdashline{2-11}
 
 & \normstatKL & \ms{53.18}{0.60} & \ms{53.17}{1.96} & \ms{40.39}{0.63} & \ms{57.35}{1.97} & \ms{60.17}{0.89} & \ms{58.07}{0.55} & \ms{88.73}{1.67} & \ms{59.47}{0.88} & \ms{49.59}{0.34} \\
 
  & \vecstatKL & \ms{98.83}{0.18} & \ms{98.98}{0.39} & \ms{98.21}{0.15} & \ms{99.81}{0.23} & \ms{99.06}{0.09} & \ms{99.68}{0.14} & \ms{99.52}{0.00} & \ms{99.81}{0.12} & \ms{99.26}{0.20} \\
 
 & \vecstatCos & \ms{98.51}{0.19} & \ms{98.78}{0.37} & \ms{97.79}{0.19} & \ms{99.85}{0.17} & \ms{99.06}{0.14} & \ms{99.35}{0.15} & \ms{99.84}{0.14} & \ms{99.55}{0.11} & \ms{99.02}{0.29} \\

 \cdashline{2-11}

& \llmcall{0} & \ms{94.64}{0.57} & \ms{99.35}{0.09} & \ms{88.23}{0.70} & \ms{99.52}{0.17} & \ms{99.88}{0.02} & \ms{98.34}{0.07} & \ms{99.05}{0.41} & \ms{99.89}{0.00} & \ms{89.18}{0.44} \\
 
\midrule

\multirow{7}{*}{Qwen3-1.7B-Base} & \embedavg & \ms{99.97}{0.03} & \ms{99.96}{0.04} & \ms{99.97}{0.03} & \ms{100.00}{0.00} & \ms{99.89}{0.05} & \ms{100.00}{0.00} & \ms{99.92}{0.14} & \ms{100.00}{0.00} & \ms{99.96}{0.04} \\

 & \embedlast & \ms{99.28}{0.19} & \ms{99.37}{0.19} & \ms{99.23}{0.12} & \ms{99.96}{0.06} & \ms{99.25}{0.14} & \ms{99.27}{0.19} & \ms{99.84}{0.27} & \ms{99.75}{0.14} & \ms{99.50}{0.03} \\

 & \linearavg & \ms{99.97}{0.03} & \ms{99.96}{0.04} & \ms{99.98}{0.03} & \ms{100.00}{0.00} & \ms{99.89}{0.04} & \ms{100.00}{0.00} & \ms{99.92}{0.14} & \ms{100.00}{0.00} & \ms{99.96}{0.04} \\

 \cdashline{2-11}

 
 & \normstatKL & \ms{47.76}{0.18} & \ms{43.75}{1.18} & \ms{36.92}{0.96} & \ms{60.44}{1.13} & \ms{56.03}{0.82} & \ms{54.88}{0.78} & \ms{88.33}{2.42} & \ms{57.85}{1.25} & \ms{51.20}{0.25} \\
 
 & \vecstatKL & \ms{98.29}{0.13} & \ms{98.37}{0.53} & \ms{97.00}{0.21} & \ms{99.66}{0.11} & \ms{98.63}{0.09} & \ms{99.45}{0.12} & \ms{99.68}{0.14} & \ms{99.64}{0.12} & \ms{99.17}{0.19} \\
 
 & \vecstatCos & \ms{98.27}{0.17} & \ms{98.23}{0.50} & \ms{96.91}{0.05} & \ms{99.63}{0.17} & \ms{98.84}{0.09} & \ms{99.47}{0.13} & \ms{99.68}{0.14} & \ms{99.58}{0.07} & \ms{98.54}{0.30} \\

\midrule

\multirow{6}{*}{Llama-3.2-1B} & \embedavg & \ms{99.97}{0.03} & \ms{99.98}{0.04} & \ms{99.97}{0.05} & \ms{99.96}{0.06} & \ms{99.91}{0.02} & \ms{100.00}{0.00} & \ms{100.00}{0.00} & \ms{100.00}{0.00} & \ms{99.94}{0.03} \\

 & \embedlast & \ms{99.21}{0.24} & \ms{99.08}{0.16} & \ms{99.41}{0.10} & \ms{99.85}{0.17} & \ms{99.25}{0.26} & \ms{99.21}{0.32} & \ms{99.84}{0.27} & \ms{99.75}{0.09} & \ms{99.58}{0.11} \\

 & \linearavg & \ms{99.97}{0.03} & \ms{99.96}{0.04} & \ms{99.94}{0.05} & \ms{99.96}{0.06} & \ms{99.89}{0.01} & \ms{99.98}{0.03} & \ms{100.00}{0.00} & \ms{100.00}{0.00} & \ms{99.93}{0.03} \\

\cdashline{2-11}
 
 & \normstatKL & \ms{44.23}{1.37} & \ms{39.54}{2.26} & \ms{35.09}{1.07} & \ms{49.25}{1.60} & \ms{38.65}{0.71} & \ms{54.91}{0.78} & \ms{90.95}{1.56} & \ms{55.63}{1.24} & \ms{34.67}{1.81} \\
  
  & \vecstatKL & \ms{98.77}{0.19} & \ms{98.37}{0.43} & \ms{97.14}{0.07} & \ms{99.63}{0.17} & \ms{98.66}{0.21} & \ms{99.79}{0.03} & \ms{99.84}{0.27} & \ms{99.73}{0.18} & \ms{99.02}{0.25} \\
 
 & \vecstatCos & \ms{98.53}{0.12} & \ms{97.78}{0.25} & \ms{96.34}{0.09} & \ms{99.52}{0.28} & \ms{98.86}{0.18} & \ms{99.48}{0.15} & \ms{99.60}{0.14} & \ms{99.49}{0.25} & \ms{98.81}{0.36} \\

\midrule

\multirow{6}{*}{Llama-3.2-1B-Instruct} & \embedavg & \ms{99.95}{0.05} & \ms{99.98}{0.04} & \ms{99.97}{0.03} & \ms{99.96}{0.06} & \ms{99.91}{0.01} & \ms{100.00}{0.00} & \ms{99.92}{0.14} & \ms{100.00}{0.00} & \ms{99.99}{0.03} \\

 & \embedlast & \ms{99.07}{0.39} & \ms{99.23}{0.15} & \ms{99.12}{0.40} & \ms{99.55}{0.22} & \ms{99.27}{0.11} & \ms{99.00}{0.49} & \ms{99.84}{0.27} & \ms{99.70}{0.16} & \ms{99.64}{0.04} \\

 & \linearavg & \ms{99.97}{0.03} & \ms{99.96}{0.04} & \ms{99.95}{0.05} & \ms{99.96}{0.06} & \ms{99.88}{0.03} & \ms{100.00}{0.00} & \ms{99.92}{0.14} & \ms{100.00}{0.00} & \ms{99.99}{0.03} \\

 \cdashline{2-11}

 
 & \normstatKL & \ms{41.34}{1.45} & \ms{45.89}{0.57} & \ms{33.32}{1.51} & \ms{45.23}{2.25} & \ms{46.42}{1.32} & \ms{55.80}{0.67} & \ms{89.76}{2.52} & \ms{57.64}{0.72} & \ms{34.05}{2.20} \\
  
  & \vecstatKL & \ms{98.56}{0.19} & \ms{97.58}{0.19} & \ms{97.25}{0.13} & \ms{99.70}{0.17} & \ms{98.87}{0.16} & \ms{99.66}{0.08} & \ms{99.84}{0.27} & \ms{99.62}{0.14} & \ms{99.02}{0.27} \\
 
 & \vecstatCos & \ms{98.56}{0.19} & \ms{97.46}{0.34} & \ms{96.85}{0.23} & \ms{99.78}{0.19} & \ms{99.00}{0.13} & \ms{99.37}{0.10} & \ms{99.84}{0.27} & \ms{99.64}{0.14} & \ms{98.87}{0.32} \\

 \midrule

\multirow{7}{*}{Qwen3-4B} & \embedavg & \ms{99.97}{0.03} & \ms{99.94}{0.00} & \ms{100.00}{0.00} & \ms{99.96}{0.06} & \ms{99.89}{0.01} & \ms{100.00}{0.00} & \ms{100.00}{0.00} & \ms{100.00}{0.00} & \ms{99.94}{0.03} \\

 & \embedlast & \ms{99.37}{0.11} & \ms{99.53}{0.13} & \ms{99.46}{0.03} & \ms{99.81}{0.13} & \ms{99.48}{0.02} & \ms{99.35}{0.27} & \ms{99.84}{0.27} & \ms{99.73}{0.14} & \ms{99.66}{0.07} \\

 & \linearavg & \ms{99.97}{0.03} & \ms{99.96}{0.07} & \ms{99.98}{0.03} & \ms{99.96}{0.06} & \ms{99.89}{0.02} & \ms{100.00}{0.00} & \ms{100.00}{0.00} & \ms{100.00}{0.00} & \ms{99.93}{0.05} \\

 \cdashline{2-11}
 
 
 & \normstatKL & \ms{54.55}{1.04} & \ms{40.86}{1.39} & \ms{35.45}{0.28} & \ms{56.26}{0.68} & \ms{69.71}{1.04} & \ms{68.76}{0.65} & \ms{93.10}{1.33} & \ms{72.17}{1.20} & \ms{48.70}{0.32} \\
 
 & \vecstatKL & \ms{99.26}{0.08} & \ms{98.60}{0.38} & \ms{97.95}{0.12} & \ms{99.74}{0.26} & \ms{99.07}{0.09} & \ms{99.82}{0.03} & \ms{99.76}{0.00} & \ms{99.87}{0.12} & \ms{99.30}{0.19} \\
 
 & \vecstatCos & \ms{98.92}{0.09} & \ms{98.86}{0.44} & \ms{98.18}{0.05} & \ms{99.89}{0.19} & \ms{99.13}{0.08} & \ms{99.53}{0.10} & \ms{99.84}{0.14} & \ms{99.79}{0.13} & \ms{99.26}{0.23} \\

 \cdashline{2-11}

 & \llmcall{0} & \ms{99.11}{0.17} & \ms{99.78}{0.07} & \ms{99.49}{0.12} & \ms{99.89}{0.00} & \ms{100.00}{0.00} & \ms{99.97}{0.03} & \ms{99.92}{0.14} & \ms{94.11}{0.62} & \ms{52.13}{0.36} \\

\midrule

 \multirow{7}{*}{Qwen3-8B} & \embedavg & \ms{99.95}{0.05} & \ms{100.00}{0.00} & \ms{100.00}{0.00} & \ms{99.96}{0.06} & \ms{99.91}{0.01} & \ms{100.00}{0.00} & \ms{100.00}{0.00} & \ms{100.00}{0.00} & \ms{99.94}{0.03} \\
 
 & \embedlast & \ms{99.61}{0.16} & \ms{99.69}{0.18} & \ms{99.47}{0.12} & \ms{100.00}{0.00} & \ms{99.39}{0.22} & \ms{99.53}{0.10} & \ms{99.92}{0.14} & \ms{99.75}{0.07} & \ms{99.61}{0.10} \\

 & \linearavg & \ms{99.97}{0.03} & \ms{100.00}{0.00} & \ms{100.00}{0.00} & \ms{99.96}{0.06} & \ms{99.91}{0.03} & \ms{100.00}{0.00} & \ms{100.00}{0.00} & \ms{100.00}{0.00} & \ms{99.93}{0.03} \\

 \cdashline{2-11}
 
 
 & \normstatKL & \ms{47.37}{0.94} & \ms{38.26}{1.28} & \ms{27.59}{0.65} & \ms{60.14}{1.27} & \ms{67.73}{0.54} & \ms{63.05}{0.43} & \ms{91.51}{0.60} & \ms{62.54}{1.25} & \ms{49.28}{0.84} \\
 
 & \vecstatKL & \ms{99.09}{0.11} & \ms{98.58}{0.42} & \ms{97.78}{0.09} & \ms{99.78}{0.22} & \ms{99.01}{0.04} & \ms{99.81}{0.00} & \ms{99.76}{0.00} & \ms{99.91}{0.09} & \ms{99.38}{0.20} \\
 
 & \vecstatCos & \ms{98.90}{0.24} & \ms{98.96}{0.38} & \ms{98.49}{0.03} & \ms{99.85}{0.13} & \ms{99.13}{0.12} & \ms{99.55}{0.07} & \ms{100.00}{0.00} & \ms{99.81}{0.12} & \ms{99.35}{0.19} \\

\cdashline{2-11}

& LM Call & \ms{99.81}{0.15} & \ms{100.00}{0.00} & \ms{99.98}{0.03} & \ms{100.00}{0.00} & \ms{100.00}{0.00} & \ms{100.00}{0.00} & \ms{99.92}{0.14} & \ms{100.00}{0.00} & \ms{90.33}{0.35} \\

\midrule
\multirow{7}{*}{Qwen3-32B} & \embedavg & \ms{99.91}{0.03} & \ms{99.98}{0.04} & \ms{100.00}{0.00} & \ms{99.96}{0.06} & \ms{99.91}{0.03} & \ms{99.97}{0.06} & \ms{100.00}{0.00} & \ms{100.00}{0.00} & \ms{99.96}{0.04} \\

 & \embedlast & \ms{98.01}{0.41} & \ms{97.27}{0.95} & \ms{97.99}{0.53} & \ms{98.77}{0.22} & \ms{98.26}{0.21} & \ms{99.21}{0.06} & \ms{98.97}{0.60} & \ms{98.71}{0.48} & \ms{98.50}{0.07} \\

 & \linearavg & \ms{99.91}{0.03} & \ms{99.98}{0.04} & \ms{100.00}{0.00} & \ms{99.96}{0.06} & \ms{99.92}{0.00} & \ms{99.97}{0.06} & \ms{100.00}{0.00} & \ms{100.00}{0.00} & \ms{99.96}{0.04} \\

\cdashline{2-11}

 
 & \normstatKL & \ms{53.57}{0.94} & \ms{36.94}{1.38} & \ms{25.47}{1.43} & \ms{59.47}{0.57} & \ms{61.66}{1.40} & \ms{67.63}{0.54} & \ms{89.29}{0.71} & \ms{69.35}{1.19} & \ms{49.84}{1.08} \\
 
 & \vecstatKL & \ms{99.01}{0.11} & \ms{99.35}{0.34} & \ms{98.72}{0.03} & \ms{99.89}{0.11} & \ms{99.29}{0.06} & \ms{99.90}{0.05} & \ms{100.00}{0.00} & \ms{99.85}{0.09} & \ms{99.75}{0.16} \\
 
 & \vecstatCos & \ms{99.18}{0.14} & \ms{99.53}{0.19} & \ms{99.01}{0.10} & \ms{99.89}{0.11} & \ms{99.50}{0.04} & \ms{99.81}{0.13} & \ms{100.00}{0.00} & \ms{99.87}{0.12} & \ms{99.72}{0.09} \\

\cdashline{2-11}

& \llmcall{0} & \ms{99.90}{0.05} & \ms{100.00}{0.00} & \ms{99.98}{0.03} & \ms{99.96}{0.06} & \ms{99.98}{0.00} & \ms{100.00}{0.00} & \ms{99.84}{0.14} & \ms{100.00}{0.00} & \ms{98.68}{0.23} \\

\bottomrule
\end{tabular}%
}
\end{table}

\begin{table}[h]
        \centering
        \small
        \caption{\small Task~3 classification results for all seven LLMs (level-1 granularity). Values represent per-language accuracy across five natural languages from the Aya dataset.}
        \label{tab:level-nat-lang-results}
        \resizebox{\linewidth}{!}{%
        \begin{tabular}{llccccc}
        \toprule
        Model & Method & English & Japanese & Moroccan Arabic & Sinhala & Tamil \\
        \midrule
        \multirow{7}{*}{Qwen3-1.7B} & \embedavg & \ms{99.74}{0.12} & \ms{99.99}{0.02} & \ms{99.97}{0.03} & \ms{99.99}{0.02} & \ms{99.94}{0.04} \\
 
 & \embedlast & \ms{99.49}{0.15} & \ms{100.00}{0.00} & \ms{99.97}{0.03} & \ms{99.93}{0.06} & \ms{99.94}{0.10} \\

 & \linearavg & \ms{99.79}{0.11} & \ms{100.00}{0.00} & \ms{99.97}{0.03} & \ms{99.99}{0.02} & \ms{99.94}{0.04} \\

 \cdashline{2-7}
 
 
 & \normstatKL & \ms{82.26}{1.39} & \ms{81.42}{0.81} & \ms{85.78}{0.86} & \ms{99.87}{0.12} & \ms{99.52}{0.12} \\
 
 & \vecstatKL & \ms{96.27}{0.34} & \ms{99.83}{0.03} & \ms{99.97}{0.00} & \ms{99.99}{0.02} & \ms{99.93}{0.06} \\
 
 & \vecstatCos & \ms{98.09}{0.10} & \ms{99.82}{0.05} & \ms{99.99}{0.02} & \ms{99.99}{0.02} & \ms{99.94}{0.04} \\

\cdashline{2-7}

& \llmcall{0} & \ms{97.27}{0.12} & \ms{72.24}{0.48} & \ms{94.57}{0.38} & \ms{9.78}{0.38} & \ms{99.92}{0.02} \\
 
\midrule

\multirow{7}{*}{Qwen3-1.7B-Base} & \embedavg & \ms{99.64}{0.15} & \ms{100.00}{0.00} & \ms{99.94}{0.04} & \ms{99.99}{0.02} & \ms{99.96}{0.02} \\
 
 & \embedlast & \ms{99.67}{0.12} & \ms{100.00}{0.00} & \ms{99.97}{0.03} & \ms{99.98}{0.02} & \ms{99.87}{0.03} \\

 & \linearavg & \ms{99.66}{0.16} & \ms{99.98}{0.04} & \ms{99.94}{0.04} & \ms{99.99}{0.02} & \ms{99.92}{0.08} \\
 
\cdashline{2-7}
 
 & \normstatKL & \ms{82.58}{1.88} & \ms{83.08}{1.35} & \ms{73.59}{0.57} & \ms{99.83}{0.09} & \ms{98.43}{0.17} \\
 
 & \vecstatKL & \ms{97.41}{0.13} & \ms{99.84}{0.05} & \ms{99.97}{0.03} & \ms{99.99}{0.02} & \ms{99.93}{0.06} \\
 
 & \vecstatCos & \ms{98.46}{0.13} & \ms{99.83}{0.09} & \ms{100.00}{0.00} & \ms{99.99}{0.02} & \ms{99.94}{0.04} \\

\cdashline{2-7}

& \llmcall{0} & \ms{82.37}{0.30} & \ms{58.67}{0.67} & \ms{96.14}{0.41} & \ms{35.30}{0.88} & \ms{99.26}{0.07} \\

\midrule

\multirow{6}{*}{Llama-3.2-1B} & \embedavg & \ms{99.63}{0.21} & \ms{100.00}{0.00} & \ms{99.97}{0.03} & \ms{99.99}{0.02} & \ms{99.96}{0.02} \\

 & \embedlast & \ms{99.60}{0.18} & \ms{100.00}{0.00} & \ms{99.97}{0.03} & \ms{99.87}{0.00} & \ms{99.71}{0.28} \\

 & \linearavg & \ms{99.61}{0.08} & \ms{99.99}{0.02} & \ms{99.97}{0.03} & \ms{99.99}{0.02} & \ms{99.96}{0.02} \\
 
\cdashline{2-7}
 
 & \normstatKL & \ms{77.26}{6.21} & \ms{57.36}{3.14} & \ms{98.56}{0.11} & \ms{99.67}{0.20} & \ms{99.82}{0.19} \\
 
 & \vecstatKL & \ms{96.12}{0.39} & \ms{99.92}{0.02} & \ms{99.96}{0.02} & \ms{100.00}{0.00} & \ms{99.96}{0.02} \\
 
 & \vecstatCos & \ms{98.76}{0.07} & \ms{99.98}{0.02} & \ms{99.96}{0.02} & \ms{99.99}{0.02} & \ms{99.94}{0.04} \\

\midrule

\multirow{6}{*}{Llama-3.2-1B-Instruct} & \embedavg & \ms{99.72}{0.10} & \ms{99.99}{0.02} & \ms{99.94}{0.04} & \ms{99.99}{0.02} & \ms{99.97}{0.03} \\

 & \embedlast & \ms{99.58}{0.15} & \ms{100.00}{0.00} & \ms{99.96}{0.02} & \ms{99.94}{0.07} & \ms{99.82}{0.20} \\

 & \linearavg & \ms{99.67}{0.12} & \ms{100.00}{0.00} & \ms{99.94}{0.04} & \ms{99.99}{0.02} & \ms{99.96}{0.02} \\
 
\cdashline{2-7}
 
 & \normstatKL & \ms{79.47}{5.07} & \ms{42.79}{2.16} & \ms{98.54}{0.18} & \ms{99.86}{0.13} & \ms{99.86}{0.07} \\
 
 & \vecstatKL & \ms{96.04}{0.39} & \ms{99.84}{0.13} & \ms{99.98}{0.02} & \ms{99.99}{0.02} & \ms{99.96}{0.02} \\
 
 & \vecstatCos & \ms{98.59}{0.12} & \ms{99.98}{0.02} & \ms{99.94}{0.04} & \ms{99.99}{0.02} & \ms{99.94}{0.04} \\

\midrule

\multirow{7}{*}{Qwen3-4B} & \embedavg & \ms{99.71}{0.10} & \ms{100.00}{0.00} & \ms{99.94}{0.04} & \ms{99.99}{0.02} & \ms{99.96}{0.02} \\

 & \embedlast & \ms{99.60}{0.15} & \ms{100.00}{0.00} & \ms{99.97}{0.03} & \ms{99.99}{0.02} & \ms{99.99}{0.02} \\

 & \linearavg & \ms{99.74}{0.12} & \ms{100.00}{0.00} & \ms{99.94}{0.04} & \ms{99.99}{0.02} & \ms{99.96}{0.02} \\

 \cdashline{2-7}

 & \normstatKL & \ms{84.60}{0.34} & \ms{90.29}{1.27} & \ms{90.53}{0.61} & \ms{99.84}{0.07} & \ms{99.11}{0.54} \\
 
 & \vecstatKL & \ms{93.99}{0.82} & \ms{99.77}{0.09} & \ms{100.00}{0.00} & \ms{99.99}{0.02} & \ms{99.93}{0.06} \\
 
 & \vecstatCos & \ms{97.73}{0.12} & \ms{99.88}{0.10} & \ms{100.00}{0.00} & \ms{99.99}{0.02} & \ms{99.94}{0.04} \\

 \cdashline{2-7}

 & \llmcall{0} & \ms{97.00}{0.09} & \ms{99.16}{0.08} & \ms{98.56}{0.19} & \ms{88.37}{0.66} & \ms{99.93}{0.00} \\

\midrule

\multirow{7}{*}{Qwen3-8B} & \embedavg & \ms{99.70}{0.06} & \ms{100.00}{0.00} & \ms{99.94}{0.04} & \ms{99.99}{0.02} & \ms{99.97}{0.00} \\
 
 & \embedlast & \ms{99.53}{0.15} & \ms{100.00}{0.00} & \ms{99.97}{0.03} & \ms{99.99}{0.02} & \ms{99.92}{0.08} \\

 & \linearavg & \ms{99.81}{0.05} & \ms{100.00}{0.00} & \ms{99.94}{0.04} & \ms{99.99}{0.02} & \ms{99.96}{0.02} \\

 \cdashline{2-7}
 
 & \normstatKL & \ms{85.06}{0.85} & \ms{92.03}{1.32} & \ms{76.76}{0.77} & \ms{99.92}{0.04} & \ms{96.69}{2.25} \\
 
 & \vecstatKL & \ms{96.27}{0.38} & \ms{99.82}{0.05} & \ms{100.00}{0.00} & \ms{100.00}{0.00} & \ms{99.93}{0.06} \\
 
 & \vecstatCos & \ms{98.66}{0.13} & \ms{99.83}{0.09} & \ms{100.00}{0.00} & \ms{99.99}{0.02} & \ms{99.94}{0.04} \\

 \cdashline{2-7}

 & \llmcall{0} & \ms{97.64}{0.12} & \ms{97.28}{0.14} & \ms{99.96}{0.05} & \ms{13.83}{0.61} & \ms{99.89}{0.04} \\

\midrule
\multirow{6}{*}{Qwen3-32B} & \embedavg & \ms{99.76}{0.17} & \ms{99.99}{0.02} & \ms{99.94}{0.04} & \ms{99.99}{0.02} & \ms{99.96}{0.04} \\
 
 & \embedlast & \ms{99.31}{0.02} & \ms{100.00}{0.00} & \ms{99.96}{0.02} & \ms{99.99}{0.02} & \ms{99.94}{0.04} \\

 & \linearavg & \ms{99.87}{0.06} & \ms{100.00}{0.00} & \ms{99.94}{0.04} & \ms{99.98}{0.02} & \ms{99.93}{0.03} \\

 \cdashline{2-7}
 
 
 & \normstatKL & \ms{82.00}{0.74} & \ms{72.11}{1.63} & \ms{96.17}{0.40} & \ms{99.48}{0.10} & \ms{98.32}{0.35} \\
 
 & \vecstatKL & \ms{94.08}{0.27} & \ms{99.68}{0.18} & \ms{100.00}{0.00} & \ms{99.99}{0.02} & \ms{99.94}{0.04} \\
 
 & \vecstatCos & \ms{98.20}{0.12} & \ms{99.78}{0.07} & \ms{100.00}{0.00} & \ms{99.99}{0.02} & \ms{99.96}{0.02} \\

\cdashline{2-7}

& \llmcall{0} & \ms{97.66}{0.02} & \ms{98.34}{0.12} & \ms{99.96}{0.05} & \ms{99.84}{0.05} & \ms{100.00}{0.00} \\

\bottomrule
\end{tabular}%
}
\end{table}

\subsection{Additional Results for Task~4 (Level-2)}\label{appendix:l2-exp}

\Cref{tab:level-2-math-results} reports per-subfield accuracy for Task~4 (mathematical subfield classification) across all seven LLMs.

\begin{table}[h]
\centering
\small
\caption{\small Task~4 classification results for all seven LLMs (level-2 granularity). Values represent per-subfield accuracy across seven mathematical subfields from the Competition Math dataset.}
\label{tab:level-2-math-results}
\resizebox{\linewidth}{!}{%
\begin{tabular}{llccccccc}
\toprule
Model & Method & Algebra & Counting \& Probability & Geometry & Intermediate Algebra & Number Theory & Prealgebra & Precalculus \\
\midrule
\multirow{7}{*}{Qwen3-1.7B} & \embedavg & \ms{71.78}{5.57} & \ms{79.48}{1.91} & \ms{89.98}{3.46} & \ms{79.33}{3.96} & \ms{84.02}{2.56} & \ms{50.00}{2.29} & \ms{82.79}{3.37} \\
 
 & \embedlast & \ms{65.79}{3.55} & \ms{76.10}{4.24} & \ms{84.88}{4.43} & \ms{71.86}{6.02} & \ms{82.32}{3.80} & \ms{50.31}{2.59} & \ms{73.58}{4.30} \\

 & \linearavg & \ms{64.13}{3.32} & \ms{75.88}{0.79} & \ms{89.31}{1.47} & \ms{81.14}{3.20} & \ms{79.47}{4.34} & \ms{52.66}{2.86} & \ms{79.74}{1.54} \\

\cdashline{2-9}
 
 
 & \normstatKL & \ms{23.67}{1.27} & \ms{37.83}{0.65} & \ms{48.15}{4.09} & \ms{58.75}{1.08} & \ms{65.46}{1.25} & \ms{0.97}{0.28} & \ms{33.88}{1.02} \\

& \vecstatKL & \ms{33.29}{1.39} & \ms{53.71}{2.00} & \ms{35.88}{3.28} & \ms{81.88}{0.58} & \ms{86.26}{0.62} & \ms{1.46}{0.05} & \ms{46.07}{1.82} \\

& \vecstatCos & \ms{57.66}{0.70} & \ms{61.50}{2.16} & \ms{36.07}{3.37} & \ms{73.58}{0.83} & \ms{86.92}{0.38} & \ms{2.22}{0.32} & \ms{51.22}{1.61} \\

\cdashline{2-9}

 &\llmcall{0} & \ms{28.30}{1.21} & \ms{53.18}{1.06} & \ms{55.49}{1.47} & \ms{1.31}{0.15} & \ms{14.61}{0.72} & \ms{7.37}{0.16} & \ms{71.48}{2.46} \\
 
\midrule

\multirow{7}{*}{Qwen3-1.7B-Base} & \embedavg & \ms{67.96}{1.83} & \ms{79.55}{0.67} & \ms{90.53}{0.55} & \ms{83.93}{2.34} & \ms{84.02}{2.62} & \ms{48.85}{2.19} & \ms{83.47}{3.46} \\
 
 & \embedlast & \ms{70.97}{4.48} & \ms{80.97}{2.25} & \ms{90.89}{1.14} & \ms{74.49}{2.32} & \ms{81.50}{6.42} & \ms{51.93}{7.96} & \ms{78.25}{2.55} \\

 & \linearavg & \ms{66.46}{4.16} & \ms{75.81}{2.04} & \ms{88.95}{1.00} & \ms{81.57}{1.99} & \ms{81.06}{2.52} & \ms{51.33}{2.29} & \ms{79.54}{1.31} \\

 \cdashline{2-9}
 
 
 & \normstatKL & \ms{23.35}{0.70} & \ms{37.23}{2.12} & \ms{47.36}{4.20} & \ms{56.65}{1.30} & \ms{58.84}{2.38} & \ms{5.67}{2.17} & \ms{34.82}{0.82} \\

& \vecstatKL & \ms{39.48}{3.28} & \ms{50.86}{1.72} & \ms{35.88}{3.37} & \ms{81.31}{0.65} & \ms{88.12}{0.58} & \ms{1.52}{0.52} & \ms{46.41}{1.24} \\

& \vecstatCos & \ms{61.24}{0.90} & \ms{62.62}{2.31} & \ms{36.00}{3.55} & \ms{73.46}{0.80} & \ms{88.18}{0.28} & \ms{2.66}{0.67} & \ms{53.59}{0.82} \\

\cdashline{2-9}

& \llmcall{0} & \ms{82.37}{0.30} & \ms{58.67}{0.67} & \ms{96.14}{0.41} & \ms{35.30}{0.88} & \ms{99.26}{0.07} \\
 
\midrule

\multirow{6}{*}{Llama-3.2-1B} & \embedavg & \ms{57.99}{6.53} & \ms{75.13}{5.42} & \ms{86.22}{4.37} & \ms{75.94}{4.00} & \ms{78.43}{0.09} & \ms{41.67}{5.38} & \ms{76.02}{1.42} \\

 & \embedlast & \ms{58.38}{4.88} & \ms{78.35}{2.40} & \ms{91.20}{1.28} & \ms{74.44}{3.57} & \ms{77.45}{8.21} & \ms{36.26}{4.71} & \ms{67.01}{5.55} \\

 & \linearavg & \ms{53.70}{6.28} & \ms{70.94}{5.13} & \ms{81.60}{9.17} & \ms{74.15}{2.30} & \ms{80.41}{7.23} & \ms{44.36}{5.95} & \ms{76.96}{7.70} \\

 \cdashline{2-9}

 
 & \normstatKL & \ms{19.46}{3.91} & \ms{16.33}{1.35} & \ms{38.07}{2.98} & \ms{37.01}{3.69} & \ms{80.24}{2.55} & \ms{0.05}{0.05} & \ms{27.71}{1.73} \\

& \vecstatKL & \ms{34.57}{0.92} & \ms{50.34}{1.92} & \ms{38.19}{4.30} & \ms{76.13}{0.97} & \ms{84.89}{0.59} & \ms{1.38}{0.12} & \ms{48.17}{2.66} \\
 
& \vecstatCos & \ms{44.71}{0.78} & \ms{59.48}{2.26} & \ms{44.63}{5.73} & \ms{72.65}{0.35} & \ms{83.96}{1.19} & \ms{2.30}{0.32} & \ms{52.57}{0.92} \\

\midrule

\multirow{6}{*}{Llama-3.2-1B-Instruct} & \embedavg & \ms{64.88}{4.97} & \ms{78.95}{1.01} & \ms{86.76}{4.30} & \ms{76.82}{5.70} & \ms{85.39}{0.43} & \ms{45.27}{2.58} & \ms{84.35}{2.00} \\

 & \embedlast & \ms{73.63}{4.07} & \ms{79.48}{2.60} & \ms{88.71}{2.03} & \ms{78.02}{2.53} & \ms{83.58}{0.16} & \ms{51.28}{3.11} & \ms{83.47}{1.84} \\

 & \linearavg & \ms{57.23}{5.70} & \ms{77.83}{1.72}& \ms{83.36}{2.16} & \ms{82.14}{0.69} & \ms{80.57}{1.53} & \ms{48.75}{2.43} & \ms{82.05}{0.96} \\

 \cdashline{2-9}
 
 
 & \normstatKL & \ms{4.38}{3.24} & \ms{11.69}{1.57} & \ms{36.13}{3.19} & \ms{37.34}{3.24} & \ms{84.40}{1.00} & \ms{0.26}{0.12} & \ms{27.98}{4.39} \\
  
  & \vecstatKL & \ms{40.03}{2.62} & \ms{45.09}{1.50} & \ms{35.88}{3.46} & \ms{74.32}{1.49} & \ms{88.67}{0.43} & \ms{1.25}{0.14} & \ms{46.75}{2.34} \\
 
 & \vecstatCos & \ms{53.20}{1.76} & \ms{59.85}{1.80} & \ms{36.79}{4.02} & \ms{72.72}{1.04} & \ms{86.59}{0.19} & \ms{1.78}{0.09} & \ms{52.17}{1.54} \\


\midrule

 \multirow{7}{*}{Qwen3-4B} & \embedavg & \ms{73.52}{1.54} & \ms{79.63}{4.68} & \ms{89.80}{2.86} & \ms{79.16}{1.48} & \ms{80.84}{0.78} & \ms{51.23}{3.88} & \ms{84.62}{2.94} \\
 
 & \embedlast & \ms{68.98}{9.23} & \ms{82.47}{1.96} & \ms{82.70}{2.37} & \ms{67.84}{7.26} & \ms{81.55}{3.51} & \ms{46.89}{4.30} & \ms{79.95}{1.50} \\

 & \linearavg & \ms{70.64}{2.89} & \ms{80.07}{0.34} & \ms{90.59}{3.57} & \ms{80.02}{2.61} & \ms{79.97}{3.76} & \ms{49.53}{5.57} & \ms{81.91}{2.12} \\

 \cdashline{2-9}
 
 & \normstatKL & \ms{19.57}{3.97} & \ms{17.15}{0.91} & \ms{42.26}{3.41} & \ms{47.00}{2.72} & \ms{76.79}{2.63} & \ms{0.60}{0.25} & \ms{34.28}{1.70} \\
 
 & \vecstatKL & \ms{29.80}{3.56} & \ms{51.99}{1.30} & \ms{35.82}{3.37} & \ms{82.31}{1.15} & \ms{89.00}{0.28} & \ms{1.36}{0.05} & \ms{54.81}{3.00} \\
 
 & \vecstatCos & \ms{62.24}{2.09} & \ms{61.72}{2.03} & \ms{36.13}{3.19} & \ms{73.89}{1.24} & \ms{88.83}{0.33} & \ms{1.96}{0.08} & \ms{53.59}{1.32} \\

\cdashline{2-9}

& \llmcall{0} & \ms{6.05}{0.17} & \ms{78.88}{0.81} & \ms{59.14}{1.17} & \ms{60.21}{0.32} & \ms{67.16}{1.78} & \ms{21.39}{0.77} & \ms{80.69}{0.20} \\

 \midrule

 \multirow{7}{*}{Qwen3-8B} & \embedavg & \ms{73.94}{1.63} & \ms{82.70}{0.98} & \ms{87.13}{0.56} & \ms{81.00}{3.61} & \ms{86.59}{4.35} & \ms{52.95}{5.24} & \ms{89.36}{0.42} \\
 
 & \embedlast & \ms{68.43}{1.81} & \ms{80.52}{2.03} & \ms{89.92}{1.91} & \ms{73.87}{6.12} & \ms{82.92}{5.55} & \ms{53.58}{0.64} & \ms{81.10}{5.72} \\

 & \linearavg & \ms{73.28}{4.60} & \ms{80.82}{3.83} & \ms{90.29}{2.47} & \ms{79.92}{3.87} & \ms{82.65}{0.90} & \ms{52.19}{3.70} & \ms{84.76}{0.54} \\

 \cdashline{2-9}
 
 & \normstatKL & \ms{21.01}{6.35} & \ms{19.03}{1.50} & \ms{40.19}{3.56} & \ms{46.38}{0.91} & \ms{73.45}{4.22} & \ms{0.73}{0.59} & \ms{31.98}{1.31} \\
 
 & \vecstatKL & \ms{36.23}{3.18} & \ms{51.24}{1.17} & \ms{35.88}{3.28} & \ms{81.00}{1.60} & \ms{88.89}{0.47} & \ms{1.52}{0.12} & \ms{55.15}{3.27} \\
 
 & \vecstatCos & \ms{64.46}{1.45} & \ms{61.95}{1.87} & \ms{36.25}{3.19} & \ms{74.13}{0.95} & \ms{89.33}{0.33} & \ms{2.06}{0.40} & \ms{55.08}{1.08} \\

\cdashline{2-9}

& \llmcall{0} & \ms{44.49}{0.22} & \ms{83.22}{1.58} & \ms{78.32}{2.46} & \ms{67.76}{0.44} & \ms{52.38}{2.21} & \ms{18.70}{1.02} & \ms{64.63}{1.63} \\

 \midrule

 \multirow{7}{*}{Qwen3-32B} & \embedavg & \ms{73.25}{4.75} & \ms{77.83}{2.89} & \ms{89.01}{3.74} & \ms{84.14}{3.27} & \ms{81.88}{2.64} & \ms{57.92}{4.22} & \ms{86.65}{3.02} \\
 
 & \embedlast & \ms{66.70}{8.95} & \ms{78.43}{2.02} & \ms{88.52}{4.65} & \ms{79.61}{3.33} & \ms{74.06}{11.10} & \ms{50.57}{10.68} & \ms{80.62}{3.30} \\

& \linearavg & \ms{71.97}{5.76} & \ms{79.93}{3.82} & \ms{87.86}{0.92} & \ms{82.33}{3.88} & \ms{85.06}{0.72} & \ms{52.82}{4.70} & \ms{87.53}{1.89} \\

\cdashline{2-9}
 
 
 & \normstatKL & \ms{31.99}{2.02} & \ms{21.72}{0.85} & \ms{40.56}{3.48} & \ms{44.61}{1.14} & \ms{75.53}{2.56} & \ms{0.13}{0.05} & \ms{31.44}{1.47} \\
 
 & \vecstatKL & \ms{55.76}{1.11} & \ms{58.50}{2.04} & \ms{35.94}{3.28} & \ms{77.80}{1.04} & \ms{89.87}{0.34} & \ms{1.46}{0.12} & \ms{52.78}{0.59} \\
 
 & \vecstatCos & \ms{71.16}{0.46} & \ms{64.42}{2.16} & \ms{37.04}{3.20} & \ms{73.32}{0.82} & \ms{89.49}{0.87} & \ms{3.00}{0.72} & \ms{58.47}{1.31} \\

 \cdashline{2-9}

& \llmcall{0} & \ms{69.28}{0.36} & \ms{86.59}{1.44} & \ms{91.92}{0.42} & \ms{35.81}{1.07} & \ms{59.61}{1.31} & \ms{20.04}{0.51} & \ms{63.28}{0.12} \\
 
\bottomrule
\end{tabular}%
}
\end{table}

\begin{table}[t]
        \centering
        \caption{\small Robustness to adversarial attack: accuracy on adversarial MATH500 across all seven LLMs and three difficulty tiers (Easy, Medium, Hard).}
        \label{tab:adv_dataset}
        \small
        \begin{tabular}{llccc}
        \toprule
        Model & Method & Easy & Medium & Hard \\
        \midrule
        \multirow{5}{*}{Qwen3-1.7B} & \embedavg & \ms{21.07}{1.55} & \ms{12.07}{0.61} & \ms{0.00}{0.00} \\
 & \embedlast & \ms{89.13}{13.28} & \ms{54.87}{33.39} & \ms{0.00}{0.00} \\

 \cdashline{2-5}
 & \normstatKL & \ms{62.67}{0.61} & \ms{47.73}{0.50} & \ms{0.00}{0.00} \\
 & \vecstatKL & \ms{64.93}{0.46} & \ms{34.87}{0.76} & \ms{0.00}{0.00} \\
 & \vecstatCos & \ms{81.27}{0.50} & \ms{63.60}{0.20} & \ms{0.20}{0.00} \\
\midrule
\multirow{5}{*}{Qwen3-1.7B-Base} & \embedavg & \ms{2.60}{2.60} & \ms{0.87}{1.17} & \ms{0.00}{0.00} \\
 & \embedlast & \ms{69.73}{19.06} & \ms{40.73}{23.60} & \ms{0.00}{0.00} \\

 \cdashline{2-5}
 
 & \normstatKL & \ms{10.33}{0.23} & \ms{1.20}{0.00} & \ms{0.00}{0.00} \\
 & \vecstatKL & \ms{0.53}{0.12} & \ms{0.00}{0.00} & \ms{0.00}{0.00} \\
 & \vecstatCos & \ms{15.53}{0.46} & \ms{8.13}{0.12} & \ms{0.00}{0.00} \\
\midrule
\multirow{5}{*}{Qwen3-32B} & \embedavg & \ms{11.73}{8.92} & \ms{0.40}{0.69} & \ms{0.00}{0.00} \\
 & \embedlast & \ms{61.60}{14.67} & \ms{0.93}{0.81} & \ms{0.20}{0.00} \\

 \cdashline{2-5}
 
 & \normstatKL & \ms{92.33}{0.12} & \ms{80.67}{0.12} & \ms{0.00}{0.00} \\
 & \vecstatKL & \ms{70.93}{0.31} & \ms{32.33}{0.31} & \ms{0.00}{0.00} \\
 & \vecstatCos & \ms{64.87}{0.58} & \ms{30.33}{0.64} & \ms{0.00}{0.00} \\
\midrule
\multirow{5}{*}{Qwen3-4B} & \embedavg & \ms{18.47}{3.19} & \ms{10.27}{2.34} & \ms{0.00}{0.00} \\
 & \embedlast & \ms{60.73}{16.21} & \ms{8.33}{3.83} & \ms{0.00}{0.00} \\

 \cdashline{2-5}
 
 & \normstatKL & \ms{31.87}{0.31} & \ms{28.33}{0.42} & \ms{0.00}{0.00} \\
 & \vecstatKL & \ms{36.33}{0.12} & \ms{23.33}{0.23} & \ms{0.00}{0.00} \\
 & \vecstatCos & \ms{79.47}{0.31} & \ms{55.80}{0.72} & \ms{0.20}{0.00} \\
\midrule
\multirow{5}{*}{Qwen3-8B} & \embedavg & \ms{12.67}{7.82} & \ms{4.20}{3.30} & \ms{0.00}{0.00} \\
 & \embedlast & \ms{36.67}{31.67} & \ms{1.47}{2.20} & \ms{0.00}{0.00} \\

\cdashline{2-5}
 
 & \normstatKL & \ms{36.13}{0.50} & \ms{27.00}{0.40} & \ms{0.00}{0.00} \\
 & \vecstatKL & \ms{41.00}{0.69} & \ms{25.00}{0.35} & \ms{0.00}{0.00} \\
 & \vecstatCos & \ms{78.60}{1.39} & \ms{54.80}{0.69} & \ms{0.20}{0.00} \\
\midrule
\multirow{5}{*}{Llama-3.2-1B} & \embedavg & \ms{27.07}{8.03} & \ms{14.80}{4.16} & \ms{0.47}{0.42} \\
 & \embedlast & \ms{95.73}{2.20} & \ms{86.80}{8.01} & \ms{6.87}{11.89} \\

\cdashline{2-5}
 
 & \normstatKL & \ms{87.60}{0.00} & \ms{59.20}{0.35} & \ms{5.33}{0.64} \\
 & \vecstatKL & \ms{67.67}{0.12} & \ms{29.67}{0.64} & \ms{1.80}{0.00} \\
 & \vecstatCos & \ms{56.73}{0.58} & \ms{32.40}{0.40} & \ms{2.33}{0.12} \\
\midrule
\multirow{5}{*}{Llama-3.2-1B-Instruct} & \embedavg & \ms{24.60}{9.04} & \ms{19.13}{6.13} & \ms{0.87}{0.99} \\
 & \embedlast & \ms{72.20}{8.72} & \ms{31.27}{8.88} & \ms{0.20}{0.20} \\

\cdashline{2-5}
 
 & \normstatKL & \ms{97.60}{0.00} & \ms{85.13}{0.23} & \ms{11.47}{0.81} \\
 & \vecstatKL & \ms{85.73}{0.12} & \ms{60.00}{0.53} & \ms{3.40}{0.40} \\
 & \vecstatCos & \ms{77.33}{0.12} & \ms{56.40}{0.35} & \ms{4.13}{0.12} \\

 \midrule

 \revise{RoBERTa} & - & \revise{\ms{10.60}{4.72}} & \revise{\ms{5.60}{7.66}} & \revise{{0.00}} \\
 
        \bottomrule
        \end{tabular}%
        \end{table}

\subsection{Additional Results for Uncertainty Quantification for Mixed-Intent Prompts}\label{appendix:mix-intent}
We construct a mixed-intent dataset by interleaving samples from both math and code datasets at five known mix ratios (math:code) \(r \in \{0:1,\ 0.25:0.75,\ 0.5:0.5,\ 0.75:0.25,\ 1:0\}\), assessed across two prompt orderings (code-first and math-first) to probe sensitivity to token-sequence position.
For each method, a scalar temperature $T$ is fitted via least
squares to calibrate raw scores into probabilities:
\[\hat{T} = \arg\min_{T} \sum_{i}
    \Bigl(\sigma\!\bigl(\Delta s_i / T\bigr) - r_i\Bigr)^2, \quad r_i \in \{0:1,\ 0.25:0.75,\ 0.5:0.5,\ 0.75:0.25,\ 1:0\}\]
where $\Delta s_i = s_{\text{math}} - s_{\text{code}}$ is the score
difference at ratio $r_i$ ($s = {-}d$ for distance-based methods;
$s = \text{logit}$ for MLP-based methods), and $\sigma$ is the sigmoid
function.

The calibration error for temperature of each method is calculated by the Root Mean Square Error (RMSE): $\mathrm{RMSE} = \sqrt{\frac{1}{N}\sum_{i=1}^{N}(\hat{p}_i - p_i^*)^2}$,
           where $\hat{p}_i = P(\mathrm{math})$ is the predicted math probability and
           $p_i^*$ is the target math fraction, evaluated at mix-ratio points.

\cref{fig:mixed-intent-prompt} presents an example of the mixed-intent prompt used in~\cref{subsec:mixed-intent}, constructed by concatenating code content with mathematical content. We conducted the experiments across four models, with the order \emph{math\_first} or \emph{code\_first}, with results shown in~\cref{fig:mix-intent-all}.  Moreover, we report in \cref{tab:mix-intent-calibration} the corresponding calibration error. 

\begin{figure}[h]
    \centering
    
    \includegraphics[width=0.998\textwidth]{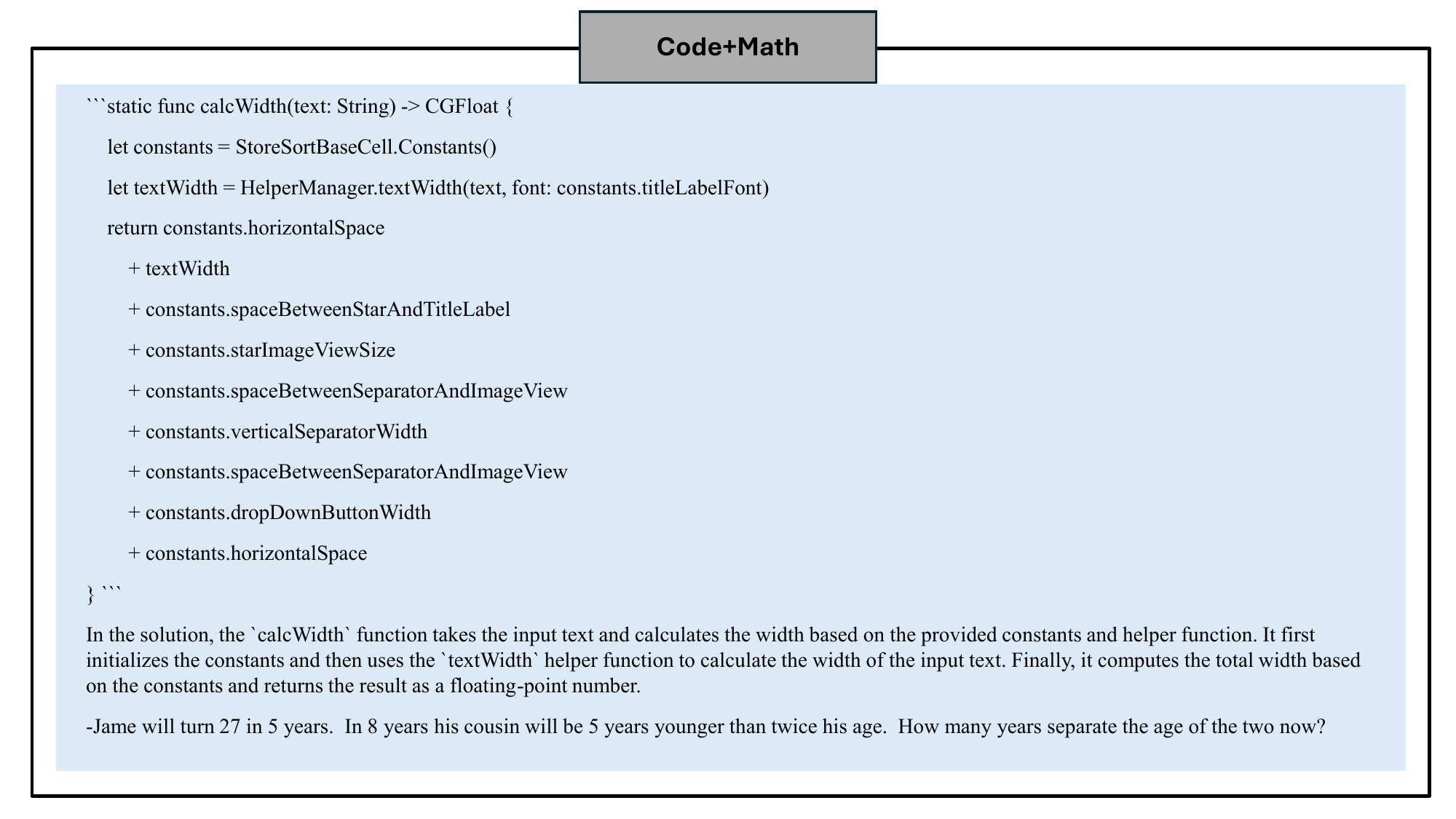} 
    \caption{\small Mixed-intent prompt example combining Swift code (width calculation function) and a mathematical problem used for uncertainty quantification analysis.}
    \label{fig:mixed-intent-prompt}
\end{figure}

\begin{figure}[htb]
    \centering
    \includegraphics[width=0.95\linewidth]{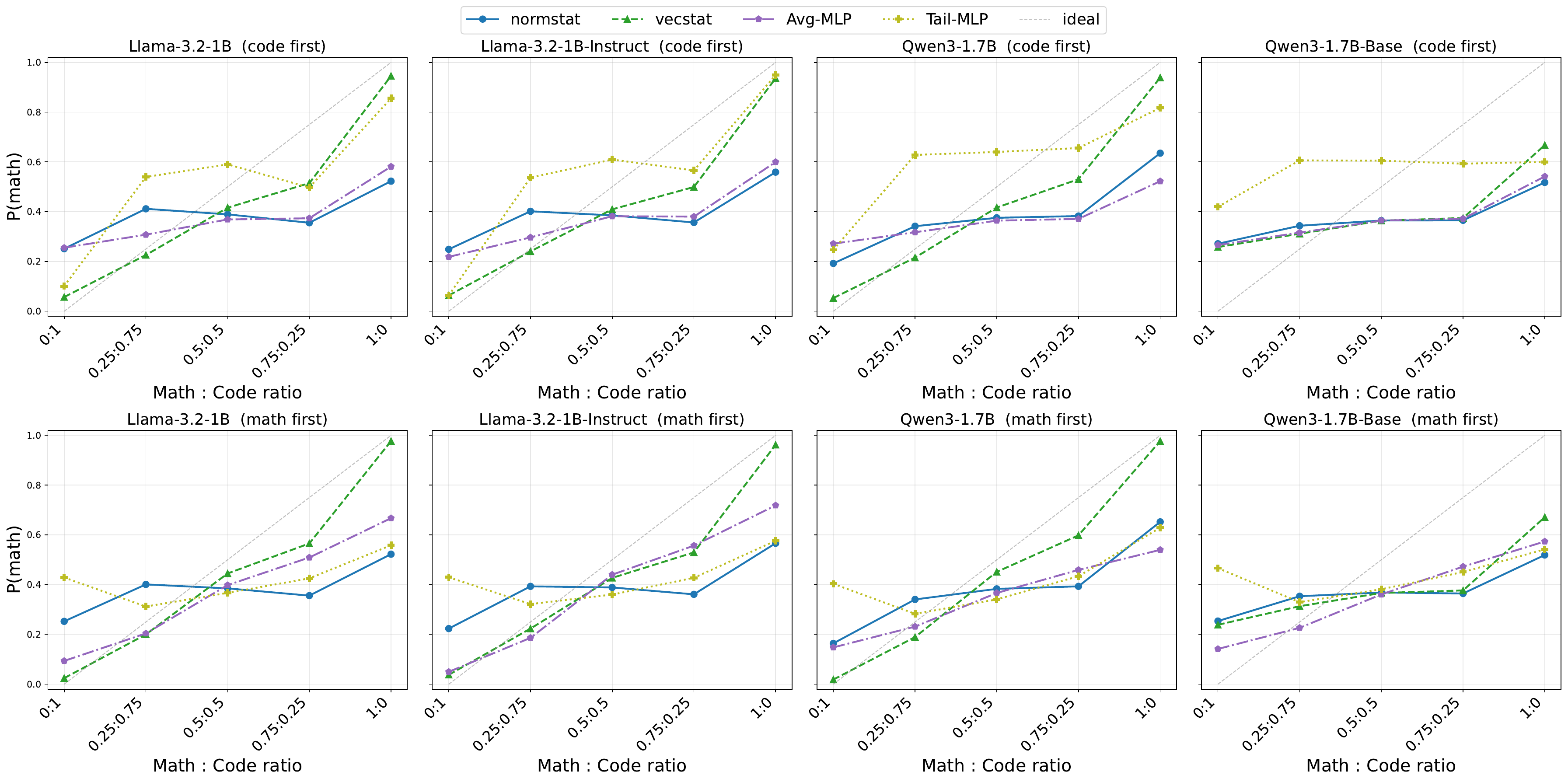}
\caption{Predicted $P(\text{math})$ vs.\ mix ratio for all models and prompt orderings.
Each curve is temperature-calibrated; the dashed diagonal is the ideal response.
Upper and lower subplots show \textit{code-first} and \textit{math-first} prompt orderings, respectively.}
    \label{fig:mix-intent-all}
\end{figure}

\begin{table}[htb]
  \centering
  \caption{Calibration error of each method per model and prompt ordering.
           \textbf{Bold} marks the best method per row.}
  \label{tab:mix-intent-calibration}
  \begin{tabular}{llcccc}
    \toprule
    \textbf{Model} & \textbf{Order} & normstat & vecstat & Avg-MLP & Tail-MLP \\
    \midrule
    \multirow{2}{*}{Llama-3.2-1B}
      & code-first & 0.2535 & \textbf{0.1286} & 0.2384 & 0.2038 \\
      & math-first & 0.2525 & \textbf{0.1002} & 0.1408 & 0.2787 \\
    \midrule
    \multirow{2}{*}{Llama-3.2-1B-Instruct}
      & code-first & 0.2514 & \textbf{0.1371} & 0.2239 & 0.1819 \\
      & math-first & 0.2419 & 0.1183 & \textbf{0.1089} & 0.2799 \\
    \midrule
    \multirow{2}{*}{Qwen3-1.7B}
      & code-first & 0.2213 & \textbf{0.1216} & 0.2451 & 0.2412 \\
      & math-first & 0.2098 & \textbf{0.0858} & 0.1764 & 0.2688 \\
    \midrule
    \multirow{2}{*}{Qwen3-1.7B-Base}
      & code-first & 0.2493 & \textbf{0.2396} & 0.2431 & 0.2910 \\
      & math-first & 0.2455 & 0.2335 & \textbf{0.1710} & 0.2861 \\
    \bottomrule
  \end{tabular}
\end{table}

\subsection{Additional Results for Robustness to Adversarial Attack}\label{appendix:adv}

\revise{\Cref{tab:adv_dataset} reports full adversarial robustness results across all seven LLMs. 
The adversarial prompts are generated by GPT-4o using the templates described in~\cref{tab:adv_prompt_templates}, which specify three increasing levels of disguise: Easy adds a thin lexical code veneer, Medium introduces a function header with a prose mathematical body, and Hard embeds the problem in a bug-report framing. 
Across all levels, the underlying mathematical task is held fixed while misleading code-like cues are progressively injected, allowing us to test whether classifiers follow semantic intent or surface form. 
The resulting dataset is therefore intended as a controlled stress test rather than simulating typical human-authored adversarial prompts or real production failures. 
GPT5-Nano and GPT5 is used only to sanity-check the intended difficulty ordering, not to validate the realism of the generated prompts; accordingly, the results characterize robustness to this specific family of synthetic rephrasings.}

The main observations from \cref{sec:adv} hold consistently: performance degrades monotonically with difficulty, and training-free methods are substantially more robust than training-based ones for larger Qwen models. 
\revise{RoBERTa similarly drops to \(10.60\%\) on Easy and \(5.60\%\) on Medium, showing that the robustness gap is not specific to MLP heads on LLM features.}
We highlight several additional patterns.

\paragraph{Effect of post-training.} Comparing Qwen3-1.7B against Qwen3-1.7B-Base reveals a striking degradation in the base model: almost all methods drop to near-zero even on the Easy tier, indicating that post-training alignment substantially improves adversarial robustness across both paradigms.

\paragraph{Consistency across Qwen scales.} The robustness advantage of training-free methods over \embedavg{} holds across all Qwen model sizes. \vecstatCos{} is the strongest training-free method for smaller and mid-size Qwen models, while \normstatKL{} dominates for Qwen3-32B. This model-dependence of the best training-free variant motivates treating method selection as a deployment decision informed by the target model family and scale.

\begin{table}[htb]
\centering
\caption{\small Adversarial prompt templates at three difficulty levels, each used with GPT-4o. A shared system prompt instructs the model to preserve all mathematical facts, numbers, equations, and constraints from the original, ensuring identical answers and identical reasoning; only framing, vocabulary, formatting, and surrounding scenario may change. Original mathematics problems are inserted at \texttt{\{problem\}}.}
\label{tab:adv_prompt_templates}
\small
\begin{tabular}{@{} p{0.32\textwidth} p{0.32\textwidth} p{0.32\textwidth} @{}}
\toprule
\textbf{Easy} & \textbf{Medium} & \textbf{Hard} \\
\midrule
\multicolumn{3}{c}{\textit{Increasing depth of disguise: lexical $\;\rightarrow\;$ structural $\;\rightarrow\;$ narrative}} \\[0.4em]
\textit{Goal:} Keep the original math problem mostly intact and add only a thin lexical code veneer.
&
\textit{Goal:} Present the problem with a short untyped Python function header followed by a purely prose mathematical description.
&
\textit{Goal:} Embed the problem in a realistic internal bug report for a misbehaving helper function.
\\[0.4em]

\begin{itemize}[leftmargin=1.2em,topsep=1pt,itemsep=1pt,parsep=0pt]
    \item Preserve original sentence order and phrasing.
    \item Name exactly one programming language.
    \item Add a minimal wrapper (\texttt{solve()} or \texttt{compute()}).
    \item Rephrase the final ask as ``what value should the function return?''
    \item Add only 2--5 code-signaling terms: \emph{function}, \emph{return}, \emph{parameter}, \emph{variable}, \emph{input}.
    \item No code block; no extra variables.
    \item Length cap: ${\leq}1.5\times$ original.
\end{itemize}

&

\begin{itemize}[leftmargin=1.2em,topsep=1pt,itemsep=1pt,parsep=0pt]
    \item Begin with one untyped \texttt{def} header; use plain parameter names matching the problem's given quantities; no type annotations; no return-type arrow.
    \item Prose body after the header describes the problem in ordinary sentences, preserving every mathematical fact.
    \item End by asking what value or expression the function should return.
    \item No type annotations, backticks, code blocks, docstrings, function body, loops, conditionals, or data structures.
    \item Length cap: ${\leq}2\times$ original.
\end{itemize}

&

\begin{itemize}[leftmargin=1.2em,topsep=1pt,itemsep=1pt,parsep=0pt]
    \item Issue title referencing a function name.
    \item 2--4 structured sections (e.g., \textsc{Context}, \textsc{Repro}, \textsc{Observed}, \textsc{Expected}); \textsc{Context} must embed all mathematical facts sufficiently to solve the problem.
    \item Example call with original problem values; output marked \texttt{?} (e.g., \texttt{f(a=3, b=4)\ ?}).
    \item End with a 1-sentence reviewer-style comment.
    \item No function body, stack traces, logs, diff hunks, or file paths.
    \item Length cap: ${\leq}4\times$ original.
\end{itemize}

\\
\midrule
\multicolumn{3}{@{} p{0.885\textwidth} @{}}{%
\textbf{Shared system constraints (all levels):}
preserve every mathematical fact, symbol, equation, and relationship from the original;
the rewritten problem must have the same answer and require the same mathematical reasoning;
do not add programming-knowledge requirements, hints, edge cases, type checks, implementation details, or extra assumptions;
output only the rewritten problem text.} \\
\bottomrule
\end{tabular}
\end{table}

\subsection{The effect of the number of layers considered} \label{appendix-nb_layers}
See~\cref{fig:nb_layers}.

\begin{figure}[h]
    \centering
    
    \includegraphics[width=0.998\textwidth]{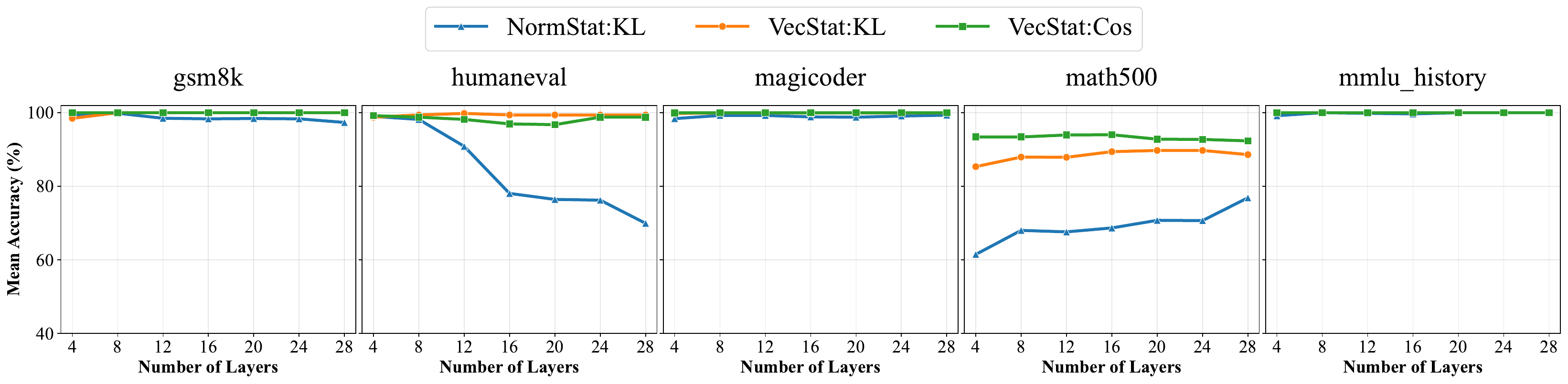} 
    
    \caption{\small Effect of the number of layers on level-1 classification accuracy for Qwen3-1.7B.}
    \label{fig:nb_layers}
\end{figure}

    
    
    


\subsection{The effect of the maximum prompt length}
\label{appendix:seqlen}

See~\cref{fig:seqlen}.

\begin{figure}[h] 
    \centering
    \includegraphics[width=0.98\textwidth]{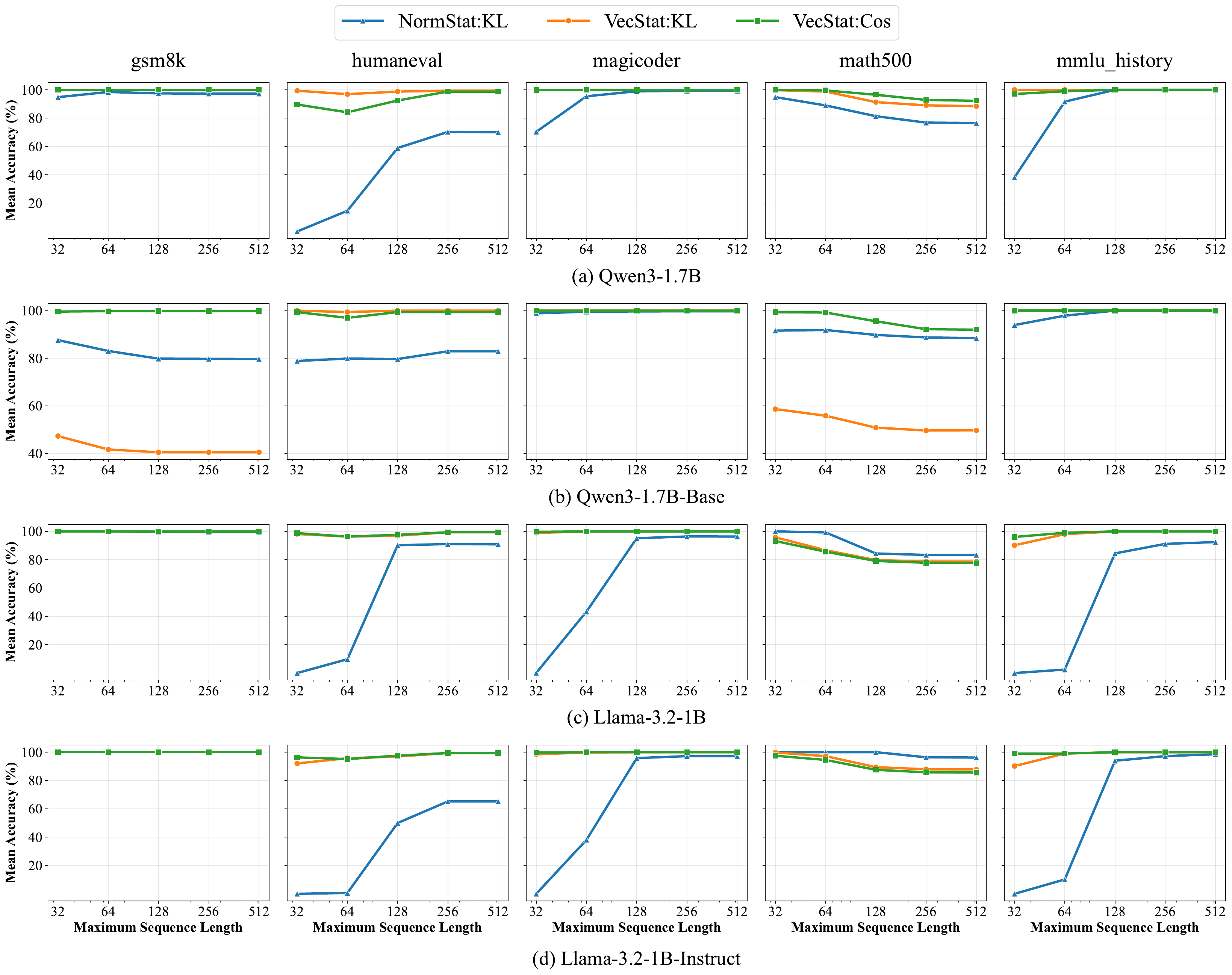} 
    \caption{\small Effect of the maximum prompt length on level-1 classification accuracy for Qwen3-1.7B (a), Qwen3-1.7B-Base (b), Llama-3.2-1B (c), and Llama-3.2-1B-Instruct (d)..}
    \label{fig:seqlen}
\end{figure}

\subsection{Calibration Convergence Analysis}\label{appendix:calibration}
See~\cref{fig:calibration-convergence-appendix}.
  \begin{figure}[h]
      \centering
      \begin{subfigure}[b]{0.48\textwidth}
          \centering
          \includegraphics[width=\textwidth]{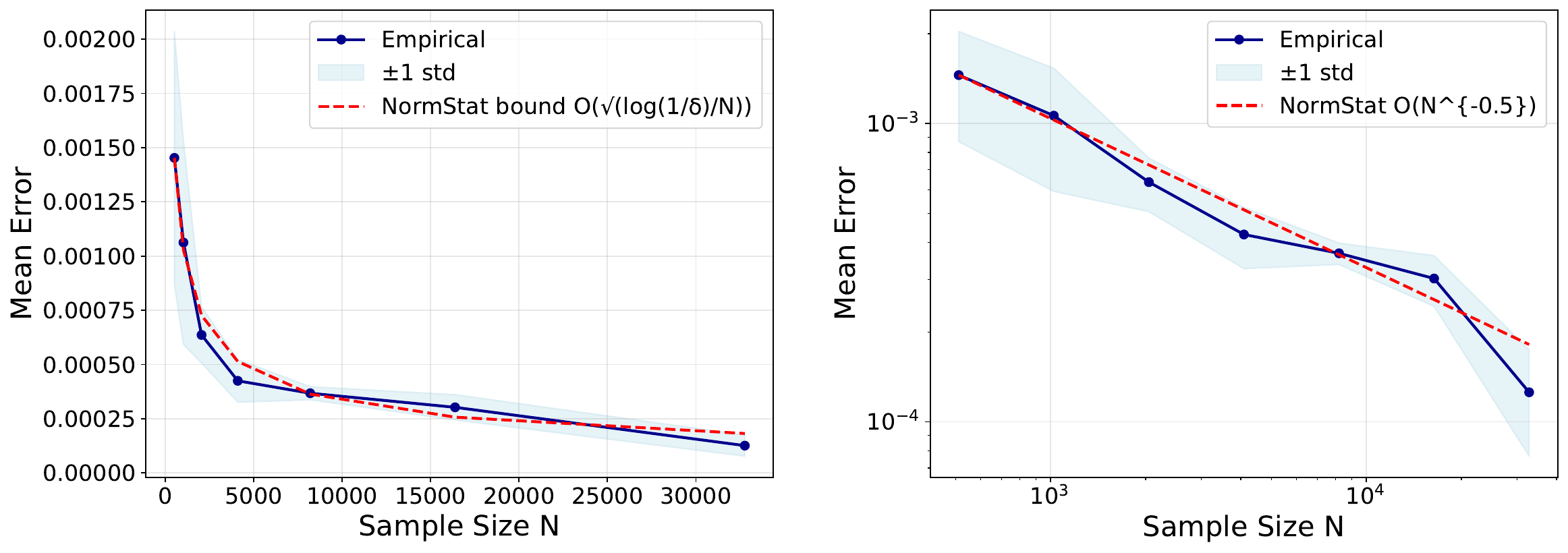}
          \caption{Llama-3.2-1B-Instruct  (\normstat)}
          \label{fig:llama-instruct-magicoder-norm}
      \end{subfigure}
      \hfill
      \begin{subfigure}[b]{0.48\textwidth}
          \centering
          \includegraphics[width=\textwidth]{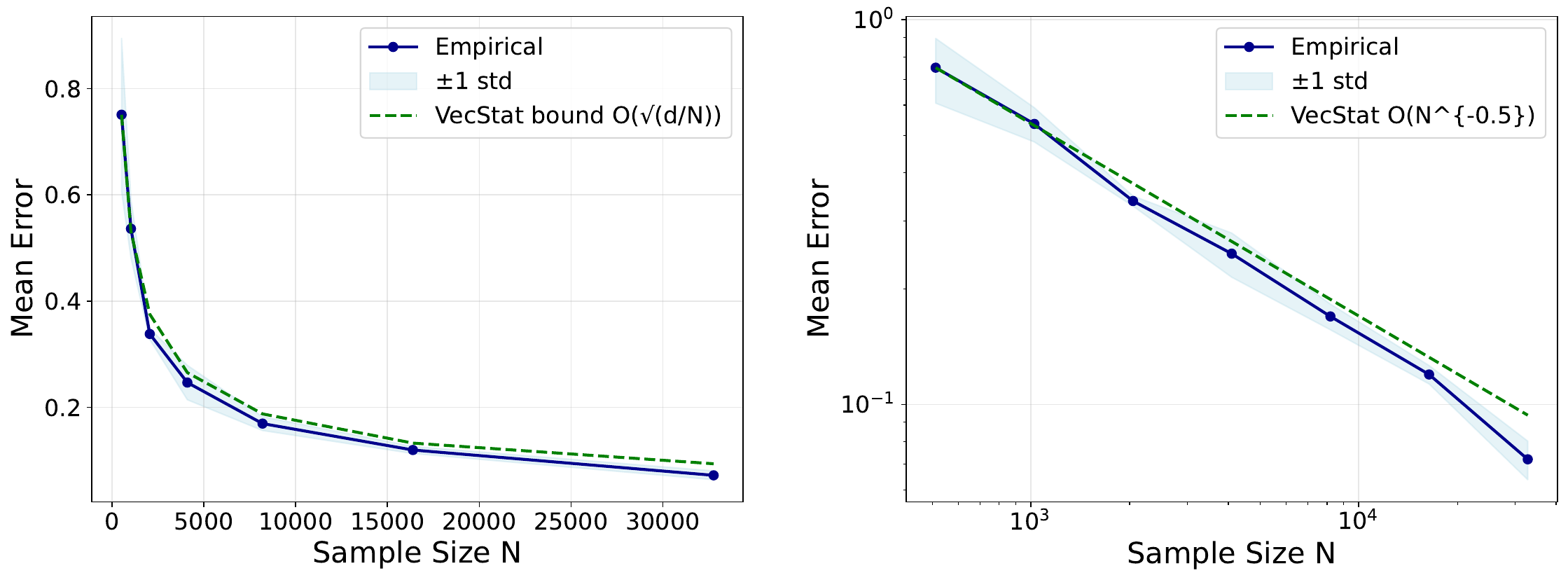}
          \caption{Llama-3.2-1B-Instruct  (\vecstat)}
          \label{fig:llama-instruct-magicoder-proj}
      \end{subfigure}

      \vspace{0.5cm}

      \begin{subfigure}[b]{0.48\textwidth}
          \centering
          \includegraphics[width=\textwidth]{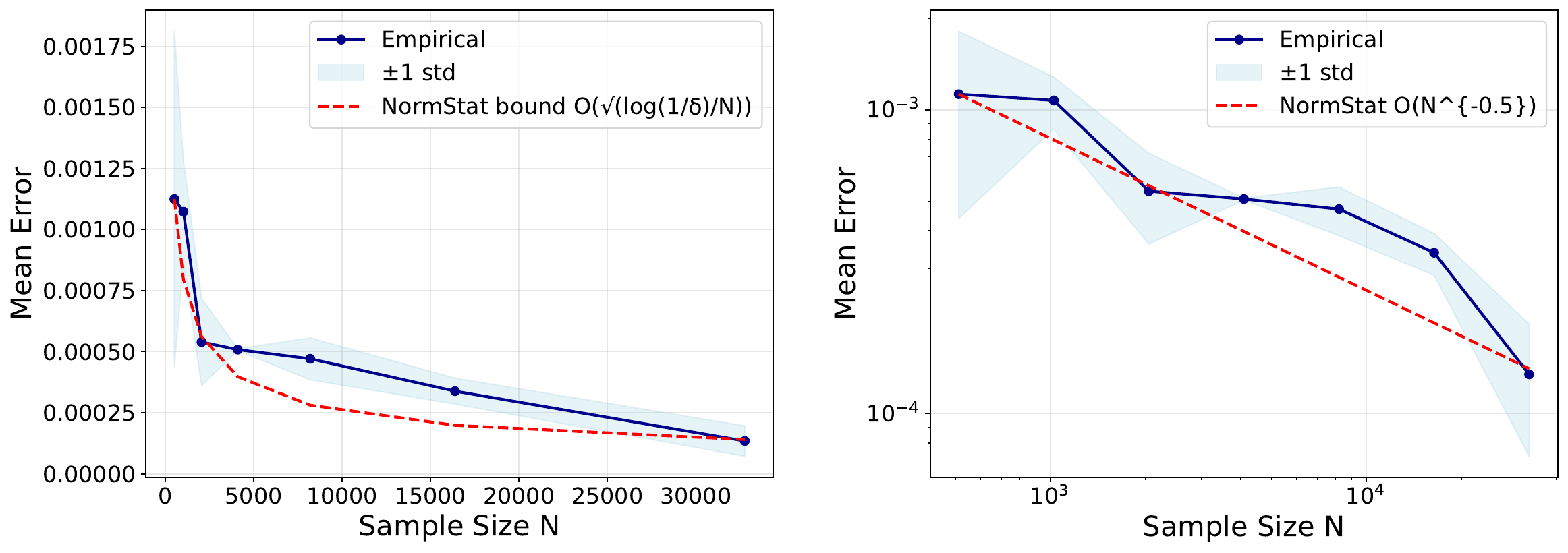}
          \caption{Llama-3.2-1B  (\normstat)}
          \label{fig:llama-magicoder-norm}
      \end{subfigure}
      \hfill
      \begin{subfigure}[b]{0.48\textwidth}
          \centering
          \includegraphics[width=\textwidth]{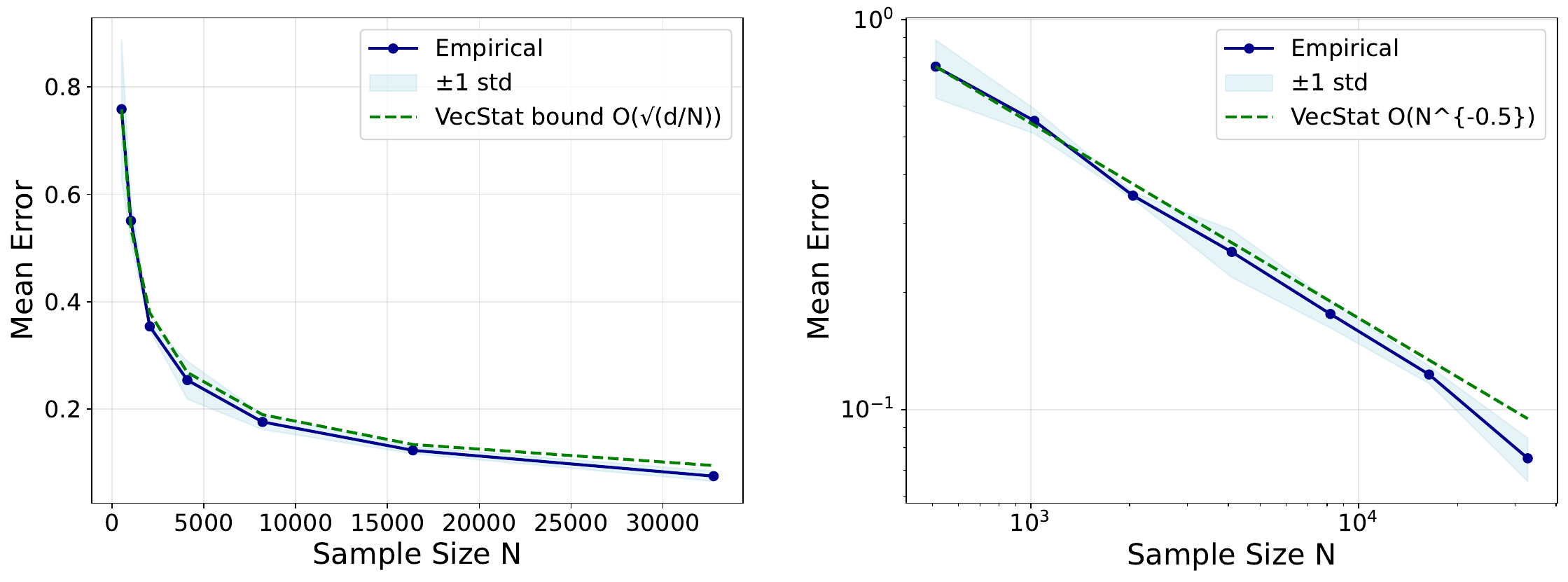}
          \caption{Llama-3.2-1B  (\vecstat)}
          \label{fig:llama-magicoder-proj}
      \end{subfigure}

      \vspace{0.5cm}

      \begin{subfigure}[b]{0.48\textwidth}
          \centering
          \includegraphics[width=\textwidth]{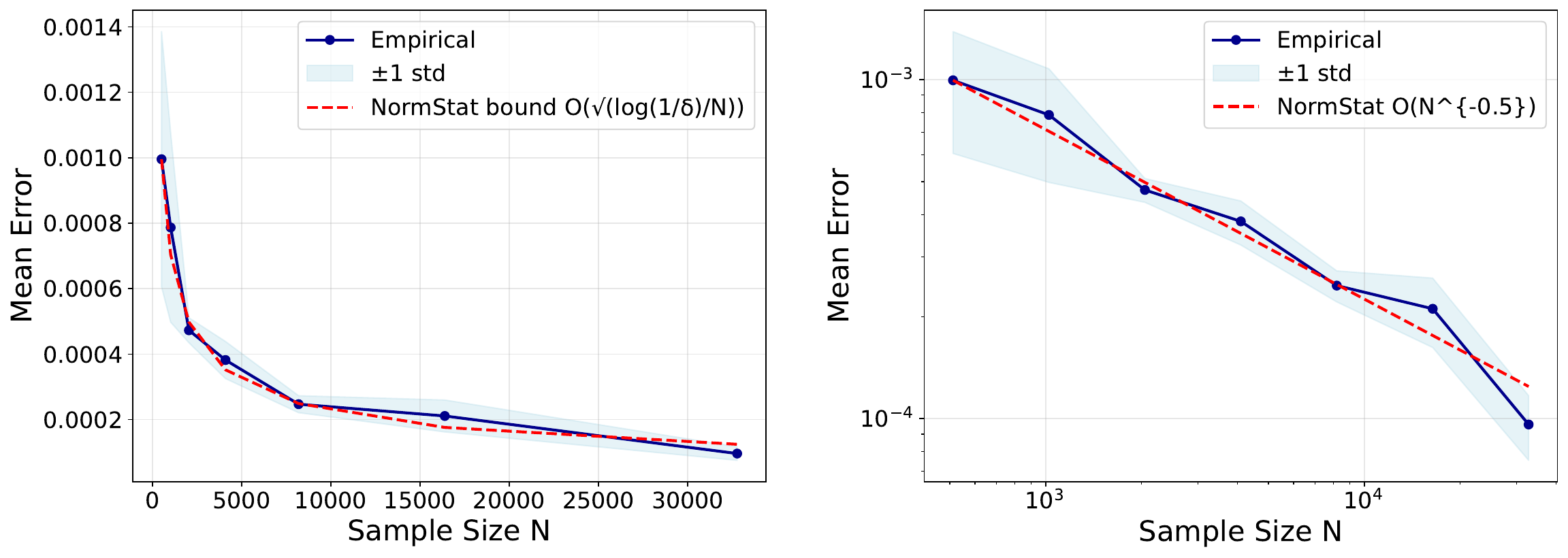}
          \caption{Qwen3-1.7B-Base  (\normstat)}
          \label{fig:qwen-base-magicoder-norm}
      \end{subfigure}
      \hfill
      \begin{subfigure}[b]{0.48\textwidth}
          \centering
          \includegraphics[width=\textwidth]{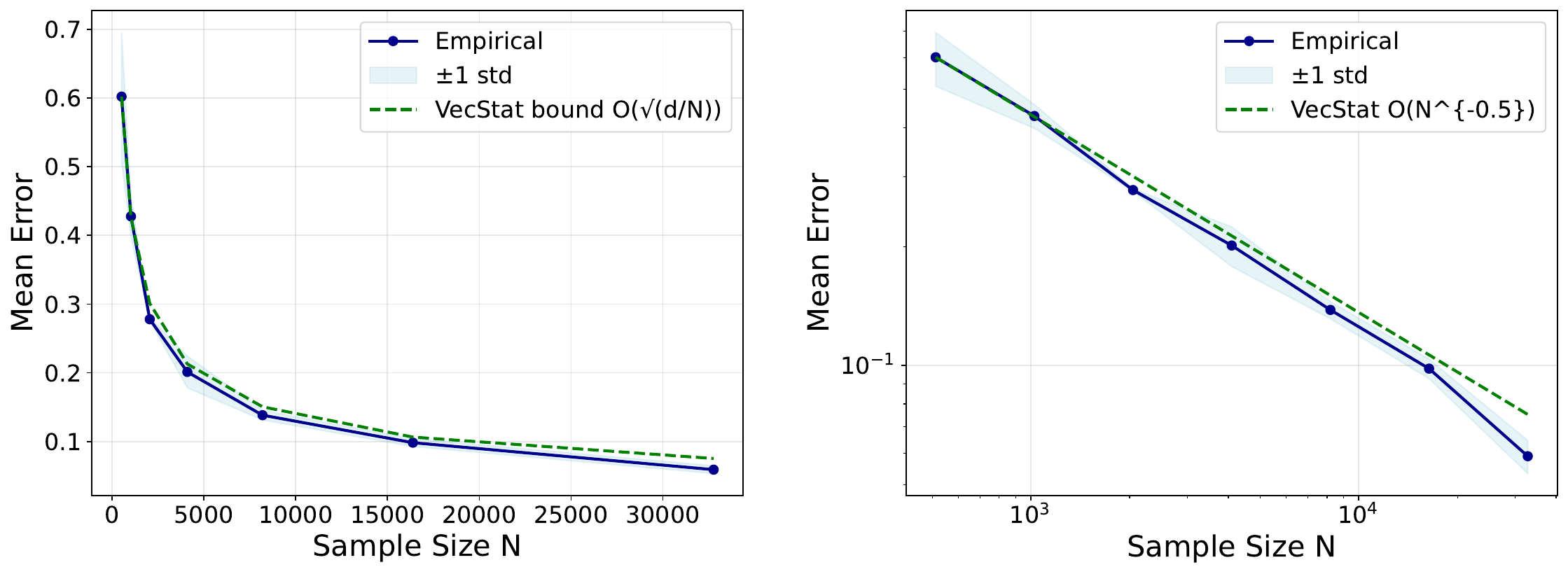}
          \caption{Qwen3-1.7B-Base  (\vecstat)}
          \label{fig:qwen-base-magicoder-proj}
      \end{subfigure}

      \vspace{0.5cm}

      \begin{subfigure}[b]{0.48\textwidth}
          \centering
          \includegraphics[width=\textwidth]{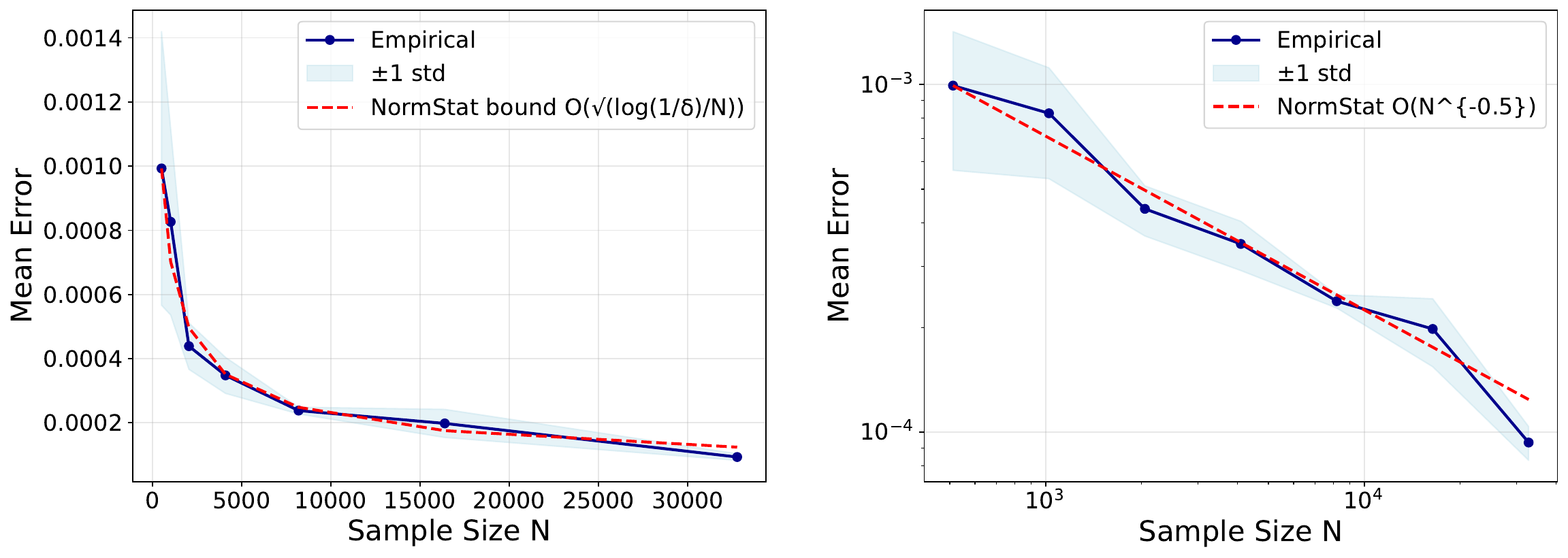}
          \caption{Qwen3-1.7B  (\normstat)}
          \label{fig:qwen-magicoder-norm}
      \end{subfigure}
      \hfill
      \begin{subfigure}[b]{0.48\textwidth}
          \centering
          \includegraphics[width=\textwidth]{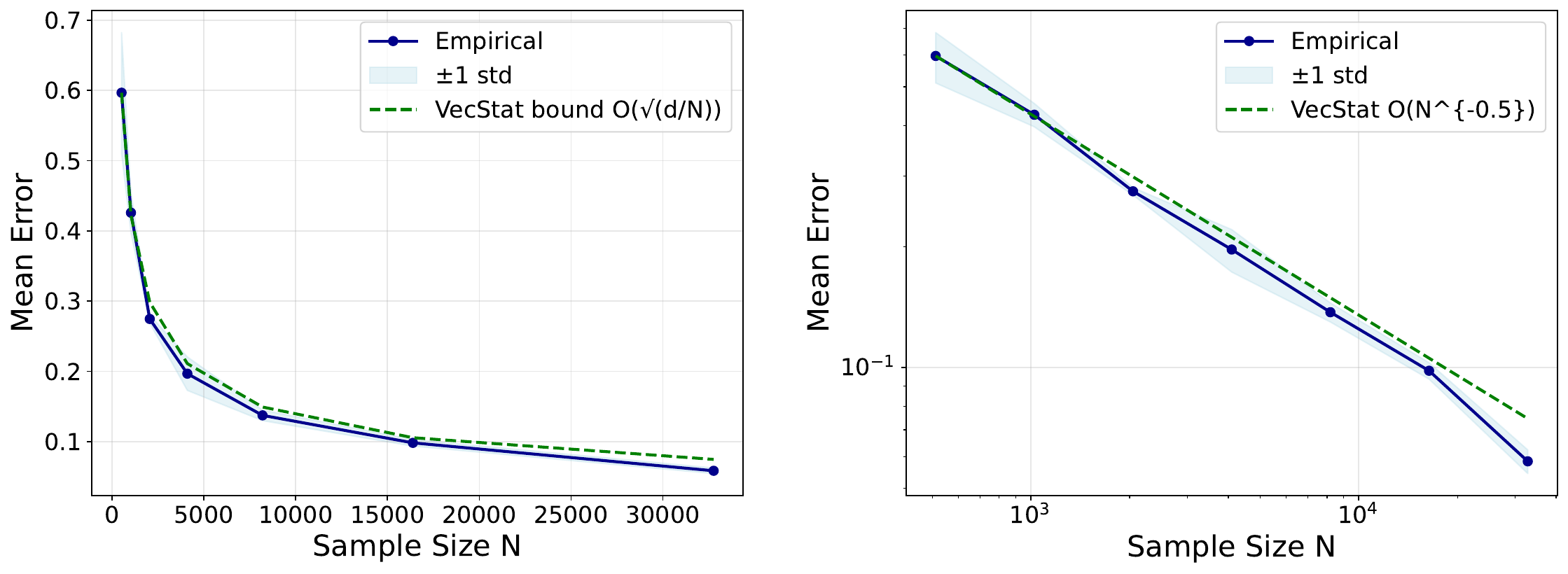}
          \caption{Qwen3-1.7B  (\vecstat)}
          \label{fig:qwen-magicoder-proj}
      \end{subfigure}

      \caption{\small Calibration convergence analysis for different models and methods on the MagiCoder dataset. Each subplot shows both linear and log-log scales
  comparing empirical results with theoretical bounds. \normstat\ (norm method) uses dimension-free bounds while \vecstat\ (projection method) uses dimension-dependent      
  bounds.}
      \label{fig:calibration-convergence-appendix}
  \end{figure}

\section{More Related Works}
\label{app:more_related_works}
\paragraph{LLM Routing}
Early work on LLM routing either ensembles outputs from multiple models \citep{jiang2023llm,wang2023fusing} or uses cascades that query models sequentially by capability \citep{aggarwal2023automix,chen2023frugalgpt,yue2023large}, but both incur high latency and cost due to multiple calls per query. Subsequent approaches train learned routers--model-based predictors--that estimate per-query quality or cost and select the target LLM \citep{hari2023tryage,stripelis2024tensoropera,feng2024graphrouter,dekoninck2024unified,somerstep2025carrot,jitkrittum2025universal}; these reduce unnecessary calls but introduce nontrivial training and maintenance overhead. Meanwhile, there are also training-free routers which choose among LLMs using lightweight ranking or budget-aware criteria \citep{zhao2024eagle,wu2025efficient}. In contrast, our training-free statistical method is cheaper still because it operates entirely within a single LLM’s prefill: we compute simple statistics of internal activations to obtain fast, calibrated intent probabilities that serve as an efficient router without extra forward passes or router training.

\end{document}